\documentclass[aoas]{imsart}

\usepackage{amssymb}
\RequirePackage{amsthm,amsmath,amsfonts,amssymb}
\RequirePackage[authoryear]{natbib}
\RequirePackage[colorlinks,citecolor=blue,urlcolor=blue]{hyperref}
\RequirePackage{graphicx}
\usepackage{bm}
\usepackage{amsmath}

\usepackage{algorithm}
\usepackage{algorithmic}
\usepackage{float}
\usepackage{caption}
\usepackage{setspace}

\newcounter{xalgsubstate}

\newenvironment{xalgsubstates}
  {\setcounter{xalgsubstate}{0}%
   \renewcommand{\STATE}{%
     \refstepcounter{xalgsubstate}%
     \item[] {\footnotesize\alph{xalgsubstate}.}\space}}
  {}

\newcounter{xalgsubsubstate}
\makeatletter
\renewcommand{\thexalgsubsubstate}{\advance\leftskip\arabic{ALC@line}.(\roman{xalgsubsubstate})}
\makeatother
\newenvironment{xalgsubsubstates}
  {\setcounter{xalgsubsubstate}{0}%
   \renewcommand{\STATE}{%
     \refstepcounter{xalgsubsubstate}%
     \item[] {\hspace{2em}\footnotesize(\roman{xalgsubsubstate}).}\space}}
  {}

\newcommand{\blue}[1]{\textbf{\textcolor{blue}{#1}}}

\startlocaldefs

\theoremstyle{plain}

\newtheorem{assumption}{Assumption}
\newtheorem{proposition}{Proposition}

\theoremstyle{remark}

\DeclareMathOperator*{\argmin}{argmin}
\DeclareMathOperator*{\argmax}{argmax}

 \newcommand{\ind}{\perp\!\!\!\!\perp} 

\endlocaldefs

\begin{document}

\begin{frontmatter}
\title{Dynamic Topic Language Model on heterogeneous Children's Mental Health Clinical Notes}

\begin{aug}
\author[A]{\fnms{}~\snm{Hanwen Ye}\ead[label=e1]{hanweny@uci.edu}},
\author[B]{\fnms{}~\snm{Tatiana Moreno}\ead[label=e2]{Tatiana.Moreno@choc.org}},
\author[B]{\fnms{}~\snm{Adrianne Alpern}\ead[label=e3]{AAlpern@choc.org}},
\author[B]{\fnms{}~\snm{Louis Ehwerhemuepha}\ead[label=e4]{LEhwerhemuepha@choc.org}}
\and
\author[A]{\fnms{}~\snm{Annie Qu}\ead[label=e5]{aqu2@uci.edu}}

\address[A]{Department of Statistics,
University of California, Irvine\printead[presep={,\ }]{e1,e5}}
\address[B]{Children's Hospital of Orange County\printead[presep={,\ }]{e2,e3,e4}}
\end{aug}

\begin{abstract}
Mental health diseases affect children's lives and well-beings which have received increased attention since the COVID-19 pandemic. Analyzing psychiatric clinical notes with topic models is critical to evaluating children's mental status over time. However, few topic models are built for longitudinal settings, and most existing approaches fail to capture temporal trajectories for each document. To address these challenges, we develop a dynamic topic model with consistent topics and individualized temporal dependencies on the evolving document metadata. Our model preserves the semantic meaning of discovered topics over time and incorporates heterogeneity among documents. In particular, when documents can be categorized, we propose a classifier-free approach to maximize topic heterogeneity across different document groups. We also present an efficient variational optimization procedure adapted for the multistage longitudinal setting. In this case study, we apply our method to the psychiatric clinical notes from a large tertiary pediatric hospital in Southern California and achieve a 38\% increase in the overall coherence of extracted topics. Our real data analysis reveals that children tend to express more negative emotions during state shutdowns and more positive when schools reopen. Furthermore, it suggests that sexual and gender minority (SGM) children display more pronounced reactions to major COVID-19 events and a greater sensitivity to vaccine-related news than non-SGM children. This study examines children's mental health progression during the pandemic and offers clinicians valuable insights to recognize disparities in children's mental health related to their sexual and gender identities.
\end{abstract}

\begin{keyword}
\kwd{Classifier-free}
\kwd{multistage topic language models}
\kwd{sexual and gender identity}
\kwd{time-consistent topics}
\kwd{variational inference}
\end{keyword}

\end{frontmatter}

\section{Introduction}
\subsection{Motivation}
Mental health conditions, such as anxiety, depression, and substance abuse, are prevalent among children and can have long-lasting impacts on their social relationships and academic performance. Without proper intervention, mental health conditions can lead to school absences, academic failure, isolation from peers, and-in some cases-an increased risk of suicide. Unfortunately, the outbreak of COVID-19 has intensified pediatric mental health issues due to the extent level of enforced physical/social isolation \citep{wu2021prevalence, ravens2022impact}. In particular, sexual and gender minority (SGM) youth who live in low-supportive home environments and lose access to previously affirming environments, such as schools and social activities, are placed at an increased risk for abuse and rejections from their family members \citep{mcgeough2018systematic,thoma2021disparities}. Therefore, to develop effective interventions and create post-pandemic support systems, it is crucial for mental health professionals to understand the dynamic changes and disparities in children's mental status concerning their sexual and gender identities during the pandemic \citep{salerno2020sexual}.

However, evaluating children’s mental health, in general, is challenging due to the self-reported symptoms and complex heterogeneity among subjects. For instance, studies found that girls and SGM children tend to report their symptoms differently, and are more likely to be diagnosed with depression compared to boys and non-SGM youth even though they express similar symptoms \citep{afifi2007gender,rosenfield2013gender,marshal2011suicidality,ploderl2015mental,russell2016mental}. To mitigate this bias, questionnaires and telephone surveys with validated rating scale \citep{penninx2008netherlands, barry2014midlife, boyd2013cohort} are conducted to quantify patients’ mental health symptoms with derived metrics. Yet, these study designs still suffer from selection bias and are unable to capture important life events or stressors impacting the patient’s mental health.

In this study, we leverage inpatient mental health unit notes from a large tertiary pediatric hospital in Southern California to evaluate children's mental health over the pandemic period. Compared to survey metrics, clinical notes contain a more contextualized background of a patient (e.g., mental health history, hospitalization reasons, interactions with clinicians, etc). Our goal is to identify major life events and stressors from these clinical records and track their dynamic shifts to quantify mental health status changes among nearly 2,600 inpatient children throughout the pandemic. Importantly, these contextual factors play a crucial role in revealing the prevalent and long-lasting themes in mental health, such as depression, anxiety disorders, and suicidal intentions \citep{scott1958research, 
ronald2010evolving, usnational, ciechanowski2023public}. By understanding the evolution trends of these themes, we offer valuable insights into the progression of children's mental health throughout the pandemic era. Moreover, taking into account each child's unique life experience and reactions, we incorporate individual-level heterogeneity and aim to uncover distinct trends between SGM and non-SGM children. This allows investigators to better recognize mental health disparity among SGM youth during COVID-19 and develop tailored post-pandemic support systems concerning sexual and gender identities.

\subsection{Literature review}
Unsupervised topic modeling is a popular statistical approach which aligns well with our objective of discovering abstract themes (i.e., topics) from a large corpus of text data. By clustering common patterns and keywords across multiple documents, topic models uncover the underlying semantics associations and summarize lengthy documents into a manageable number of interpretable topics. In the existing literature, topic modeling can be typically categorized into two major frameworks: Bayesian probabilistic topic models (BPTMs) and neural topic models (NTMs). A BPTM proposes a probabilistic generative model of a document and applies Bayesian inference procedures to estimate the posterior of latent topics. Representative methods include latent Dirichlet allocation (LDA) \citep{blei2003latent, blei2006correlated}, the dynamic LDA \citep{blei2006dynamic, wang2012continuous}, where topics evolve with time, and the supervised LDA \citep{mcauliffe2007supervised, roberts2014structural, li2015supervised, sridhar2022heterogeneous} with augmentation of document metadata. However, BPTMs generally suffer from low sampling efficiency and high technical difficulties in customizing the optimization procedure for each model prior specification. NTMs, on the other hand, are based on neural networks to model the relationships between words and topics. Benefiting from the standard gradient descent optimization procedure, NTMs can be easily integrated into different application cases and achieve high training efficiency on large datasets. Under the NTMs framework, one can find topics via clustering the word embedding representations \citep{thompson2020topic, sharifian2022analysing}, or modeling topics as latent variables in autoregressive models \citep{larochelle2012neural, gupta2019document}, generative adversarial networks (GANs) \citep{wang2019atm, hu2020neural}. In particular, the variational autoencoder (VAE)-based NTMs \citep{miao2016neural, srivastava2017autoencoding, lin2019sparsemax} have received the most attention due to their ability to capture complex word-topic associations with deep learning architecture, and meanwhile, provide probabilistic interpretations to the latent topics based on variational inference.

Among these frameworks discussed above, three methods offer promising solutions to our problem: the multistage dynamic LDA \citep{blei2006dynamic}, the supervised LDA \citep{mcauliffe2007supervised}, and SCHOLAR which is the VAE-NTM with metadata augmentation \citep{card2017neural}. However, none of these methods is directly applicable to our specific use case. First, the time-varying topics found by the dynamic LDAs may distort the meaning of each topic, fail to capture consistent mental health themes, and impose difficulties in interpreting the topic proportion trend due to the loss of time consistency. In addition, both supervised LDAs and SCHOLAR are single-stage topic models and rely on classifiers, instead of topic distributions, to differentiate groups. Moreover, few studies have extended the VAE-NTMs framework to the multistage longitudinal setting, though VAEs with spatiotemporal dependencies have been actively explored in the computer vision field \citep{gulrajani2016pixelvae, casale2018gaussian,fortuin2020gp, ramchandran2021longitudinal}.


\subsection{Contribution}

This paper proposes a novel multistage dynamic VAE-NTM, namely Heterogeneous Classifier-Free Dynamic Topic Model (HCF-DTM), with grouping information to address the challenges discussed above. In contrast to the dynamic LDAs, our method finds a number of time-consistent topics among all documents at any time point. This not only maintains the semantic meaning of discovered topics over the investigation period, but also enables the direct use of obtained topic proportions to infer the dynamic change in the popularity of each topic. Additionally, we augment the document metadata into the topic-finding procedure to account for the longitudinal heterogeneity among documents. Moreover, compared to the supervised LDAs, our proposed model increases the group-wise differences directly on the latent topic distribution level. Instead of relying on additional downstream classifiers, we introduce the counterfactual topic proportions and maximize the inter-distributional distances between topic proportions of the ground truth group and those as if the documents belonged to the other groups. As a result, the distinct characteristics of each group and their corresponding topic evolution trend can be easily identified.

The main clinical contributions of our paper are as follows. First, this work is among the first to utilize topic models on unstructured psychiatric clinical notes to unfold the longitudinal mental health disparities concerning sexual and gender identities during the pandemic. Knowledge of possible pronounced reactions among SGM children towards major COVID-19 events advocates for developing tailored post-pandemic treatments. Clinicians can design programs, such as virtual support groups and online mental health platforms \citep{whaibeh2022addressing, karim2022support,mcgregor2023providing}, to help address the heightened stress, anxiety, and isolation experienced by SGM youth when their previously accessible resources become limited due to the pandemic. In addition, our study assists clinicians in further examining SGM-related contextual stressors. This not only raises community awareness about the unique challenges faced by SGM children, promoting family education and fostering community support, but also informs future research on the mental health of SGM children and gets better prepared for future pandemics or social crises.


The remainder of this article is structured as follows. In Section \ref{sec:Background}, we introduce the notations and limitations of the dynamic LDA method. In Section \ref{sec:Methodology}, we propose the generative process of HCF-DTM, present variational inference details, and introduce a classifier-free approach to maximize heterogeneity between groups. Section \ref{sec:Implementation} explains the implementation algorithm. In Section \ref{sec:Simulation}, extensive simulation results are presented to illustrate the performance advantages of our proposed HCF-DTM method. In Section \ref{sec:real-data}, we apply the proposed method to the psychiatric inpatient clinical notes provided by a large tertiary pediatric hospital in Southern California. Lastly, we conclude with discussions in Section \ref{sec:Discussion}. Technical details and proofs are provided in Supplementary Materials \citep{YeSupplemental}. 

\section{Background and Related Works} \label{sec:Background}

\subsection{Notation and Preliminary}

Consider a balanced multistage study where $N$ participants undergo a total $T$ finite number of stages (visits). Each participant belongs to a corresponding group. For the illustration purpose, we consider a two-group scenario (e.g., SGM and non-SGM), denoted as $Y_i \in \mathcal{Y} = \{0,1\}$. At the $t^{th}$ stage, where $1 \le t \le T$, a set of subject's time-varying covariates and a clinical note (document) are recorded. The structured subjects' covariates are regarded as the metadata, denoted as $X_{it} \in \mathcal{X}_t$, and the unstructured clinical notes are the text data of interest. 

To formalize the unstructured text data, we assume that the number of unique words across all recorded documents is $V$ (vocabulary size). By assigning a unique id to each word, we represent any document of an arbitrary number of words $N_d$ with a vector of constant size $V$, i.e., $(\text{cnt}_{w_1}, \text{cnt}_{w_2}..., \text{cnt}_{w_v}) \in \mathbb{Z}_{\ge 0}^V$. This vector is known as the Bag of Words (BOW) representation, where each element counts the number of appearances of the corresponding word in a document. With BOW, we are able to vectorize the entire corpus of documents over $T$ time points with a structured tensor of size $T \times N \times V$, denoted as $\mathcal{D}=[\mathcal{D}_1,\mathcal{D}_2,...,\mathcal{D}_T]$. Now suppose at each time point $t$, there exist $K$ topics within documents $\mathcal{D}_t$, and each document is a mixture of these topics. To capture the proportions of each topic found in a document, we create a document-topic matrix $\Theta_{t,N \times K}$. Furthermore, we define a word-topic matrix $\beta_{t, V \times K}$ to represent the relevance of each word to the $K$ topics. The primary goal of topic modeling is to find the matrices $\Theta_{t,N \times K}$ and $\beta_{t,V \times K}$ which can best represent the documents $\mathcal{D}_t$ at time stage $t$ (i.e., $\mathcal{D}_{t, N \times V} \approx \Theta_{t,N \times K} \cdot \beta_{t,V \times K}^\intercal$).

\subsection{Multistage dynamic LDA}

Traditional non-negative matrix factorization (NMF) methods \citep{paatero1994positive, lee1999learning, li2021topic} view this problem as a matrix decomposition task, where $\Theta_{t, N \times K}$ and $\beta_{t, V \times K}$ are treated as two lower-rank matrices with $1 \le K << \min(N,V)$. By minimizing the distance (e.g., Frobenius norm) between $\mathcal{D}_{t, N \times V}$ and $\Theta_{t, N \times K} \cdot \beta_{t, V \times K}^\intercal$, the NMF methods estimate the representative topics and proportions. However, NMF is not a probabilistic method and therefore is limited in providing valid statistical inferences to the estimated topics. LDAs, on the other hand, provide probabilistic distributions on $\Theta_{t, N \times K}$ and $\beta_{t, V \times K}$ and incorporate them into the generative process of a document. For example, in the multistage dynamic LDA \citep{blei2006dynamic} generative process \ref{algo:D-LDA} defined as below, $\delta^2$, $\xi^2$, and $a^2$ are the variance priors for the latent topics and $\sigma(x_{1:V})_j=\frac{\exp{x_j}}{\sum_{j=1}^V \exp{x_j}}$ is a softmax function constraining the unbounded multivariate normal mean parameters to a valid multinomial probability simplex.

\makeatletter
\renewcommand\ALC@linenodelimiter{.} 
\makeatother
\floatname{algorithm}{Generative Process}
\begin{algorithm}[ht!]
\normalsize\setstretch{1.03}
  \begin{algorithmic}[1]
    \STATE{Draw word-topics distribution: $\beta_t\;|\; \beta_{t-1} \sim \mathcal{N}(\beta_{t-1}, \delta^2 I)$.}
    \STATE{Draw document-topics proportion: $\alpha_t\;|\; \alpha_{t-1} \sim \mathcal{N}(\alpha_{t-1}, \xi^2 I)$}
    \STATE{For each document $d$ at time $t$:}
    \begin{xalgsubstates}
      \STATE{Draw topics proportion: $\eta_{t,d,1:K} \sim \mathcal{N}(\alpha_t, a^2 I)$, \quad $\theta_{t,d,1:K} = \sigma(\eta_{t,d,1:K})$}
      \STATE{For each word at position $j$:}
      \begin{xalgsubsubstates}
            \STATE{Draw a topic for this word: $Z_{t,d,j} \sim \text{Mult}(\theta_{t,d,1:K})$}
            \STATE{Draw a word: $W_{t,d,j} \sim \text{Mult}(\sigma(\beta_{t,1:V, Z_{t,d,j}}))$.}
    \end{xalgsubsubstates}
    \end{xalgsubstates}
  \end{algorithmic}
\caption{Dynamic LDA}
\label{algo:D-LDA}
\end{algorithm}

Based on the Generative Process \ref{algo:D-LDA} formulation, we identify the following three limitations of the dynamic LDA to our application. First, the word-topics distribution $\beta_t$, which characterizes the semantic meaning of each topic, changes at each time stage. Commonly, the meaning of topics can be summarized by a broader theme, and the magnitude of change can be controlled via the variance prior $\delta^2$. However, finding a suitable prior requires manual inspections of the word-clouds at each time stage to interpret and confirm that the topics are under the same theme. As the number of stages and topics increases, the cumbersome nature of this process will inevitably impose a significant challenge in topic interpretations and representations. Secondly, the topic proportion $\theta_{t,d,1:K}$ is not directly correlated to its precursor $\theta_{t-1, d, 1:K}$ which describes the topic proportions of the same document $d$ from the previous time stage. Instead, all $\{\theta_{t,d,1:K}\}_{d=1}^N$ share a common corpus-level hyper-parameter $\alpha_t$, making it difficult for the current process to account for document-level heterogeneity. Lastly, the generative process above does not incorporate subjects' metadata and group information. Despite that the follow-up supervised LDA \citep{mcauliffe2007supervised} accounts for the group information by adding an extra classification task to the end of the generative process (i.e., $Y_{d} | Z_{t,d,.}, \phi, \gamma \sim \text{logit-Normal}(\phi^\intercal \cdot Z_{t,d,.}, \gamma^2))$, the performance of topic separation relies heavily on the prior parameters of the classifier, $\phi$ and $\gamma$, and still, it is challenging to distinguish whether a strong classification result is due to the quality of the classifier or the actual presence of interpretable and separable underlying topic distributions.




\section{Methodology} \label{sec:Methodology}
To address the challenges listed above, we propose the Heterogeneous Classifier-Free Dynamic Topic Model (HCF-DTM) with consistent topic interpretations and dynamic incorporation of documents' time-varying meta-data. In the following, we describe the generative process and a detailed variational inference procedure of our proposed model. In addition, we introduce a novel classifier-free approach to directly maximize the group-wise heterogeneity among topics with a notion of counterfactual topic distributions.

\subsection{Heterogeneous DTM with consistent topics}
The main objective of our proposed model is to identify a set of time-consistent topics while accounting for the evolving heterogeneity within documents over a specified longitudinal timeframe. This subsection presents the detailed specification of HCF-DTM in Generative Process \ref{algo: HDTM} and illustrates with the corresponding graphical model in Figure \ref{fig:Graphical-model-HDTM}.

\makeatletter
\renewcommand\ALC@linenodelimiter{.} 
\makeatother
\floatname{algorithm}{Generative Process}
\begin{algorithm}[hbt]
\normalsize\setstretch{1.05}
  \begin{algorithmic}[1]
    \STATE{Draw time-consistent word-topics distribution: $\beta \sim \mathcal{N}(\beta_0, \delta^2 I)$}
    \STATE{For each document $d$ at time $t$:}
    \begin{xalgsubstates}
    \STATE{Draw topic proportions: \\ \qquad  $\eta_{t,d,1:K} \;|\; \eta_{t-1,d,1:K}, X_{d,t}, Y_d, \phi_t \sim \mathcal{N}(f_{t, \phi_t}(\eta_{t-1,d,1:K}, X_{d,t}, Y_d), a^2 I)$,  \\ \qquad $\theta_{t,d,1:K} = \sigma(\eta_{t,d,1:K})$}
    \STATE{For each word at position $j$:}
      \begin{xalgsubsubstates}
            \STATE{Draw the word: $W_{t,d,j} \sim \text{Mult}(\theta_{t,d,1:K} \cdot \sigma(\beta)^\intercal)$}
    \end{xalgsubsubstates}
    \end{xalgsubstates}
  \end{algorithmic}
\caption{Heterogeneous DTM with consistent topics}
\label{algo: HDTM}
\end{algorithm}

Compared to the Generative Process \ref{algo:D-LDA} of previous DTMs, our method differentiates itself in three significant ways. First, instead of allowing each time point to have its own word-topic matrices $\{\beta_t\}_{t=1}^T$, we remove the time dependency and assume the existence of a single word-topic matrix, $\beta$, which is shared by all documents and held constant regardless of time or group memberships. Specifically, our approach achieves this by parameterizing the generative distribution of $\beta$ with a single time-invariant mean prior $\beta_0$ and sampling it once at the beginning of the process. In addition, we keep the generative variance prior, $\delta^2 I$, diagonal to ensure orthogonality and maximize disparities among topics. As a result, the topic matrix $\beta$ provides a consistent and distinguishable interpretation for each topic at every time point. This is more desirable as it helps uncover long-lasting topics and guarantees consistent themes without the need to navigate through a multitude of topics $\{\beta_t\}_{t=1}^T$ over time.

Apart from providing consistent topics, we augment documents' metadata into the generative process and establish individualized trajectories of the longitudinal topic proportions. Notably, since the topics $\beta$ do not vary with time in our design, the topic proportions $\{\theta_{t,d,1:K}\}_{t=1}^T$ encapsulate all temporal correlations within each document $d$. In particular, we utilize a function $f_t: \mathbb{R}^{K} \times \mathcal{X}_t \times \mathcal{Y} \mapsto \mathbb{R}^K$ parameterized by prior $\phi_t$ to capture the mean trend of the topic proportions. At every time stage $t$ for a document $d$, the mean function $f_t$ takes three inputs: the past topic proportions $\theta_{t-1,d,1:K}$, subject's longitudinal covariates $X_{d,t}$, and group membership $Y_d$ to provide predictions for the expected topic proportions. Consequently, the topic proportions $\Theta_t$ incorporate the dynamic heterogeneity among subjects, and are able to control for confounders which may affect their distributions. More importantly, we let the topic proportions $\Theta_t$ directly depend on $\Theta_{t-1}$ from the previous stage, rather than a corpus-level hyperparameter $\alpha_t$ used in dynamic LDA. This enables us to better track the individual progression of topic proportions on the document/subject level.

\begin{figure}[bt!]
    \centering
    \includegraphics[scale=0.5]{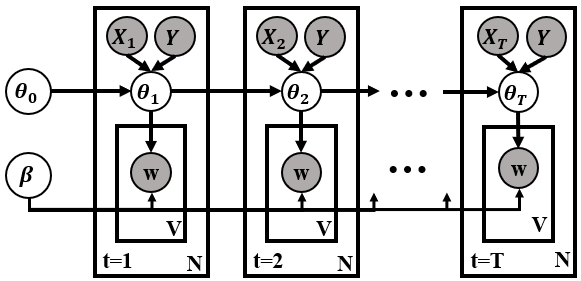}
    \caption{A graphical model of the heterogeneous DTM in a balanced multistage longitudinal setting, where $\theta_0$ is a prior for the initial topic proportions. Topics $\beta$'s are held constant and provided at every time stage. Both $\theta$ and $\beta$ are the latent variables, whereas $W$, $X$ and $Y$ are observed variables.}
    \label{fig:Graphical-model-HDTM}
\end{figure}

Furthermore, our heterogeneous DTM generative process simplifies the optimization procedure. By estimating the time-consistent latent topics simultaneously across different time periods, we reduce the number of model parameters and decrease the inference complexity level. Additionally, inspired by the work of \cite{miao2016neural}, we collapse the latent random variable $Z$, which determines the topic assignment for each word, from the word sampling distribution based on the following results,
 \begin{equation}
     p(w_t | \theta_t, \beta) = \sum_{z=1}^K p(z | \theta_t) \cdot p(w_t | z, \beta) = \sum_{z=1}^K \theta_{t,z} \cdot \sigma(\beta_{z}) = (\theta_t \cdot \sigma(\beta)^\intercal)_w.
 \end{equation}
Unlike the two-layer hierarchical dependency outlined in Generative Process \ref{algo:D-LDA}, we model the word generative process jointly with the topics $\beta$ and their proportions $\Theta_t$, as in $W_t \sim \text{Mult}(1, \theta_t \cdot \sigma(\beta)^\intercal)$. This collapsing approach has the advantage of eliminating the need to sample the discrete random variable $Z$ during the optimization process, which is considered to be more challenging. 

Finally, we aim to estimate the posterior of the latent topic proportions via the Bayes rule according to the defined prior distributions of the latent topics and the likelihood model of the documents in Generative Process \ref{algo: HDTM}. To simplify notations, we use $\theta_t$, $X_t$, and $Y$ as shorthand to represent $\theta_{t,d,1:K}$, $X_{d,t}$, and $Y_d$ respectively for each document $d$ from this point onward.
\begin{equation}
p(\theta_t \;|\; \theta_{1:(t-1)}, \bm{w}_t, X_t, Y, \beta) = \frac{p(\theta_t | \theta_{1:(t-1)},X_t, Y) \cdot p(\bm{w}_t|\theta_{t}, \beta)}{\int_{\theta_t} p(\bm{w}_t | \theta_t, \beta) \cdot p(\theta_t | \theta_{1:(t-1)}, X_t, Y) d\theta_t}.
\label{eq:posterior}
\end{equation}
With the posterior in Equation \eqref{eq:posterior}, we can infer the latent topic proportions $\theta_t$ at any time $t$ from the observed document data. However, obtaining this posterior distribution is extremely difficult. This is because the marginalized data distribution $p(\bm{w}_t | \theta_{1:(t-1)}, \beta)$ in the denominator is intractable due to the unbounded space of the latent variable $\theta_{t}$ during integration. To address this challenge, we leverage the variational Bayes and propose a novel solution to extend the inference procedure of heterogeneous topic models to a multistage longitudinal setting. A detailed explanation of the longitudinal variational Bayes approach under our generative process specifications is provided in the next subsection.

\subsection{Longitudinal variational inference}
\label{sec:method-variational}

Variational inference is a powerful technique for estimating the intractable posterior of interest $P$ with a parametrizable distribution $Q$, such as a Gaussian distribution. To ensure $Q$ is a valid approximation to $P$, variational methods minimize the Kullback-Leibler (KL)-divergence between the two distributions. This transforms the inference problem into an efficient optimization task, as future posterior inference can be conducted directly from $Q$. In our specific case, we aim to find the optimal set of parameters $\psi^*_{1:T}$ over the $T$ number of stages so that the parameterized approximating distribution $Q^* \doteq Q_{\psi^*_{1:T}}$ is as close as possible to the true topic posterior distribution $P$, i.e.,
\begin{align}
\psi^*_{1:T}&= \argmin_{\psi_{1:T} \in \Psi^{|T|}}\mathbb{KL}\left(Q_{\psi_{1:T}}(\theta_{1:T}|\bm{w}_{1:T},X_{1:T},Y) \;||\; P(\theta_{1:T}|\bm{w}_{1:T},X_{1:T},Y,\beta)\right) \label{eq: KL} \\[-0.5em]
&= \argmin_{\psi_{1:T} \in \Psi^{|T|}}\underbrace{\mathbb{KL} \left(Q_{\psi_{1:T}}(\theta_{1:T}|\bm{w}_{1:T},X_{1:T},Y) \;||\; p(\theta_{1:T}|X_{1:T},Y)\right)}_{\text{\textbf{Approximation error}: variational distribution $Q$ and geneartive prior $p$}} + \label{eq: KL-single} \\[-0.5em]
& \quad \qquad \qquad \underbrace{-\mathbb {E}_{\theta_{1:T} \sim Q_{\psi_{1:T}}} \left(\log P(\bm{w}_{1:T} | \theta_{1:T}, \beta)\right)}_{\text{\textbf{Reconstruction error}: model likelihood from variational topics}} \notag
\end{align}

In a single-stage setting, optimizing $\psi^*$ is straightforward through minimizing the negative evidence lower bound (ELBO) in Equation \eqref{eq: KL-single}. However, the inference is more difficult to define when there are multiple stages involved since additional temporal dependencies arise among the topic proportions $\theta_{1:T}$. Moreover, as $\theta_{1:T}$ are latent variables and not directly observable, we sample $\theta_t$ along with its temporal predecessors $\theta_{1:(t-1)}$ at each time stage, which increases the computational complexity exponentially with the number of stages. Therefore, it is more desirable to derive a variational objective which can efficiently break the long dependent sequence into smaller sets of stages while preserving temporal correlations in a multistage longitudinal setting. Before presenting our longitudinal ELBO, we first introduce the following regularity assumptions:
\vspace{-0.5em}
\begin{assumption}
(Markov property of latent topic proportions): Topic proportions at the next stage only depend on the current stage, not on any past stages: $\theta_{t+1} \ind \theta_1,...,\theta_{t-1} | \theta_t$.
\label{assumption:markov}
\end{assumption}
\vspace{-1.5em}
 \begin{assumption}
(Independence generation of documents): The distribution of every word from documents at time $t$ only depends on the time-consistent topics and current-stage topic proportions: $\bm{w}_t \ind \left(\bm{w}_{1:(t-1)}, \theta_{1:(t-1)}, X_{1:t},Y\right) \;|\; \theta_t, \beta$.
\label{assumption: independence}
\end{assumption}

\vspace{-0.25em}
Assumption \ref{assumption:markov} relaxes the temporal dependencies between topic proportions which are more than two stages apart since topics of a document are majorly influenced by the ones from the previous stage. In other words, current-stage topic proportions capture all the past information required for the proportions at the next stage. Assumption \ref{assumption: independence} states that it is sufficient to generate all words in a document given the current-stage topic proportions and topics, as subjects' heterogeneity and potential confounders have been considered in $\theta_t$ by the mean trend function $f_t$. Under these two assumptions and the variational distribution factorization $Q_{\psi_{1:T}}(\theta_{1:T}, \bm{w}_{1:T}, X_{1:T}, Y) = \prod_{t=1}^T q_{\psi_t}(\theta_t, \bm{w}_{t}, X_{t}, Y)$, we present the following ELBO as the variational objective under our multistage heterogeneous DTM Generative Process \ref{algo: HDTM}.
\vspace{-0.3em}
\begin{proposition} \label{prop:LELBO}
Under Assumptions \ref{assumption:markov}- \ref{assumption: independence}, the evidence lower bound (ELBO) for a single document generated by Process \ref{algo: HDTM} over a finite $T$-stage longitudinal time horizon is 
\vspace{-0.2em}
\begin{align}
\hspace{3em} & \log P(\bm{w}_{1:T} | X_{1:T}, Y, \beta) \ge \label{eq: ELBO-log-likelihood}\raisetag{-2.5em} \\[-0.15em]
& \hspace{1.5em}\underbrace{-\mathbb{KL}(q_{\psi_1}(\theta_1 | \bm{w}_1, X_1, Y) \;||\; p(\theta_1 | \theta_0, X_1, Y)) +  \mathop{\mathbb{E}}_{\theta_1 \sim q_{\psi_1}} (\log P(\bm{w}_1| \theta_1, \beta))}_{\text{Single-stage ELBO for the first stage}} + 
\label{eq: L-ELBO} \raisetag{-2.5em} \\[-0.5em]
& \hspace{-2.7em} \underbrace{\sum_{t=2}^T \left\{-\mathop{\mathbb{E}}_{\theta_{t-1} \sim q_{\psi_{t-1}}} \left( \mathbb{KL}[q_{\psi_t}(\theta_t | \bm{w}_t, X_t, Y) \;||\; p(\theta_t | \theta_{t-1}, X_t, Y)]\right) + \mathop{\mathbb{E}}_{\theta_t \sim q_{\psi_t}} (\log P(\bm{w}_t| \theta_t, \beta))\right\}}_{\text{Temporal-dependent ELBO for the follow-up stages}}  \notag
\end{align}
\end{proposition}
Equation \eqref{eq: L-ELBO} provides a lower bound for the log-likelihood of a document generated according to Generative Process \ref{algo: HDTM}. Maximizing this lower bound is equivalent to minimizing the negative ELBO in Equation \eqref{eq: KL}, where both optimizations result in the same set of optimal variational parameters $\psi_{1:T}^*$. Under the regularity assumptions, the presented longitudinal ELBO can be further decomposed into two components: the standard single-stage ELBO for the first stage, and the temporal-correlated ELBO for the follow-up stages. Specifically, the latter divides all future stages into adjacent-stage pairs based on Assumption \ref{assumption:markov}. By calculating the KL-divergence of current-stage topic proportions using the proportions sampled from the variational distribution at the previous stage, the approximation error across all time stages can be divided into adjacent stages. This leads to a more efficient inference process as the adjacent-stage dependency eliminates the need for sampling proportions of the entire time sequence at each stage. Additionally, based on Assumption \ref{assumption: independence}, the reconstruction error can be calculated individually at each time stage even without the pairwise dependency, which further enhances the efficiency of the variational learning process. 

In summary, the longitudinal ELBO shown in Equation \eqref{eq: L-ELBO} extends variational inference for DTMs to a multistage longitudinal setting. It provides an efficient optimization objective to approximate the true posterior of the topic proportions according to our proposed heterogeneous DTMs generative process. However, as the posterior of topic proportions $\theta_t$ depends on the topics $\beta$ shown in Equation \eqref{eq:posterior}, the proportion of a topic can vary significantly based on the learned latent topics. For instance, the differences in the topic proportion distributions between non-SGM and SGM youth may be more pronounced if the topic is related to mental health rather than physiological measures. Depending on the provided latent topics, the estimated topic proportions may not optimally present the heterogeneity from the groups. Thus, to better understand group-wise differences, we propose a classifier-free approach via distributional distances to learn the latent topics.

\subsection{Group-wise topic separation} 
In this subsection, our goal is not only to identify the most representative topics under Generative Process \ref{algo: HDTM}, but more importantly, to maximize the heterogeneity between groups in their respective proportions. To achieve this, we introduce the counterfactual topic distributions, which describe the topic proportion distributions for subjects who have the same documents $w_{1:T}$ and measurements $X_{1:T}$ but belong to different groups $Y$. The value function of HCF-DTM is presented as follows,
\begin{align}
\label{eq: KL-group}
V^{HCF}(\psi_{1:T}) &= \underbrace{-\mathbb{KL}\left(Q_{\psi_{1:T}}(\theta_{1:T}|\bm{w}_{1:T},X_{1:T},Y) \;||\; P(\theta_{1:T}|\bm{w}_{1:T},X_{1:T},Y,\beta)\right)}_{\text{Variational inference objective \eqref{eq: L-ELBO}}} + \\
&\quad \underbrace{\text{Dist} \left( \left\{Q_{\psi_{1:T}}(\theta_{1:T}|\bm{w}_{1:T},X_{1:T},Y=y) \right\}_{y\in \mathcal{Y}} \right)}_{\text{Group-wise topic proportion distribution distance}}. \notag 
\end{align}
where information radius \citep{sibson1969information}, average divergence score \citep{sgarro1981informational}, or mutual information (MI) in a two-group setting, can be selected as the distance metrics. To represent the counterfactual distribution, we use the constructed variational $Q$ by changing its group membership covariate $Y$ under the \textit{no unmeasured confounding} assumption \citep{robins1986new}. The main objective is to find the set of variational parameters which maximizes the value function, i.e., $\psi^*_{1:T}= \argmax_{\psi_{1:T} \in \psi^{|T|}} V^{HCF}(\psi_{1:T})$. In particular, the first term encourages closer approximation of variational distribution $Q$ to the underlying posterior $P$, meanwhile, the second term aims to increase the distances among the marginal topic proportion distributions, $P(\theta_{1:T} | Y)$, for each group membership Y.

 The proposed distance maximizing approach has the following two advantages. First, it provides explicit guidance to learn the latent topics that have the largest group-wise difference in their corresponding proportions. Due to the co-dependency between the latent topics and their proportions, the maximum amount of group heterogeneity which can be captured by the topic proportions depends on the provided latent topics. To illustrate this, consider the example of a two-group scenario shown in Figure \ref{fig:topic-distributions}. Compared to latent topics $\beta_1$ and $\beta_2$, topics $\Tilde{\beta_1}$ and $\Tilde{\beta_2}$ lead to a larger averaged distance between the topic proportion marginal distributions, and noticeably can better disentangle the group identities. To make $\Tilde{\beta_1}$ and $\Tilde{\beta_2}$ more likely to be identified in practice, the proposed second term in value function \eqref{eq: KL-group} explicitly maximizes the inter-distributional distances during the optimization procedure. As a result, latent topics that have larger distances in their proportion distributions are favored. In addition, regularized by the first variational term, the optimized latent topics are also pertained to the Generative Process \ref{algo: HDTM} and aimed to obtain the best representation of the documents.

\begin{figure}
    \centering
    \includegraphics[width=0.8\textwidth]{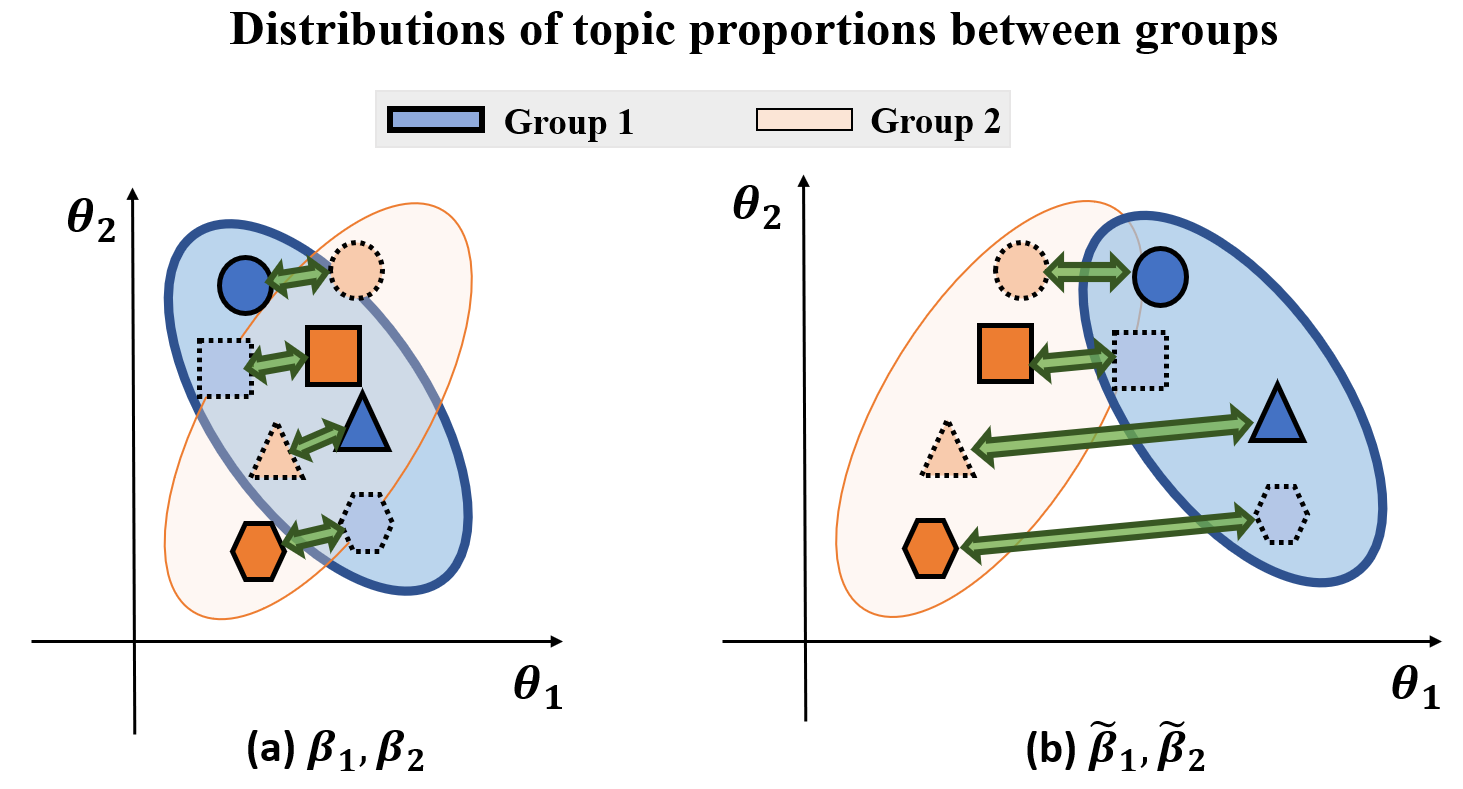}
    \caption{Comparison of topic proportion distributions under two sets of latent topics: $\beta_1$, $\beta_2$ versus $\Tilde{\beta}_1$, $\Tilde{\beta}_2$. The shapes displayed on a two-dimensional plane correspond to optimal topic proportions for each latent topic specification. Solid outlines indicate proportions from the true group identity, while dashed outlines represent counterfactual proportions.}
    \label{fig:topic-distributions}
\end{figure}

Secondly, we maximize the group disparity in topic proportions via a classifier-free manner. Unlike the previous supervised topic models \citep{mcauliffe2007supervised, card2017neural} which rely on a parameterized classifier to distinguish group identities, we evaluate group-wise topic proportion similarity by calculating the closed-form distance metrics directly from their distributions. At the group-wise distribution level, we optimize the variational free-parameters to increase the group-wise disparity. As a result, this approach not only relieves the need to posit and estimate extra parameters for the classifier, but also protects the latent topics from any semantic distortion resulting from potential projections applied by the classifier. Moreover, by leveraging the counterfactual topic proportions, we no longer require separating topic proportion distribution for each group, and therefore can increase the sample efficiency and reduce the label imbalance level.

To conclude, we propose the HCF-DTM which extends the VAE-NTM to a longitudinal setting and provides an efficient variational objective by segmenting the long dependent sequence into adjacent stages. To preserve the semantic meaning of the topics, HCF-DTM finds a number of time-consistent latent topics and maximizes group-wise heterogeneity via inter-distributional distances. In the next section, we present the implementation details of the proposed method.

\section{Implementation and Algorithm} 
\label{sec:Implementation}

The optimization procedure focuses on two sets of parameters of interest: the generative parameters $\Omega$ and the variational parameters $\Psi$. Precisely, $\Omega$ specifies the Generative Process \ref{algo: HDTM}, including the time-invariant topics $\beta$ and parameters $\phi_{1:T}$ of the mean trend functions $f_{1:T}$; whereas $\Psi$ parameterizes the mean functions $g^\mu_{1:T}$ and variance functions $g^\sigma_{1:T}$ of the variational distribution $Q$. Our goal is to find the optimal set of $\Omega$ and $\Psi$ which maximize the value function presented in Equation \eqref{eq: KL-group}.

The implementation detail is summarized in Algorithm \ref{algo}. It starts with learning the variational mean and variance, $\mu^q$ and $\sigma^q$, and then estimates the proposed value function which can be decomposed into three major components: the KL divergence between the variational and generative priors ($\mu^0$ and $\sigma^0$), the log-likelihood of the documents generated from the latent topics, and the distributional distances between the counterfactual topic proportions. For demonstration purpose, we consider a two-group setting and adopt the MI as the distance metrics. Calculating the log-likelihood and the group-wise MI is straightforward and can be conducted at each time stage. However, obtaining the longitudinal KL divergence is more challenging due to the temporal dependencies present in the follow-up stages. Though Equation \eqref{eq: L-ELBO} breaks down the long temporal dependency to adjacent stages, it is still required to compute an expectation of KL divergence over the intractable space of all latent topic proportions at the previous stage, as shown in  $\mathop{\mathbb{E}}_{\theta_{t-1} \sim q_{\psi_{t-1}}} \left( \mathbb{KL}[q_{\psi_t}(\theta_t | \bm{w}_t, X_t, Y) \;||\; p(\theta_t | \theta_{t-1}, X_t, Y)]\right)$. 

\makeatletter
\renewcommand\ALC@linenodelimiter{.} 
\makeatother
\setcounter{algorithm}{0}
\floatname{algorithm}{Algorithm}
\begin{algorithm}[hbt]
\normalsize\setstretch{1}
  \begin{algorithmic}[1]
    \STATE{\textbf{Initialize} generative parameters $\Omega=\{\beta, \phi_{1:T}\}$ and variational parameters $\Psi=\{\psi^\mu_{1:T}, \psi^\sigma_{1:T}\}$; variational mean functions $g^\mu_{1:T}$ and variance functions $g^\sigma_{1:T}$; topic proportion mean prior $\mu^0_1$ and variance prior $\sigma^0_1$; KL sample number $M$; learning rate $\lambda$; maximum iterations $T_{max}$; and a stopping error criterion $\epsilon_s$.}
    \STATE{\textbf{Input} all observed documents, metadata, and group $\{W_{i, 1:T}, X_{i, 1:T}, Y_i\}_{i=1}^N$.}
    \STATE{\textbf{For} $k \gets 1$ to $T_{max}$ \textbf{do}}
    \STATE{\quad Compute variational $\mu^q_t = g^{\mu}_t(W_t, X_t, Y)$ and std.err $\sigma^q_t = g^{\sigma}_t(W_t, X_t, Y)$}
    \STATE{\quad \textbf{For} $j \gets 1$ to $M$ \textbf{do}}
    \STATE{\quad \quad Sample Gaussian errors $\epsilon_{j,t} \sim N(0, 1)$} and reparametrize $\eta^q_{t,j} = \mu^q_t + \epsilon_{j,t} \cdot \sigma^q_t$
    \STATE{\quad \quad Compute prior mean $\mu^0_{1,j} = f_1(\mu^0_1, X_1, Y)$, and $\mu^0_{t,j} = f_t(\eta^q_{t-1,j}, X_t, Y)$}
    \STATE{\quad Compute counterfactual $\Tilde{\mu}^q_t = g^{\mu}_t(W_t, X_t, 1-Y)$ and $\Tilde{\sigma}^q_t = g^{\sigma}_t(W_t, X_t, 1-Y)$}
    \STATE{\quad Compute gradient $\nabla$ of Equation \eqref{eq:loss-func} w.r.t. $\Omega$ and $\Psi$}.
    \STATE{\quad Update $(\Omega, \Psi)^k \gets (\Omega, \Psi)^{k-1} - \lambda \cdot \nabla $}
    \STATE{\quad Stop if $|\mathcal{L}^k - \mathcal{L}^{k-1}| \le \epsilon$}
    \STATE{\textbf{Return} estimated topics $\beta$ and variational functions $\{g^{\mu}_{t}\}_{t=1}^T$ and $\{g^{\sigma}_t\}_{t=1}^T$ }
  \end{algorithmic}
\caption{Heterogeneous Classifier-Free Dynamic Topic Model}
\label{algo}
\end{algorithm}

To address this challenge, we apply the re-parametrization schemes $M$ times to have a random sample of $\theta_{t-1}$ with size $M$ drawn from the variational distribution. The sampled $\theta_{t-1}$ are later used to generate the prior distribution of $\theta_t$ based on the mean function $f_t$, which enables us to obtain an empirical estimate of the expected KL divergence at time $t$. Together with two other loss terms, the final objective function can be derived as follows,
\begin{align}
\label{eq:loss-func}
\quad &\left(\Omega^*, \Psi^*\right) = \argmin_{\Omega, \Psi} \frac{1}{N \cdot M}\sum_{i=1}^N \sum_{t=1}^T \sum_{j=1}^M  \underbrace{\left\{\log\left(\frac{\sigma^0}{\sigma^q_{i,t}}\right) + \frac{(\sigma^q_{i,t})^2 + (\mu^q_{i,t} - \mu^0_{i, j, t})^2)}{2\cdot (\sigma^0)^2} - \frac{1}{2}\right\}}_{\text{Gaussian KL divergence term}}- \\[-0.5em]
& \qquad \quad \underbrace{\vphantom{\log\left(\frac{\sigma^q}{\sigma^q_{i,t}}\right)} W_{i,t} \cdot \log\left\{\sigma(\eta^q_{i,t,j})\cdot \sigma(\beta)\right\}}_{\text{Multinomial likelihood term}} - \underbrace{\frac{1}{2}\left\{ \log\left(\frac{\sigma^q_{i,t} + \Tilde{\sigma}^q_{i,t}}{4 \cdot \sigma^q_{i,t} \cdot \Tilde{\sigma}^q_{i,t}}\right) + \frac{(\mu^q_{i,t} - \Tilde{\mu}^q_{i,t})^2}{\sigma^q_{i,t} + \Tilde{\sigma}^q_{i,t}} + \frac{1}{2}\right\}}_{\text{Mutual Information term}}, \notag
\end{align}
where $\epsilon_{t,j} \overset{iid}{\sim} \mathcal{N}(0,1)$, $\eta^q_{i,t,j} = \mu^q_{i,t} + \epsilon_{t,j} \cdot \sigma^q_{i,t}$ is the unnormalized variational topic proportion sample, and $\mu^0_{i,j,t} = f_t(\eta^q_{i,t-1, j}, X_{i,t}, Y_i)$ is the resulting prior proportion mean. Parameters of interest are updated jointly via stochastic gradient descent \citep{robbins1951stochastic}. For further reference, we present a graphical illustration of our model architecture and leave more detailed optimization choices, such as hyperparameter tuning, in Supplementary Materials \citep{YeSupplemental}.


\section{Simulation} \label{sec:Simulation}

In this section, we present extensive simulation studies to illustrate the longitudinal interpretability and group-wise separation ability of our proposed method. To compare, we consider several BPTM-methods: LDA \citep{blei2003latent}, multistage dynamic LDA (mdLDA) \citep{blei2006dynamic}, and supervised LDA (sLDA) \citep{mcauliffe2007supervised}; as well as NTM-based approaches: prodLDA \citep{srivastava2017autoencoding} and SCHOLAR \citep{card2017neural}. Note that sLDA and SCHOLAR are the two methods augmenting the document metadata, and mdLDA is the only multistage topic model which incorporates the longitudinal dependency of the documents. The main objective of the simulation is to investigate whether the topic models can recover the underlying topic distributions and identify subjects' group memberships under various settings, such as the number of time stages and topics.

The detailed simulation setting is described as follows. First, we consider a cohort of $N$ subjects over $T$ number of stages. At the baseline stage ($t=1$), we assign 20 random features $\{X_{i1p}\}_{p=1}^{20}$ generated from a standard normal distribution $N(0, 1)$ to each subject $i$. During all follow-up stages ($2 \le t \le T$), we let those features be correlated over time according to $X_{i,t,p} = X_{i,t-1,p} + \epsilon$, where $\epsilon \overset{iid}{\sim} N(0, 1)$. The subjects are randomly assigned to one of the two groups $Y \in \{-1, 1\}$ with equal probabilities. After the subjects' metadata $(\{X_{i,t,.}\}_{t=1}^T, Y_i)_{i=1}^N$ are generated, we define the document-level generative priors, i.e., topics $\beta$ and proportions $\Theta$. 

Under each simulation setting, we assume the existence of $K$ number underlying topics $\beta_{V \times K}$, where each topic $\beta_{1:V, k}$ describes a distribution over $V$ number of words. For ease of denotations, the words are represented as numerical integer values ranging from $1$ to $V$. To maximize the differences among topics, we let $\beta_{v, k} \sim \text{logit-Normal}(\mu_k, 1)$, where $\mu_k = \lfloor k \cdot \frac{V}{K} \rfloor$, and keep them constant across the time stages. Then, we design functions $\{f_{t, 1},...,f_{t,K}\}_{t=1}^T$ to generate time-varying topic proportions based on the document-level metadata. Specifically, the topics proportions at time $t$ for subject $i$ can be represented as,  $(\theta_{t,i,1},...,\theta_{t,i,K}) = \sigma\left(f_{t,1}(X_{i,t,.}, Y_i, \theta_{t-1,i,1}),...,f_{t,K}(X_{i,t,.}, Y_i, \theta_{t-1,i,K})\right)$, where each function $f_{t,k}$ 
is parameterized into the following three components: 
\begin{equation}
f_{t,k} = \underbrace{\gamma_{t,k}^m \cdot X_{i,t,.}}_{\text{covariates main effect}} + \underbrace{\gamma_t^\theta \cdot \theta_{t-1,i,k}}_{\text{dependency from previous $\theta$}} + \underbrace{Y_i \cdot \gamma_{t,k}^g \cdot X_{i,t,.}}_{\text{group membership effect}}.
\label{sim:f}
\end{equation}
The regression coefficients $\gamma_{t,k} \overset{iid}{\sim} N(0,1)$ are shared among all subjects at time $t$. In addition, Equation \eqref{sim:f} lists a linear relationship between the document metadata and the topic proportions. To increase the prior function complexity, we also impose non-linearity by transforming metadata based on a specified function basis $\{X, X^2, X^3, \arctan X, \text{sign}(X)\}$. Finally, we generate documents $d_{i, t} \overset{iid}{\sim} \text{Mult}(\text{cnt}_{i,t}, \theta_{t,i,1:K} \cdot \beta_{1:V,1:K}^\intercal)$ with number of words $\text{cnt}_{i, t} \overset{iid}{\sim} \text{Unif}(\{50,..,150\})$.

We train the topic models on $80\%$ of the generated documents and make inferences on a hold-out $1000$ document set. Each topic model provides a posterior of the topic distributions $\hat{\beta}$ and proportions $\hat{\theta}$, and can be evaluated based on the following three criteria: topics recovery rate, dominant topics identification, and group-wise topic separation capability. All evaluation metrics are reported after 50 times of repeated experiments under each parameter setting. For demonstration purpose, we set $N=1000$, $V=200$, $T \in \{3,5,8\}$, $K \in \{3,5,8\}$, and generate topic proportions under non-linear function setting. Additional parameter specifications and ablation studies are provided in Supplemental Materials section 3 \citep{YeSupplemental}. 

\subsection{Topics recovery rate}
To assess the model topic recovery rate, we compute the averaged empirical KL-divergence between the model estimated topic distributions $\hat{\beta}$ and ground-truth simulated topics $\beta$, i.e.,
\begin{equation}
\widehat{\mathbb{KL}}(\hat{\bm{\beta}}_{1:T} || \beta) 
= \frac{1}{K \cdot T} \sum_{t=1}^T \sum_{k=1}^K \sum_{v=1}^V \hat{\beta}_{t,v,k} \cdot \log(\beta_{v,k}/ \hat{\beta}_{t,v,k}),
\end{equation}
where a lower KL divergence indicates a smaller distributional-wise discrepancy and is more desired. However, due to the nature of unsupervised learning, the correspondence between the predicted and ground-true topics is undetermined. To incorporate this, we find the best permutation order $\mathcal{O}^*$ of the predicted topics, which reaches the minimum KL divergence metric, i.e., $\mathcal{O}^*_t = \argmin_{\mathcal{O}_t \in \text{perm}(1,..,K)} \widehat{\mathbb{KL}} (\hat{\beta}_{t, 1:V,\mathcal{O}_t} || \beta_{1:V,1:K})$, as the recovered correspondence at each time stage $t$. Following the above procedure, we summarize the obtained KL metrics under each simulation setting in Table \ref{tab:sim-topics}.

\begin{table}[H]
    \centering
        \def\arraystretch{1.5}%
    \resizebox{0.98\textwidth}{!}{
\begin{tabular}{|l|l|lllllll|}
\hline
  K & T &             LDA &            sLDA &         prodLDA &         SCHOLAR &            mdLDA &                   HCF-DTM &          Imp-rate \\
\hline
3 & 3 &  15.326 (2.363) &  19.319 (0.153) &  22.066 (0.030) &  15.663 (0.783) &   9.633 (5.594) &   \textbf{3.990} (0.429) &  \textbf{58.580}\% \\
  & 5 &  15.243 (2.278) &  19.274 (0.207) &  22.071 (0.024) &  15.298 (0.528) &  11.007 (5.933) &   \textbf{3.805} (0.390) &  \textbf{65.431}\% \\
  & 8 &  15.334 (2.327) &  19.293 (0.152) &  22.072 (0.015) &  15.022 (0.566) &  14.438 (2.683) &   \textbf{3.532} (0.395) &  \textbf{75.537}\% \\ \hline
5 & 3 &  18.976 (1.811) &  23.154 (0.825) &  27.007 (0.016) &  21.843 (0.491) &  13.042 (3.294) &  \textbf{11.772} (1.454) &   \textbf{9.738}\% \\
  & 5 &  19.051 (1.864) &  23.312 (0.711) &  27.002 (0.015) &  21.662 (0.615) &  16.773 (3.001) &   \textbf{9.017} (2.460) &  \textbf{46.241}\% \\
  & 8 &  19.037 (1.868) &  23.066 (0.695) &  27.003 (0.013) &  21.548 (0.596) &  18.708 (2.077) &   \textbf{8.740} (2.877) &  \textbf{53.282}\% \\ \hline
8 & 3 &  23.384 (1.117) &  27.544 (1.077) &  31.004 (0.012) &  27.097 (0.252) &  23.737 (0.606) &  \textbf{17.945} (1.629) &  \textbf{23.259}\% \\
  & 5 &  23.452 (1.051) &  27.577 (0.578) &  30.999 (0.010) &  26.971 (0.239) &  24.940 (0.670) &  \textbf{15.280} (1.939) &  \textbf{34.846}\% \\
  & 8 &  23.592 (1.051) &  27.815 (0.460) &  31.001 (0.009) &  26.763 (0.287) &  26.491 (0.543) &  \textbf{13.068} (2.041) &  \textbf{44.608}\% \\
\hline
\end{tabular}
}
\caption{Empirical KL divergence between the estimated and actual topic word distributions when the generative prior function is non-linear. Standard errors are summarized in the parentheses next to the estimated means. The improvement rate compares HCF-DTM against the best performer of the competing methods.}
\label{tab:sim-topics}
\end{table}

As shown in Table \ref{tab:sim-topics}, HCF-DTM outperforms all other competing methods with respect to the topic distribution KL divergence. In particular, under a fixed number of topics, HCF-DTM obtains a better topics recovery rate and the improvement margin enlarges compared to the best-performing competing method (mdLDA) when the number of time stages increases. This illustrates the advantage of HCF-DTM which assumes a time-constant topic distribution by design. Such an assumption minimizes the number of inference parameters and eases the optimization procedure. Notably, though the single-stage topic model could adaptively search the optimal topics at each time stage, our time-consistent assumption, in alignment with the generative process, allows HCF-DTM to fully utilize documents across all time stages, and thus, makes it powerful to recover underlying consistent topics if there exists any. 

\subsection{Dominant topic identification}
\label{sec: simulation-deominant-topic}
In this subsection, we test model inference performance on topic proportions and evaluations according to the dominant topic alignment accuracy, i.e.
\begin{equation}
    \widehat{\text{ACC}}_{\text{dom-topic}} = \frac{1}{T \cdot N} \sum_{t=1}^T \sum_{i=1}^N  \mathbb{I}\left\{\argmax(\hat{\theta}_{t,i,1:K}) = \argmax(\theta_{t,i,1:K})\right\}.
\end{equation}
A topic is considered dominant when its probabilistic proportion is the largest among all other topics for each document, and therefore, accurately identifying the dominant topics is essential as the number of dominant topics at each time stage could be utilized to track the topic's popularity. After updating the order of the estimated topics with the sequential correspondence from previous simulations, we provide simulation results for $N=1000$ and $K=8$ in Figure \ref{fig:dominant-acc-sim}.

\begin{figure}[hbt!]
    \centering
    \includegraphics[width=0.8\textwidth]{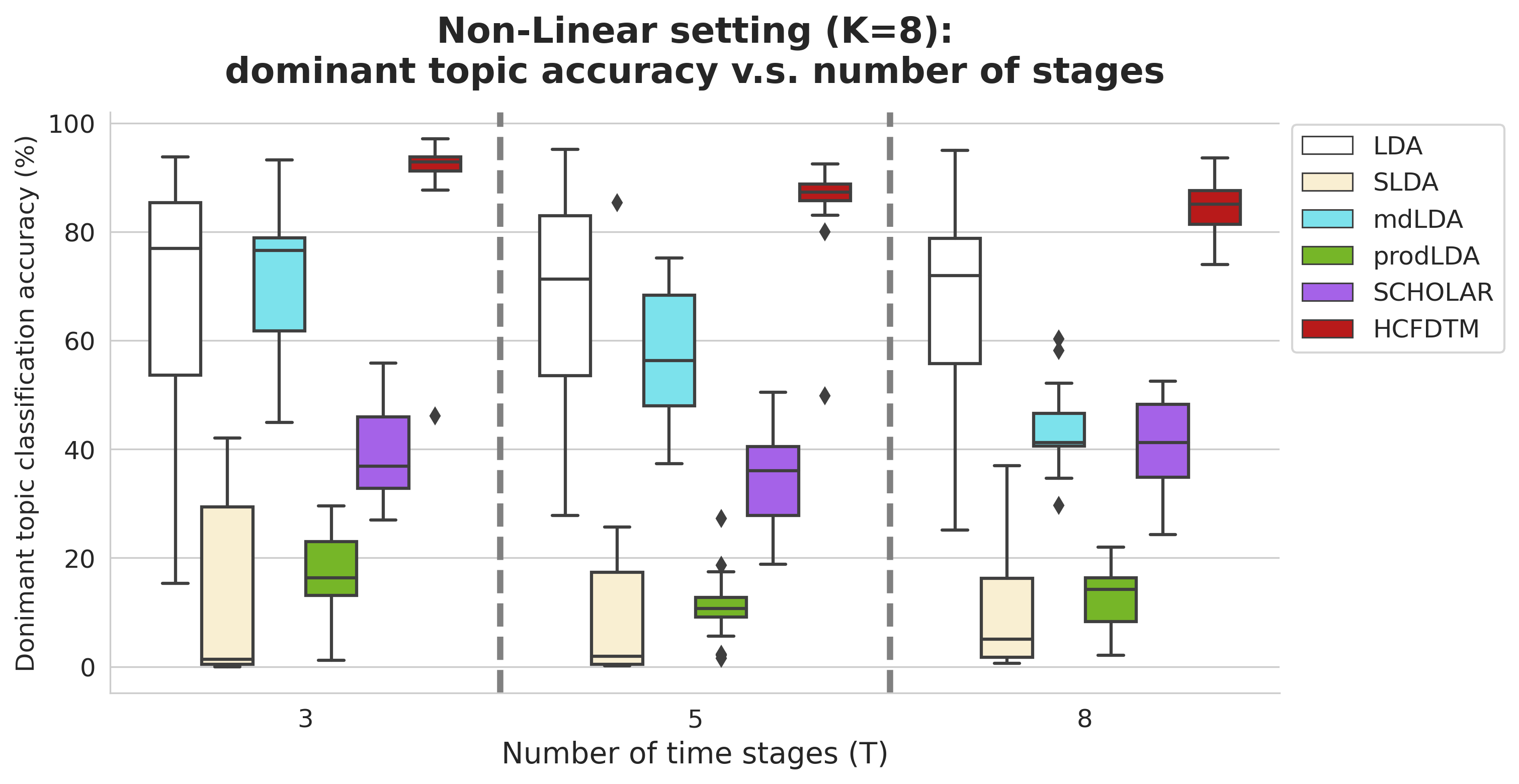}
    \caption{Boxplots of the dominant topic accuracy from estimated topic proportions versus the number of time stages when N=1000, K=8, where the generative prior function is non-linear. The left-to-right order of boxplot methods matches the top-to-bottom order of the legend.}
    \label{fig:dominant-acc-sim}
\end{figure}

Based on Figure \ref{fig:dominant-acc-sim}, we notice that the proposed HCF-DTM method obtains the highest averaged dominant topic identification accuracy, and meanwhile, achieves the smallest standard errors, compared to all the rest competing methods. This result is expected due to the following three reasons. First, compared to the single-stage methods, the proposed HCF-DTM incorporates the dynamic document metadata, which utilizes the topic proportion generative function with the subject-level longitudinal information and is able to control potential confounders. Secondly, instead of letting the topic proportions share a common corpus-level hyper-parameter as mdLDA, our method directly establishes temporal dependency on topic proportions from previous stages so that the information on dominant topics can be propagated into the future stages. Lastly, due to the mutual dependency between the topics and corresponding proportions, the topic proportions generated by HCF-DTM can be enhanced from the previously best-performing estimated topic distributions. 

\subsection{Group-wise topic separation}
\label{sec:simulation-group}
In this simulation, we examine the model capability of distinguishing topic proportions based on the subjects' group membership. In terms of evaluation metrics, we fit a simple logistic regression on the topic proportion posterior to predict subjects' group membership, i.e., $Y_i \sim \text{Logistic}(\hat{\theta}_{t,i,1:K})$ and use the resulting classification accuracy to evaluate the modeling of group-wise topic separation capability. As the classification task is directly applied to the obtained topic proportions, a higher accuracy score indicates more group-wise information having been incorporated within the proportions and thus more desired. According to this criteria, we show the simulation results in Figure \ref{fig:topic-separation-sim}.

\begin{figure}[hbt!]
    \centering
    \includegraphics[width=0.8\textwidth]{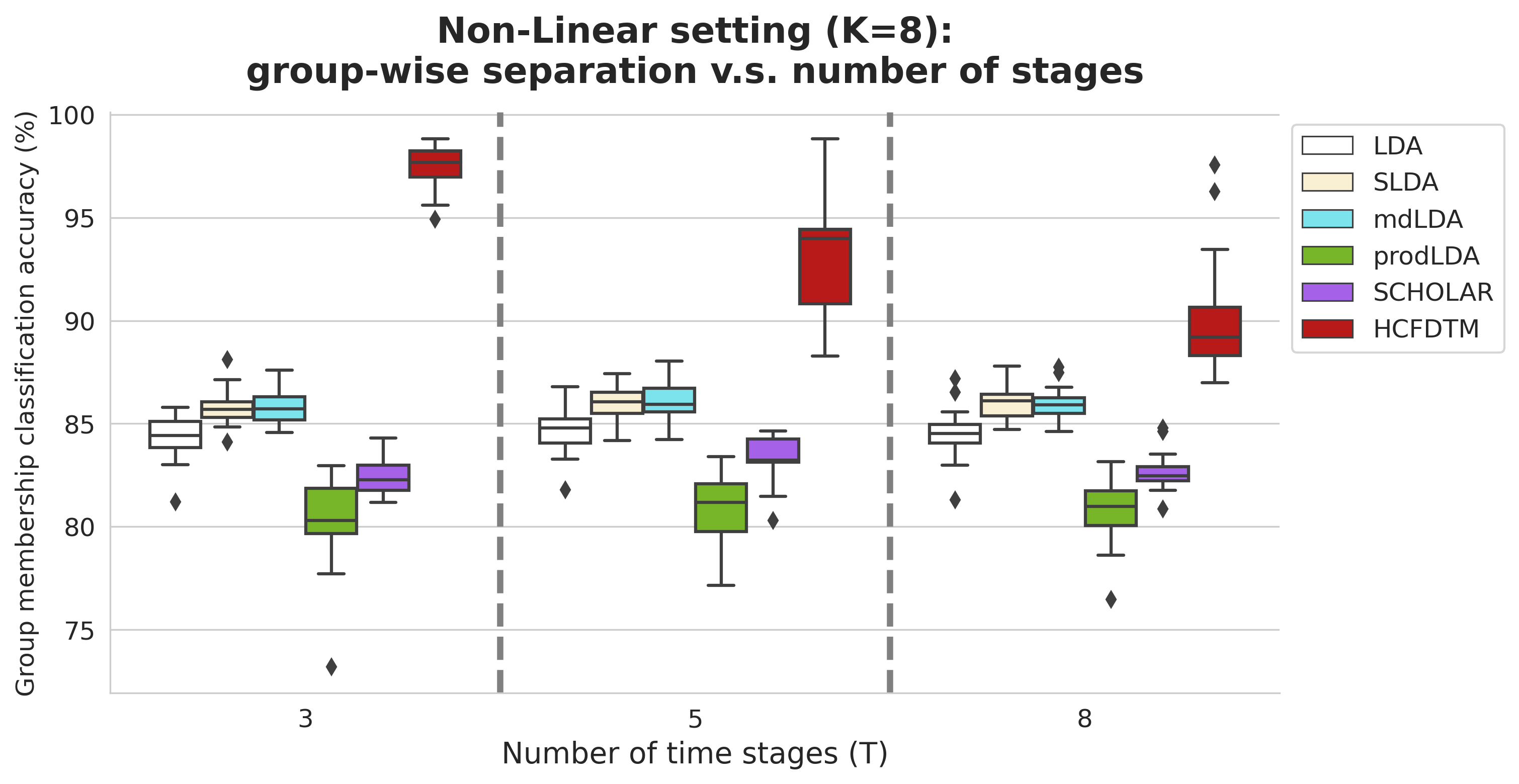}
    \caption{Boxplots of the group-membership accuracy from estimated topic proportions versus the number of time stages when N=1000, K=8, and the generative prior function is non-linear. The left-to-right order of boxplot methods matches the top-to-bottom order of the legend.}
    \label{fig:topic-separation-sim}
\end{figure}

From Figure \ref{fig:topic-separation-sim}, we observe that while the HCF-DTM model performance drops as the number of time stages increases, HCF-DTM still outperforms the rest of competing methods under each simulation setting. In particular, when there are three number of stages ($T=3$), HCF-DTM improves the group-membership accuracy most substantially due to the small number of inference tasks. More importantly, according to the simulation results, we conclude that the proposed classifier-free approach to maximize the topic proportion differences among different groups is more efficient in embedding group-wise heterogeneity to the topic proportions compared to the additional classifiers imposed by the SCHOLAR and sLDA.


\section{Real data analysis}
\label{sec:real-data}

In this section, we apply the proposed topic model HCF-DTM to the clinical notes provided by the pediatric mental health department of the Children's Health of Orange County (CHOC). Over a five-year period from 01/10/2019 to 06/01/2023, CHOC has collected a total number of 25,957 notes from 2,564 inpatient children with an average age of 14.48 years. These notes contain detailed contextual information about a patient such as hospitalization reasons, initial nursing assessment, and conversations/interactions with the clinicians. The goal of this study is to understand the progression of children's mental health from the text data during COVID-19. In particular, we investigate any disparities in the SGM and non-SGM children's mental responses to stress factors associated with the pandemic.

In our analysis, we first consolidate the time frame of interest by selecting three important events based on the timeline of COVID-19 \citep{cdc}, which are 03/15/2020 when the states began to implement shutdown, 05/13/2021 when the vaccines were released and became available to the 12+ teenagers, and 08/01/2021 when the school reopen policy was announced. These three events segment the collection of notes into four distinct time periods and are able to formulate the task into a four-stage longitudinal topic modeling setting. Additionally, as the sexual and gender identity of a patient is undisclosed, we search terms like ``pronouns'' and ``transgender'' from the notes guided by the LGBTQIA+ glossary \citep{lgbtqterm}, which reveals more than one-third ($36\%$) of the children in our dataset potentially belonging to SGM. Remarkably, this number is much larger compared to the $2.2\%$-$4.0\%$ survey estimates of SGM proportion in the United States \citep{gates2014lgbt}. 


Within each specified time period, we pre-process the clinical notes according to standard procedure, including stemming, lemmatization, and removal of stop words. However, there still remain significant challenges. Due to the unstructured and distinctive characteristics of clinical notes, we encounter frequent spelling errors and extensive usage of medical terminology abbreviations, which largely increase the vocabulary size. To maintain a more manageable bag of words, we select terms which can be found in more than one-third of the notes. This filtering process results in a set of 273 instead of 37,890 number of unique words. Together with the tabular patients' demographic information and vital measurements, we apply the proposed HCF-DTM and competing methods to extract consistent themes from the notes to investigate inpatient children's mental status over the four stages of the pandemic.

In our application, all methods are evaluated under two criteria: the accuracy in predicting children's sexual and gender identity based on the inferred topic proportions, and the UCI coherence score \citep{newman2010evaluating} for the extracted topics. Specifically, the first measures the extent of group-wise heterogeneity captured by the topic proportions, and the later score quantifies the semantic interpretability of the topics. To comprehensively analyze model performance, we randomly select 80\% of the notes as a training set and repeat the process 20 times to obtain a Monte-Carlo sample of the model performance scores. The results are summarized in Table \ref{tab:DA}, where higher accuracy and coherence scores indicate better model performance. 

\begin{table}[H]
    \centering
        \def\arraystretch{1.5}%
    \resizebox{0.98\textwidth}{!}{
\begin{tabular}{|l|l|llllll|}
\hline
           & K &              LDA &             sLDA &          prodLDA &          SCHOLAR &             mdLDA &           HCF-DTM \\
\hline
Accuracy & 3 &  62.979\% (1.627) &  65.119\% (2.181) &  62.060\% (1.455) &  65.693\% (1.969) &  62.186\% (1.922) &  \textbf{66.200}\% (1.705) \\
              & 5 &  62.956\% (1.663) &  65.642\% (2.289) &  62.060\% (1.455) &  65.456\% (2.026) &  63.147\% (1.866) &  \textbf{68.017}\% (2.205) \\
              & 8 &  64.512\% (1.450) &  66.592\% (2.231) &  62.060\% (1.455) &  65.759\% (2.086) &  64.749\% (2.425) &  \textbf{70.615}\% (2.906) \\ \hline
Coherence & 3 &   -0.162 (0.005) &   -0.035 (0.038) &   -0.344 (0.051) &    0.096 (0.024) &   -0.202 (0.023) &    \textbf{0.152} (0.037) \\
              & 5 &   -0.091 (0.021) &    0.099 (0.012) &   -0.282 (0.044) &    0.108 (0.110) &   -0.219 (0.026) &    \textbf{0.161} (0.058) \\
              & 8 &   -0.143 (0.007) &    0.096 (0.020) &   -0.246 (0.017) &    0.164 (0.036) &   -0.116 (0.027) &    \textbf{0.181} (0.035) \\
\hline
\end{tabular}
}
\caption{Model performance scores evaluated on the testing set of Monte-Carlo samples. Each test set consists of 1,587 notes from 564 patients. Standard errors are summarized in the parentheses next to the estimated means.}
\label{tab:DA}
\end{table}

According to Table \ref{tab:DA}, our HCF-DTM outperforms all other methods under each specification of the number of topics. Compared to the supervised models (sLDA and SCHOLAR) which rely on additional classifiers, our increased accuracy scores demonstrate that the introduced group-wise topic separation component of our model can be more effective in incorporating the sexual and gender identities of patients into the inferred topic proportions. Meanwhile, our augmentation of longitudinal covariate information further improves the obtained topics by capturing more enriched time-dependent heterogeneity among subjects. Furthermore, the proposed model's generative assumption of time-invariant topics, as opposed to the dynamically changing topics found by mdLDA, significantly enhances the interpretability of the topics, indicated by the higher coherence scores. We summarize that our approach, combining document longitudinal metadata and group membership, efficiently addresses existing heterogeneity and thus attains the best performance in this data application.

Furthermore, with the highest coherence score achieved, we proceed to retrain HCF-DTM using the entire clinical note dataset with three specified topics for demonstration purposes. Detailed justification of this choice can be found in Supplemental Materials section 4.3 \citep{YeSupplemental}. Figure \ref{fig:wordcloud} displays the wordclouds of the extracted topics that persist over time. Upon analyzing the word distributions, we assign the label ``Interaction" to the first topic due to the recurring words such as ``parent", ``school", and ``interaction" which describe the social support related to inpatient children. Similarly, we name the second and third topics ``Positive" and ``Negative" respectively to reflect their positive or negative emotion and feelings. Finally, we identify the dominant topic with the highest topic proportion for each patient and time period. The dominant topic, which reflects a patient's primary mental status, illustrates the evolution of inpatient children's mental health over time. We first visualize the progression of the dominant topic proportions on a two-dimensional latent space extracted by t-SNE \citep{van2008visualizing} in Figure \ref{fig:tSNE}.

\begin{figure}[hbt]
    \centering
    \includegraphics[width=\textwidth]{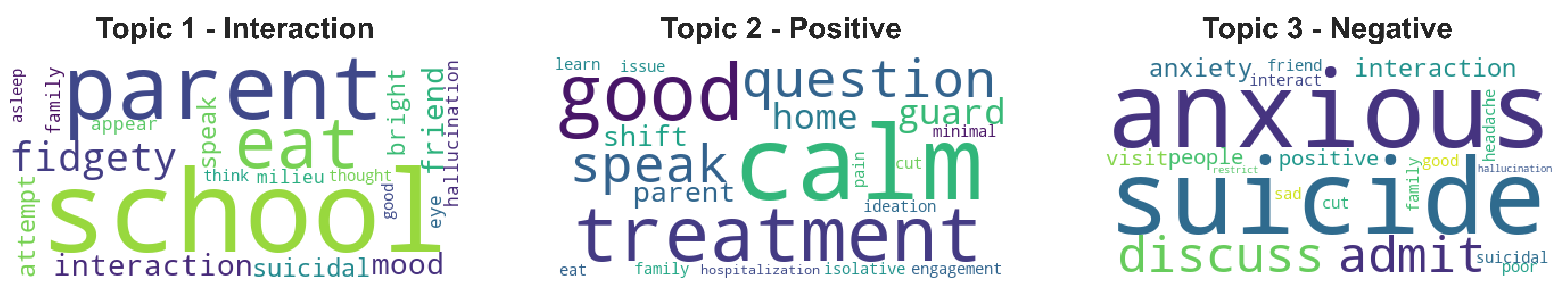}
    \caption{Wordclouds of three topics extracted by HCF-DTM.}
    \label{fig:wordcloud}
\end{figure}

A direct observation of Figure \ref{fig:tSNE} reveals that each topic is well-separated and their relative positions on the two-dimensional latent space change dynamically across the four time periods. Notably, during the implementation of shutdown measures, the inter-distances between three topics decrease while their variations increase, suggesting a significant impact on topic distributions from the shutdown event. Similarly, there is an increase in the ``Negative" topic's variance in the latent space upon the reopening of schools. To further investigate how each topic progresses over the four COVID periods, we compute the percentage changes of the dominant topics and illustrate it in Figure \ref{fig:topic-progression}.

\begin{figure}[hbt]
    \centering
    \includegraphics[width=\textwidth]{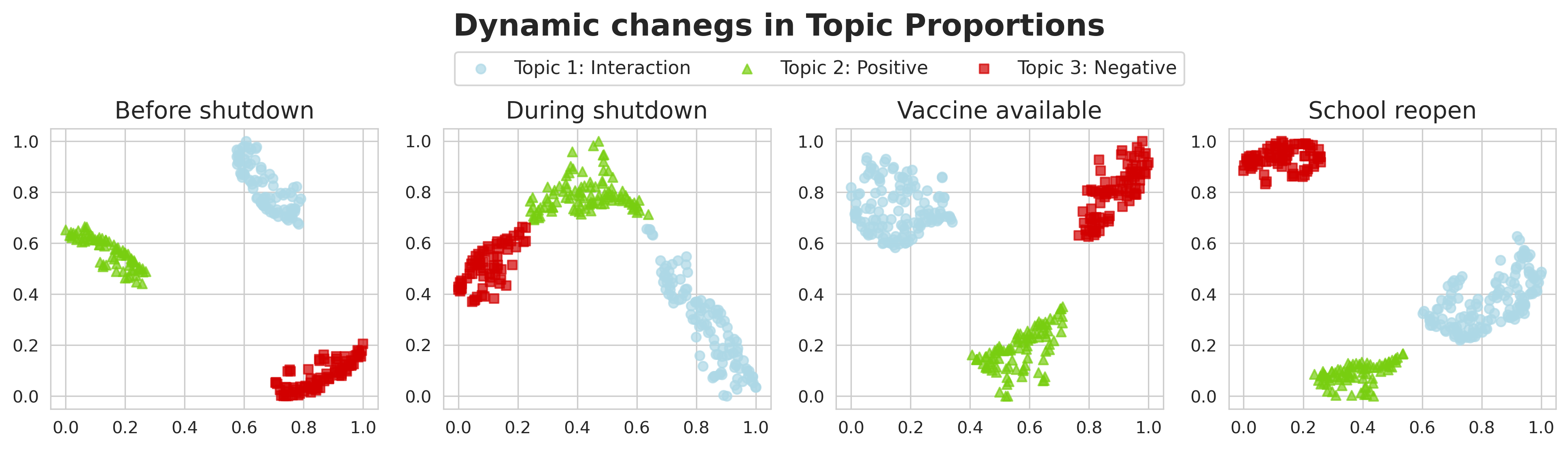}
    \caption{Dynamic changes in the topic proportions on a latent space extracted by the t-SNE. Each shape represents the corresponding dominant topic.}
    \label{fig:tSNE}
\end{figure}

Figure \ref{fig:topic-progression} shows there is an increase in the prevalence of ``Negative" emotions and a decrease in the ``Interaction" topics among both SGM and non-SGM children after the implementation of state shutdowns; whereas both trends reverse once schools began to reopen. In particular, compared to non-SGM children, SGM inpatient children exhibit more pronounced shifts, and their ``Negative" emotions started to decline one stage earlier when news of vaccine availability was released. These disparities in topic trajectories suggest the existence of heterogeneity associated with children's sexual and gender identity. Importantly, our findings also align with recent research which shows that social isolation due to quarantine measures has led to elevated levels of anxiety but decreased vaccine hesitancy in the SGM community during COVID-19 \citep{pharr2022impact, adzrago2023associations}.

\begin{figure}[h]
    \centering
    \includegraphics[width=0.9\textwidth]{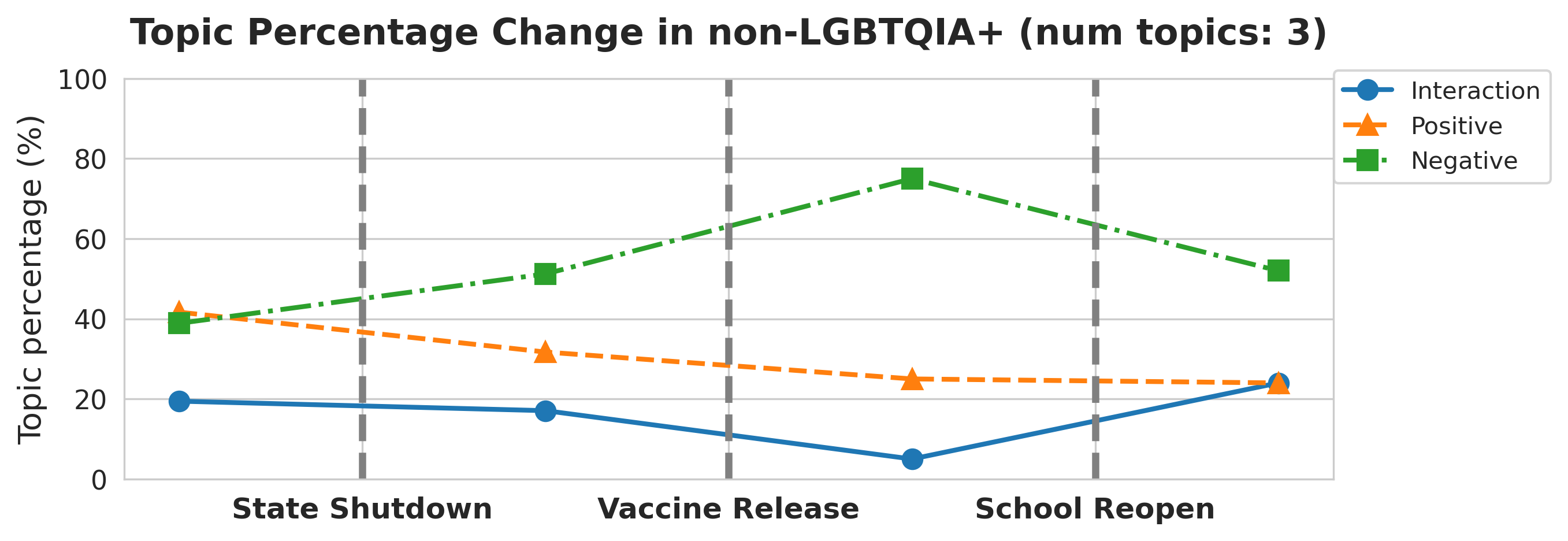}
    \includegraphics[width=0.9\textwidth]{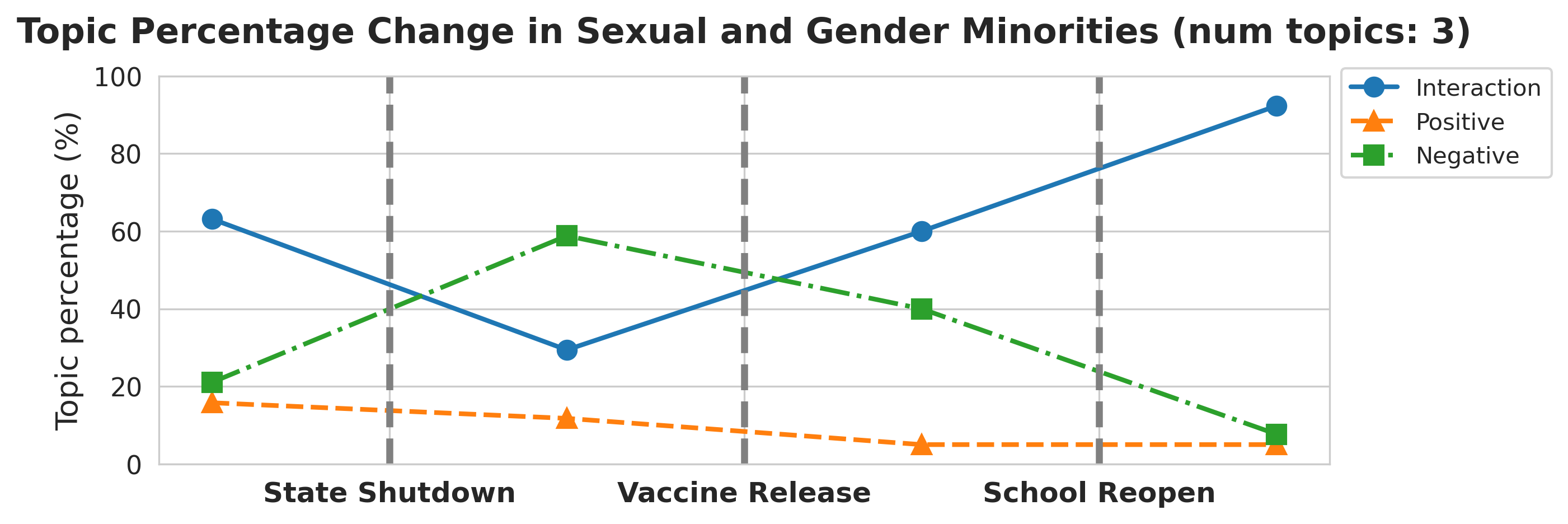}
    \caption{Percentage changes in the three topics fitted by HCFDTM within two patient groups according to sexual and gender identities. }
    \label{fig:topic-progression}
\end{figure}

To conclude, the proposed HCF-DTM method effectively extracts interpretable topics from the unstructured clinical notes. Our classifier-free group-wise separation approach reveals noticeable disparities in the progression of mental status between SGM and non-SGM children. This study effectively demonstrates the feasibility of using clinical notes to evaluate children's mental status and provides valuable insights for clinicians.

\section{Discussion} \label{sec:Discussion}

In this paper, we develop a heterogeneous dynamic topic model with an efficient variational inference procedure. The proposed model is designed to extract consistent topics in a multi-stage longitudinal setting. Specifically, our method maintains a set of time-invariant topics and incorporates document metadata into the topic proportions, where the first preserves the semantic meaning of each topic and the second captures the document temporal heterogeneity. In addition, when the documents can be categorized, we introduce a classifier-free topic learning approach that utilizes counterfactual topic distribution and inter-distributional distances to maximize topic heterogeneity across different document groups.

The proposed topic model is applied to clinical notes data from a large tertiary pediatric hospital in Southern California to evaluate inpatient children's mental health concerning their sexual and gender identities during the COVID-19 pandemic. We demonstrate three unique advantages of our method in this data application. First, without the need to navigate through a multitude of similar topics at each time stage, the extracted time-invariable topics readily represent children's mental status over time, and enable us to quantify the children's mental health progression via the corresponding topic proportions. Second, the augmentation of documents metadata can efficiently incorporate the heterogeneity among inpatient children, such as their demographics features and evolving vital measurements. Lastly, our classifier-free group-wise heterogeneity maximization approach can effectively identify any disparity in children's mental status related to their sexual and gender identities. Importantly, our model can be applied to other scenarios using topic modeling for longitudinal text data, particularly in the presence of heterogeneity.

Our real data analysis indicates that children tend to express more negative emotions during the state shutdowns and more positive when schools reopen. In particular, SGM children exhibit more pronounced reactions towards major COVID-19 events and greater sensitivity to vaccine-related news. This implies that increased social isolation due to enforced quarantine has a noticeable negative impact on children's mental health, especially among SGM children. As a result, engaging more social activities and support can be crucial for children's mental well-being, and our study may facilitate clinicians to understand the importance of rebuilding social connections and activities as a post-pandemic support system for children to recover from mental distress.

\begin{supplement}

\stitle{Supplementary Materials to ``Dynamic Topic Language Model on heterogeneous Children's Mental Health Clinical Notes"}

\sdescription{The supplement contains variational inference proof, details of the optimization algorithm, additional simulations and ablation studies, and descriptions of real-world applications along with an examination of model assumptions.}

\end{supplement}

{
\small
\bibliographystyle{imsart-nameyear} 
\bibliography{bibliography}       

\begin{thebibliography}{57}

\bibitem[\protect\citeauthoryear{Adzrago et~al.}{2023}]{adzrago2023associations}
\begin{barticle}[author]
\bauthor{\bsnm{Adzrago},~\bfnm{David}\binits{D.}}, \bauthor{\bsnm{Ormiston},~\bfnm{Cameron~K}\binits{C.~K.}}, \bauthor{\bsnm{Sulley},~\bfnm{Saanie}\binits{S.}} \AND \bauthor{\bsnm{Williams},~\bfnm{Faustine}\binits{F.}}
(\byear{2023}).
\btitle{Associations between the Self-Reported Likelihood of Receiving the COVID-19 Vaccine, Likelihood of Contracting COVID-19, Discrimination, and Anxiety/Depression by Sexual Orientation}.
\bjournal{Vaccines}
\bvolume{11}
\bpages{582}.
\end{barticle}
\endbibitem

\bibitem[\protect\citeauthoryear{Afifi}{2007}]{afifi2007gender}
\begin{barticle}[author]
\bauthor{\bsnm{Afifi},~\bfnm{Mustafa}\binits{M.}}
(\byear{2007}).
\btitle{Gender differences in mental health}.
\bjournal{Singapore medical journal}
\bvolume{48}
\bpages{385}.
\end{barticle}
\endbibitem

\bibitem[\protect\citeauthoryear{Barry}{2014}]{barry2014midlife}
\begin{barticle}[author]
\bauthor{\bsnm{Barry},~\bfnm{T~Radler}\binits{T.~R.}}
(\byear{2014}).
\btitle{The Midlife in the United States (MIDUS) series: A national longitudinal study of health and well-being}.
\bjournal{Open health data}
\bvolume{2}.
\end{barticle}
\endbibitem

\bibitem[\protect\citeauthoryear{Blei and Lafferty}{2006a}]{blei2006correlated}
\begin{barticle}[author]
\bauthor{\bsnm{Blei},~\bfnm{David}\binits{D.}} \AND \bauthor{\bsnm{Lafferty},~\bfnm{John}\binits{J.}}
(\byear{2006}a).
\btitle{Correlated topic models}.
\bjournal{Advances in neural information processing systems}
\bvolume{18}
\bpages{147}.
\end{barticle}
\endbibitem

\bibitem[\protect\citeauthoryear{Blei and Lafferty}{2006b}]{blei2006dynamic}
\begin{binproceedings}[author]
\bauthor{\bsnm{Blei},~\bfnm{David~M}\binits{D.~M.}} \AND \bauthor{\bsnm{Lafferty},~\bfnm{John~D}\binits{J.~D.}}
(\byear{2006}b).
\btitle{Dynamic topic models}.
In \bbooktitle{Proceedings of the 23rd international conference on Machine learning}
\bpages{113--120}.
\end{binproceedings}
\endbibitem

\bibitem[\protect\citeauthoryear{Blei, Ng and Jordan}{2003}]{blei2003latent}
\begin{barticle}[author]
\bauthor{\bsnm{Blei},~\bfnm{David~M}\binits{D.~M.}}, \bauthor{\bsnm{Ng},~\bfnm{Andrew~Y}\binits{A.~Y.}} \AND \bauthor{\bsnm{Jordan},~\bfnm{Michael~I}\binits{M.~I.}}
(\byear{2003}).
\btitle{Latent dirichlet allocation}.
\bjournal{Journal of machine Learning research}
\bvolume{3}
\bpages{993--1022}.
\end{barticle}
\endbibitem

\bibitem[\protect\citeauthoryear{Boyd et~al.}{2013}]{boyd2013cohort}
\begin{barticle}[author]
\bauthor{\bsnm{Boyd},~\bfnm{Andy}\binits{A.}}, \bauthor{\bsnm{Golding},~\bfnm{Jean}\binits{J.}}, \bauthor{\bsnm{Macleod},~\bfnm{John}\binits{J.}}, \bauthor{\bsnm{Lawlor},~\bfnm{Debbie~A}\binits{D.~A.}}, \bauthor{\bsnm{Fraser},~\bfnm{Abigail}\binits{A.}}, \bauthor{\bsnm{Henderson},~\bfnm{John}\binits{J.}}, \bauthor{\bsnm{Molloy},~\bfnm{Lynn}\binits{L.}}, \bauthor{\bsnm{Ness},~\bfnm{Andy}\binits{A.}}, \bauthor{\bsnm{Ring},~\bfnm{Susan}\binits{S.}} \AND \bauthor{\bsnm{Davey~Smith},~\bfnm{George}\binits{G.}}
(\byear{2013}).
\btitle{Cohort profile: the ‘children of the 90s’—the index offspring of the Avon Longitudinal Study of Parents and Children}.
\bjournal{International journal of epidemiology}
\bvolume{42}
\bpages{111--127}.
\end{barticle}
\endbibitem

\bibitem[\protect\citeauthoryear{Card, Tan and Smith}{2017}]{card2017neural}
\begin{barticle}[author]
\bauthor{\bsnm{Card},~\bfnm{Dallas}\binits{D.}}, \bauthor{\bsnm{Tan},~\bfnm{Chenhao}\binits{C.}} \AND \bauthor{\bsnm{Smith},~\bfnm{Noah~A}\binits{N.~A.}}
(\byear{2017}).
\btitle{Neural models for documents with metadata}.
\bjournal{arXiv preprint arXiv:1705.09296}.
\end{barticle}
\endbibitem

\bibitem[\protect\citeauthoryear{Casale et~al.}{2018}]{casale2018gaussian}
\begin{barticle}[author]
\bauthor{\bsnm{Casale},~\bfnm{Francesco~Paolo}\binits{F.~P.}}, \bauthor{\bsnm{Dalca},~\bfnm{Adrian}\binits{A.}}, \bauthor{\bsnm{Saglietti},~\bfnm{Luca}\binits{L.}}, \bauthor{\bsnm{Listgarten},~\bfnm{Jennifer}\binits{J.}} \AND \bauthor{\bsnm{Fusi},~\bfnm{Nicolo}\binits{N.}}
(\byear{2018}).
\btitle{Gaussian process prior variational autoencoders}.
\bjournal{Advances in neural information processing systems}
\bvolume{31}.
\end{barticle}
\endbibitem

\bibitem[\protect\citeauthoryear{Ciechanowski, Jemielniak and Silczuk}{2023}]{ciechanowski2023public}
\begin{barticle}[author]
\bauthor{\bsnm{Ciechanowski},~\bfnm{Ka{\'s}mir}\binits{K.}}, \bauthor{\bsnm{Jemielniak},~\bfnm{Dariusz}\binits{D.}} \AND \bauthor{\bsnm{Silczuk},~\bfnm{Andrzej}\binits{A.}}
(\byear{2023}).
\btitle{Public interests in mental health topics in COVID-19: evidence from Wikipedia searches}.
\bjournal{Advances in Mental Health}
\bpages{1--22}.
\end{barticle}
\endbibitem

\bibitem[\protect\citeauthoryear{{Centers for \textsc{D}isease \textsc{C}ontrol and \textsc{P}revention}}{2023}]{cdc}
\begin{barticle}[author]
\bauthor{\bsnm{{Centers for \textsc{D}isease \textsc{C}ontrol and \textsc{P}revention}}}
(\byear{2023}).
\btitle{C\textsc{DC} \textsc{M}useum \textsc{COVID}-19 \textsc{T}imeline}.
\end{barticle}
\endbibitem

\bibitem[\protect\citeauthoryear{Fortuin et~al.}{2020}]{fortuin2020gp}
\begin{binproceedings}[author]
\bauthor{\bsnm{Fortuin},~\bfnm{Vincent}\binits{V.}}, \bauthor{\bsnm{Baranchuk},~\bfnm{Dmitry}\binits{D.}}, \bauthor{\bsnm{R{\"a}tsch},~\bfnm{Gunnar}\binits{G.}} \AND \bauthor{\bsnm{Mandt},~\bfnm{Stephan}\binits{S.}}
(\byear{2020}).
\btitle{Gp-vae: Deep probabilistic time series imputation}.
In \bbooktitle{International conference on artificial intelligence and statistics}
\bpages{1651--1661}.
\bpublisher{PMLR}.
\end{binproceedings}
\endbibitem

\bibitem[\protect\citeauthoryear{Gates}{2014}]{gates2014lgbt}
\begin{barticle}[author]
\bauthor{\bsnm{Gates},~\bfnm{Gary~J}\binits{G.~J.}}
(\byear{2014}).
\btitle{LGBT demographics: Comparisons among population-based surveys}.
\end{barticle}
\endbibitem

\bibitem[\protect\citeauthoryear{Gulrajani et~al.}{2016}]{gulrajani2016pixelvae}
\begin{barticle}[author]
\bauthor{\bsnm{Gulrajani},~\bfnm{Ishaan}\binits{I.}}, \bauthor{\bsnm{Kumar},~\bfnm{Kundan}\binits{K.}}, \bauthor{\bsnm{Ahmed},~\bfnm{Faruk}\binits{F.}}, \bauthor{\bsnm{Taiga},~\bfnm{Adrien~Ali}\binits{A.~A.}}, \bauthor{\bsnm{Visin},~\bfnm{Francesco}\binits{F.}}, \bauthor{\bsnm{Vazquez},~\bfnm{David}\binits{D.}} \AND \bauthor{\bsnm{Courville},~\bfnm{Aaron}\binits{A.}}
(\byear{2016}).
\btitle{Pixelvae: A latent variable model for natural images}.
\bjournal{arXiv preprint arXiv:1611.05013}.
\end{barticle}
\endbibitem

\bibitem[\protect\citeauthoryear{Gupta et~al.}{2019}]{gupta2019document}
\begin{binproceedings}[author]
\bauthor{\bsnm{Gupta},~\bfnm{Pankaj}\binits{P.}}, \bauthor{\bsnm{Chaudhary},~\bfnm{Yatin}\binits{Y.}}, \bauthor{\bsnm{Buettner},~\bfnm{Florian}\binits{F.}} \AND \bauthor{\bsnm{Sch{\"u}tze},~\bfnm{Hinrich}\binits{H.}}
(\byear{2019}).
\btitle{Document informed neural autoregressive topic models with distributional prior}.
In \bbooktitle{Proceedings of the AAAI Conference on Artificial Intelligence}
\bvolume{33}
\bpages{6505--6512}.
\end{binproceedings}
\endbibitem

\bibitem[\protect\citeauthoryear{Hu et~al.}{2020}]{hu2020neural}
\begin{barticle}[author]
\bauthor{\bsnm{Hu},~\bfnm{Xuemeng}\binits{X.}}, \bauthor{\bsnm{Wang},~\bfnm{Rui}\binits{R.}}, \bauthor{\bsnm{Zhou},~\bfnm{Deyu}\binits{D.}} \AND \bauthor{\bsnm{Xiong},~\bfnm{Yuxuan}\binits{Y.}}
(\byear{2020}).
\btitle{Neural topic modeling with cycle-consistent adversarial training}.
\bjournal{arXiv preprint arXiv:2009.13971}.
\end{barticle}
\endbibitem

\bibitem[\protect\citeauthoryear{Karim et~al.}{2022}]{karim2022support}
\begin{barticle}[author]
\bauthor{\bsnm{Karim},~\bfnm{Sana}\binits{S.}}, \bauthor{\bsnm{Choukas-Bradley},~\bfnm{Sophia}\binits{S.}}, \bauthor{\bsnm{Radovic},~\bfnm{Ana}\binits{A.}}, \bauthor{\bsnm{Roberts},~\bfnm{Savannah~R}\binits{S.~R.}}, \bauthor{\bsnm{Maheux},~\bfnm{Anne~J}\binits{A.~J.}} \AND \bauthor{\bsnm{Escobar-Viera},~\bfnm{C{\'e}sar~G}\binits{C.~G.}}
(\byear{2022}).
\btitle{Support over social media among socially isolated sexual and gender minority youth in rural US during the COVID-19 pandemic: opportunities for intervention research}.
\bjournal{International Journal of Environmental Research and Public Health}
\bvolume{19}
\bpages{15611}.
\end{barticle}
\endbibitem

\bibitem[\protect\citeauthoryear{Larochelle and Lauly}{2012}]{larochelle2012neural}
\begin{barticle}[author]
\bauthor{\bsnm{Larochelle},~\bfnm{Hugo}\binits{H.}} \AND \bauthor{\bsnm{Lauly},~\bfnm{Stanislas}\binits{S.}}
(\byear{2012}).
\btitle{A neural autoregressive topic model}.
\bjournal{Advances in Neural Information Processing Systems}
\bvolume{25}.
\end{barticle}
\endbibitem

\bibitem[\protect\citeauthoryear{Lee and Seung}{1999}]{lee1999learning}
\begin{barticle}[author]
\bauthor{\bsnm{Lee},~\bfnm{Daniel~D}\binits{D.~D.}} \AND \bauthor{\bsnm{Seung},~\bfnm{H~Sebastian}\binits{H.~S.}}
(\byear{1999}).
\btitle{Learning the parts of objects by non-negative matrix factorization}.
\bjournal{Nature}
\bvolume{401}
\bpages{788--791}.
\end{barticle}
\endbibitem

\bibitem[\protect\citeauthoryear{Li, Ouyang and Zhou}{2015}]{li2015supervised}
\begin{barticle}[author]
\bauthor{\bsnm{Li},~\bfnm{Ximing}\binits{X.}}, \bauthor{\bsnm{Ouyang},~\bfnm{Jihong}\binits{J.}} \AND \bauthor{\bsnm{Zhou},~\bfnm{Xiaotang}\binits{X.}}
(\byear{2015}).
\btitle{Supervised topic models for multi-label classification}.
\bjournal{Neurocomputing}
\bvolume{149}
\bpages{811--819}.
\end{barticle}
\endbibitem

\bibitem[\protect\citeauthoryear{Li et~al.}{2021}]{li2021topic}
\begin{barticle}[author]
\bauthor{\bsnm{Li},~\bfnm{Yutong}\binits{Y.}}, \bauthor{\bsnm{Zhu},~\bfnm{Ruoqing}\binits{R.}}, \bauthor{\bsnm{Qu},~\bfnm{Annie}\binits{A.}}, \bauthor{\bsnm{Ye},~\bfnm{Han}\binits{H.}} \AND \bauthor{\bsnm{Sun},~\bfnm{Zhankun}\binits{Z.}}
(\byear{2021}).
\btitle{Topic modeling on triage notes with semiorthogonal nonnegative matrix factorization}.
\bjournal{Journal of the American Statistical Association}
\bvolume{116}
\bpages{1609--1624}.
\end{barticle}
\endbibitem

\bibitem[\protect\citeauthoryear{Lin, Hu and Guo}{2019}]{lin2019sparsemax}
\begin{binproceedings}[author]
\bauthor{\bsnm{Lin},~\bfnm{Tianyi}\binits{T.}}, \bauthor{\bsnm{Hu},~\bfnm{Zhiyue}\binits{Z.}} \AND \bauthor{\bsnm{Guo},~\bfnm{Xin}\binits{X.}}
(\byear{2019}).
\btitle{Sparsemax and relaxed wasserstein for topic sparsity}.
In \bbooktitle{Proceedings of the twelfth ACM international conference on web search and data mining}
\bpages{141--149}.
\end{binproceedings}
\endbibitem

\bibitem[\protect\citeauthoryear{Marshal et~al.}{2011}]{marshal2011suicidality}
\begin{barticle}[author]
\bauthor{\bsnm{Marshal},~\bfnm{Michael~P}\binits{M.~P.}}, \bauthor{\bsnm{Dietz},~\bfnm{Laura~J}\binits{L.~J.}}, \bauthor{\bsnm{Friedman},~\bfnm{Mark~S}\binits{M.~S.}}, \bauthor{\bsnm{Stall},~\bfnm{Ron}\binits{R.}}, \bauthor{\bsnm{Smith},~\bfnm{Helen~A}\binits{H.~A.}}, \bauthor{\bsnm{McGinley},~\bfnm{James}\binits{J.}}, \bauthor{\bsnm{Thoma},~\bfnm{Brian~C}\binits{B.~C.}}, \bauthor{\bsnm{Murray},~\bfnm{Pamela~J}\binits{P.~J.}}, \bauthor{\bsnm{D'Augelli},~\bfnm{Anthony~R}\binits{A.~R.}} \AND \bauthor{\bsnm{Brent},~\bfnm{David~A}\binits{D.~A.}}
(\byear{2011}).
\btitle{Suicidality and depression disparities between sexual minority and heterosexual youth: A meta-analytic review}.
\bjournal{Journal of adolescent health}
\bvolume{49}
\bpages{115--123}.
\end{barticle}
\endbibitem

\bibitem[\protect\citeauthoryear{Mcauliffe and Blei}{2007}]{mcauliffe2007supervised}
\begin{barticle}[author]
\bauthor{\bsnm{Mcauliffe},~\bfnm{Jon}\binits{J.}} \AND \bauthor{\bsnm{Blei},~\bfnm{David}\binits{D.}}
(\byear{2007}).
\btitle{Supervised topic models}.
\bjournal{Advances in neural information processing systems}
\bvolume{20}.
\end{barticle}
\endbibitem

\bibitem[\protect\citeauthoryear{McGeough and Sterzing}{2018}]{mcgeough2018systematic}
\begin{barticle}[author]
\bauthor{\bsnm{McGeough},~\bfnm{Briana~L}\binits{B.~L.}} \AND \bauthor{\bsnm{Sterzing},~\bfnm{Paul~R}\binits{P.~R.}}
(\byear{2018}).
\btitle{A systematic review of family victimization experiences among sexual minority youth}.
\bjournal{The Journal of Primary Prevention}
\bvolume{39}
\bpages{491--528}.
\end{barticle}
\endbibitem

\bibitem[\protect\citeauthoryear{McGregor et~al.}{2023}]{mcgregor2023providing}
\begin{barticle}[author]
\bauthor{\bsnm{McGregor},~\bfnm{Kerry}\binits{K.}}, \bauthor{\bsnm{Williams},~\bfnm{Coleen~R}\binits{C.~R.}}, \bauthor{\bsnm{Botta},~\bfnm{Ariel}\binits{A.}}, \bauthor{\bsnm{Mandel},~\bfnm{Francie}\binits{F.}} \AND \bauthor{\bsnm{Gentile},~\bfnm{Jennifer}\binits{J.}}
(\byear{2023}).
\btitle{Providing essential gender-affirming telehealth services to transgender youth during COVID-19: A service review}.
\bjournal{Journal of Telemedicine and Telecare}
\bvolume{29}
\bpages{147--152}.
\end{barticle}
\endbibitem

\bibitem[\protect\citeauthoryear{Miao, Yu and Blunsom}{2016}]{miao2016neural}
\begin{binproceedings}[author]
\bauthor{\bsnm{Miao},~\bfnm{Yishu}\binits{Y.}}, \bauthor{\bsnm{Yu},~\bfnm{Lei}\binits{L.}} \AND \bauthor{\bsnm{Blunsom},~\bfnm{Phil}\binits{P.}}
(\byear{2016}).
\btitle{Neural variational inference for text processing}.
In \bbooktitle{International conference on machine learning}
\bpages{1727--1736}.
\bpublisher{PMLR}.
\end{binproceedings}
\endbibitem

\bibitem[\protect\citeauthoryear{Newman et~al.}{2010}]{newman2010evaluating}
\begin{binproceedings}[author]
\bauthor{\bsnm{Newman},~\bfnm{David}\binits{D.}}, \bauthor{\bsnm{Noh},~\bfnm{Youn}\binits{Y.}}, \bauthor{\bsnm{Talley},~\bfnm{Edmund}\binits{E.}}, \bauthor{\bsnm{Karimi},~\bfnm{Sarvnaz}\binits{S.}} \AND \bauthor{\bsnm{Baldwin},~\bfnm{Timothy}\binits{T.}}
(\byear{2010}).
\btitle{Evaluating topic models for digital libraries}.
In \bbooktitle{Proceedings of the 10th annual joint conference on Digital libraries}
\bpages{215--224}.
\end{binproceedings}
\endbibitem

\bibitem[\protect\citeauthoryear{{National Institute of Mental Health}}{2021}]{usnational}
\begin{barticle}[author]
\bauthor{\bsnm{{National Institute of Mental Health}}}
(\byear{2021}).
\btitle{Mental health topics}.
\bjournal{from https://www.nimh.nih.gov/health/topics}.
\end{barticle}
\endbibitem

\bibitem[\protect\citeauthoryear{Paatero and Tapper}{1994}]{paatero1994positive}
\begin{barticle}[author]
\bauthor{\bsnm{Paatero},~\bfnm{Pentti}\binits{P.}} \AND \bauthor{\bsnm{Tapper},~\bfnm{Unto}\binits{U.}}
(\byear{1994}).
\btitle{Positive matrix factorization: A non-negative factor model with optimal utilization of error estimates of data values}.
\bjournal{Environmetrics}
\bvolume{5}
\bpages{111--126}.
\end{barticle}
\endbibitem

\bibitem[\protect\citeauthoryear{Penninx et~al.}{2008}]{penninx2008netherlands}
\begin{barticle}[author]
\bauthor{\bsnm{Penninx},~\bfnm{Brenda~WJH}\binits{B.~W.}}, \bauthor{\bsnm{Beekman},~\bfnm{Aartjan~TF}\binits{A.~T.}}, \bauthor{\bsnm{Smit},~\bfnm{Johannes~H}\binits{J.~H.}}, \bauthor{\bsnm{Zitman},~\bfnm{Frans~G}\binits{F.~G.}}, \bauthor{\bsnm{Nolen},~\bfnm{Willem~A}\binits{W.~A.}}, \bauthor{\bsnm{Spinhoven},~\bfnm{Philip}\binits{P.}}, \bauthor{\bsnm{Cuijpers},~\bfnm{Pim}\binits{P.}}, \bauthor{\bsnm{De~Jong},~\bfnm{Peter~J}\binits{P.~J.}}, \bauthor{\bsnm{Van~Marwijk},~\bfnm{Harm~WJ}\binits{H.~W.}}, \bauthor{\bsnm{Assendelft},~\bfnm{Willem~JJ}\binits{W.~J.}} \betal{et~al.}
(\byear{2008}).
\btitle{The Netherlands Study of Depression and Anxiety (NESDA): rationale, objectives and methods}.
\bjournal{International journal of methods in psychiatric research}
\bvolume{17}
\bpages{121--140}.
\end{barticle}
\endbibitem

\bibitem[\protect\citeauthoryear{Pharr et~al.}{2022}]{pharr2022impact}
\begin{barticle}[author]
\bauthor{\bsnm{Pharr},~\bfnm{Jennifer~R}\binits{J.~R.}}, \bauthor{\bsnm{Terry},~\bfnm{Emylia}\binits{E.}}, \bauthor{\bsnm{Wade},~\bfnm{Andr{\'e}}\binits{A.}}, \bauthor{\bsnm{Haboush-Deloye},~\bfnm{Amanda}\binits{A.}}, \bauthor{\bsnm{Marquez},~\bfnm{Erika}\binits{E.}}, \bauthor{\bsnm{Health},~\bfnm{Nevada~Minority}\binits{N.~M.}} \AND \bauthor{\bsnm{Coalition},~\bfnm{Equity}\binits{E.}}
(\byear{2022}).
\btitle{Impact of COVID-19 on Sexual and Gender Minority Communities: Focus Group Discussions}.
\bjournal{International Journal of Environmental Research and Public Health}
\bvolume{20}
\bpages{50}.
\end{barticle}
\endbibitem

\bibitem[\protect\citeauthoryear{Pl{\"o}derl and Tremblay}{2015}]{ploderl2015mental}
\begin{barticle}[author]
\bauthor{\bsnm{Pl{\"o}derl},~\bfnm{Martin}\binits{M.}} \AND \bauthor{\bsnm{Tremblay},~\bfnm{Pierre}\binits{P.}}
(\byear{2015}).
\btitle{Mental health of sexual minorities. A systematic review}.
\bjournal{International review of psychiatry}
\bvolume{27}
\bpages{367--385}.
\end{barticle}
\endbibitem

\bibitem[\protect\citeauthoryear{Ramchandran et~al.}{2021}]{ramchandran2021longitudinal}
\begin{binproceedings}[author]
\bauthor{\bsnm{Ramchandran},~\bfnm{Siddharth}\binits{S.}}, \bauthor{\bsnm{Tikhonov},~\bfnm{Gleb}\binits{G.}}, \bauthor{\bsnm{Kujanp{\"a}{\"a}},~\bfnm{Kalle}\binits{K.}}, \bauthor{\bsnm{Koskinen},~\bfnm{Miika}\binits{M.}} \AND \bauthor{\bsnm{L{\"a}hdesm{\"a}ki},~\bfnm{Harri}\binits{H.}}
(\byear{2021}).
\btitle{Longitudinal variational autoencoder}.
In \bbooktitle{International Conference on Artificial Intelligence and Statistics}
\bpages{3898--3906}.
\bpublisher{PMLR}.
\end{binproceedings}
\endbibitem

\bibitem[\protect\citeauthoryear{Ravens-Sieberer et~al.}{2022}]{ravens2022impact}
\begin{barticle}[author]
\bauthor{\bsnm{Ravens-Sieberer},~\bfnm{Ulrike}\binits{U.}}, \bauthor{\bsnm{Kaman},~\bfnm{Anne}\binits{A.}}, \bauthor{\bsnm{Erhart},~\bfnm{Michael}\binits{M.}}, \bauthor{\bsnm{Devine},~\bfnm{Janine}\binits{J.}}, \bauthor{\bsnm{Schlack},~\bfnm{Robert}\binits{R.}} \AND \bauthor{\bsnm{Otto},~\bfnm{Christiane}\binits{C.}}
(\byear{2022}).
\btitle{Impact of the COVID-19 pandemic on quality of life and mental health in children and adolescents in Germany}.
\bjournal{European child \& adolescent psychiatry}
\bvolume{31}
\bpages{879--889}.
\end{barticle}
\endbibitem

\bibitem[\protect\citeauthoryear{Robbins and Monro}{1951}]{robbins1951stochastic}
\begin{barticle}[author]
\bauthor{\bsnm{Robbins},~\bfnm{Herbert}\binits{H.}} \AND \bauthor{\bsnm{Monro},~\bfnm{Sutton}\binits{S.}}
(\byear{1951}).
\btitle{A stochastic approximation method}.
\bjournal{The Annals of Mathematical Statistics}
\bvolume{22}
\bpages{400--407}.
\end{barticle}
\endbibitem

\bibitem[\protect\citeauthoryear{Roberts et~al.}{2014}]{roberts2014structural}
\begin{barticle}[author]
\bauthor{\bsnm{Roberts},~\bfnm{Margaret~E}\binits{M.~E.}}, \bauthor{\bsnm{Stewart},~\bfnm{Brandon~M}\binits{B.~M.}}, \bauthor{\bsnm{Tingley},~\bfnm{Dustin}\binits{D.}}, \bauthor{\bsnm{Lucas},~\bfnm{Christopher}\binits{C.}}, \bauthor{\bsnm{Leder-Luis},~\bfnm{Jetson}\binits{J.}}, \bauthor{\bsnm{Gadarian},~\bfnm{Shana~Kushner}\binits{S.~K.}}, \bauthor{\bsnm{Albertson},~\bfnm{Bethany}\binits{B.}} \AND \bauthor{\bsnm{Rand},~\bfnm{David~G}\binits{D.~G.}}
(\byear{2014}).
\btitle{Structural topic models for open-ended survey responses}.
\bjournal{American journal of political science}
\bvolume{58}
\bpages{1064--1082}.
\end{barticle}
\endbibitem

\bibitem[\protect\citeauthoryear{Robins}{1986}]{robins1986new}
\begin{barticle}[author]
\bauthor{\bsnm{Robins},~\bfnm{James}\binits{J.}}
(\byear{1986}).
\btitle{A new approach to causal inference in mortality studies with a sustained exposure period—application to control of the healthy worker survivor effect}.
\bjournal{Mathematical Modelling}
\bvolume{7}
\bpages{1393--1512}.
\end{barticle}
\endbibitem

\bibitem[\protect\citeauthoryear{Ronald et~al.}{2010}]{ronald2010evolving}
\begin{barticle}[author]
\bauthor{\bsnm{Ronald},~\bfnm{W}\binits{W.}}, \bauthor{\bsnm{Carol},~\bfnm{D}\binits{D.}}, \bauthor{\bsnm{Elsie},~\bfnm{J}\binits{J.}}, \bauthor{\bsnm{Lela},~\bfnm{R}\binits{R.}}, \bauthor{\bsnm{Satvinder},~\bfnm{D}\binits{D.}} \AND \bauthor{\bsnm{Tara},~\bfnm{W}\binits{W.}}
(\byear{2010}).
\btitle{Evolving definitions of mental illness and wellness}.
\bjournal{Preventing Chronic disease}
\bvolume{7}
\bpages{2}.
\end{barticle}
\endbibitem

\bibitem[\protect\citeauthoryear{Rosenfield and Mouzon}{2013}]{rosenfield2013gender}
\begin{barticle}[author]
\bauthor{\bsnm{Rosenfield},~\bfnm{Sarah}\binits{S.}} \AND \bauthor{\bsnm{Mouzon},~\bfnm{Dawne}\binits{D.}}
(\byear{2013}).
\btitle{Gender and mental health}.
\bjournal{Handbook of the sociology of mental health}
\bpages{277--296}.
\end{barticle}
\endbibitem

\bibitem[\protect\citeauthoryear{Russell and Fish}{2016}]{russell2016mental}
\begin{barticle}[author]
\bauthor{\bsnm{Russell},~\bfnm{Stephen~T}\binits{S.~T.}} \AND \bauthor{\bsnm{Fish},~\bfnm{Jessica~N}\binits{J.~N.}}
(\byear{2016}).
\btitle{Mental health in lesbian, gay, bisexual, and transgender (LGBT) youth}.
\bjournal{Annual review of clinical psychology}
\bvolume{12}
\bpages{465--487}.
\end{barticle}
\endbibitem

\bibitem[\protect\citeauthoryear{Salerno et~al.}{2020}]{salerno2020sexual}
\begin{barticle}[author]
\bauthor{\bsnm{Salerno},~\bfnm{John~P}\binits{J.~P.}}, \bauthor{\bsnm{Devadas},~\bfnm{Jackson}\binits{J.}}, \bauthor{\bsnm{Pease},~\bfnm{M}\binits{M.}}, \bauthor{\bsnm{Nketia},~\bfnm{Bryanna}\binits{B.}} \AND \bauthor{\bsnm{Fish},~\bfnm{Jessica~N}\binits{J.~N.}}
(\byear{2020}).
\btitle{Sexual and gender minority stress amid the COVID-19 pandemic: Implications for LGBTQ young persons’ mental health and well-being}.
\bjournal{Public health reports}
\bvolume{135}
\bpages{721--727}.
\end{barticle}
\endbibitem

\bibitem[\protect\citeauthoryear{Scott}{1958}]{scott1958research}
\begin{barticle}[author]
\bauthor{\bsnm{Scott},~\bfnm{William~A}\binits{W.~A.}}
(\byear{1958}).
\btitle{Research definitions of mental health and mental illness.}
\bjournal{Psychological Bulletin}
\bvolume{55}
\bpages{29}.
\end{barticle}
\endbibitem

\bibitem[\protect\citeauthoryear{Sgarro}{1981}]{sgarro1981informational}
\begin{barticle}[author]
\bauthor{\bsnm{Sgarro},~\bfnm{Andrea}\binits{A.}}
(\byear{1981}).
\btitle{Informational divergence and the dissimilarity of probability distributions}.
\bjournal{Calcolo}
\bvolume{18}
\bpages{293--302}.
\end{barticle}
\endbibitem

\bibitem[\protect\citeauthoryear{Sharifian-Attar et~al.}{2022}]{sharifian2022analysing}
\begin{binproceedings}[author]
\bauthor{\bsnm{Sharifian-Attar},~\bfnm{Vida}\binits{V.}}, \bauthor{\bsnm{De},~\bfnm{Suparna}\binits{S.}}, \bauthor{\bsnm{Jabbari},~\bfnm{Sanaz}\binits{S.}}, \bauthor{\bsnm{Li},~\bfnm{Jenny}\binits{J.}}, \bauthor{\bsnm{Moss},~\bfnm{Harry}\binits{H.}} \AND \bauthor{\bsnm{Johnson},~\bfnm{Jon}\binits{J.}}
(\byear{2022}).
\btitle{Analysing Longitudinal Social Science Questionnaires: Topic modelling with BERT-based Embeddings}.
In \bbooktitle{2022 IEEE International Conference on Big Data (Big Data)}
\bpages{5558--5567}.
\bpublisher{IEEE}.
\end{binproceedings}
\endbibitem

\bibitem[\protect\citeauthoryear{Sibson}{1969}]{sibson1969information}
\begin{barticle}[author]
\bauthor{\bsnm{Sibson},~\bfnm{Robin}\binits{R.}}
(\byear{1969}).
\btitle{Information radius}.
\bjournal{Zeitschrift f{\"u}r Wahrscheinlichkeitstheorie und verwandte Gebiete}
\bvolume{14}
\bpages{149--160}.
\end{barticle}
\endbibitem

\bibitem[\protect\citeauthoryear{Sridhar, Daum{\'e}~III and Blei}{2022}]{sridhar2022heterogeneous}
\begin{barticle}[author]
\bauthor{\bsnm{Sridhar},~\bfnm{Dhanya}\binits{D.}}, \bauthor{\bsnm{Daum{\'e}~III},~\bfnm{Hal}\binits{H.}} \AND \bauthor{\bsnm{Blei},~\bfnm{David}\binits{D.}}
(\byear{2022}).
\btitle{Heterogeneous Supervised Topic Models}.
\bjournal{Transactions of the Association for Computational Linguistics}
\bvolume{10}
\bpages{732--745}.
\end{barticle}
\endbibitem

\bibitem[\protect\citeauthoryear{Srivastava and Sutton}{2017}]{srivastava2017autoencoding}
\begin{barticle}[author]
\bauthor{\bsnm{Srivastava},~\bfnm{Akash}\binits{A.}} \AND \bauthor{\bsnm{Sutton},~\bfnm{Charles}\binits{C.}}
(\byear{2017}).
\btitle{Autoencoding variational inference for topic models}.
\bjournal{arXiv preprint arXiv:1703.01488}.
\end{barticle}
\endbibitem

\bibitem[\protect\citeauthoryear{{National \textsc{LGBTQIA}+ \textsc{H}ealth \textsc{E}ducation \textsc{C}enter}}{2020}]{lgbtqterm}
\begin{barticle}[author]
\bauthor{\bsnm{{National \textsc{LGBTQIA}+ \textsc{H}ealth \textsc{E}ducation \textsc{C}enter}}}
(\byear{2020}).
\btitle{\textsc{LGBTQIA}+ \textsc{G}lossary of \textsc{T}erms for \textsc{H}ealth \textsc{C}are Teams}.
\end{barticle}
\endbibitem

\bibitem[\protect\citeauthoryear{Thoma et~al.}{2021}]{thoma2021disparities}
\begin{barticle}[author]
\bauthor{\bsnm{Thoma},~\bfnm{Brian~C}\binits{B.~C.}}, \bauthor{\bsnm{Rezeppa},~\bfnm{Taylor~L}\binits{T.~L.}}, \bauthor{\bsnm{Choukas-Bradley},~\bfnm{Sophia}\binits{S.}}, \bauthor{\bsnm{Salk},~\bfnm{Rachel~H}\binits{R.~H.}} \AND \bauthor{\bsnm{Marshal},~\bfnm{Michael~P}\binits{M.~P.}}
(\byear{2021}).
\btitle{Disparities in childhood abuse between transgender and cisgender adolescents}.
\bjournal{Pediatrics}
\bvolume{148}.
\end{barticle}
\endbibitem

\bibitem[\protect\citeauthoryear{Thompson and Mimno}{2020}]{thompson2020topic}
\begin{barticle}[author]
\bauthor{\bsnm{Thompson},~\bfnm{Laure}\binits{L.}} \AND \bauthor{\bsnm{Mimno},~\bfnm{David}\binits{D.}}
(\byear{2020}).
\btitle{Topic modeling with contextualized word representation clusters}.
\bjournal{arXiv preprint arXiv:2010.12626}.
\end{barticle}
\endbibitem

\bibitem[\protect\citeauthoryear{Van~der Maaten and Hinton}{2008}]{van2008visualizing}
\begin{barticle}[author]
\bauthor{\bparticle{Van~der} \bsnm{Maaten},~\bfnm{Laurens}\binits{L.}} \AND \bauthor{\bsnm{Hinton},~\bfnm{Geoffrey}\binits{G.}}
(\byear{2008}).
\btitle{Visualizing data using t-SNE.}
\bjournal{Journal of machine learning research}
\bvolume{9}.
\end{barticle}
\endbibitem

\bibitem[\protect\citeauthoryear{Wang, Blei and Heckerman}{2012}]{wang2012continuous}
\begin{barticle}[author]
\bauthor{\bsnm{Wang},~\bfnm{Chong}\binits{C.}}, \bauthor{\bsnm{Blei},~\bfnm{David}\binits{D.}} \AND \bauthor{\bsnm{Heckerman},~\bfnm{David}\binits{D.}}
(\byear{2012}).
\btitle{Continuous time dynamic topic models}.
\bjournal{arXiv preprint arXiv:1206.3298}.
\end{barticle}
\endbibitem

\bibitem[\protect\citeauthoryear{Wang, Zhou and He}{2019}]{wang2019atm}
\begin{barticle}[author]
\bauthor{\bsnm{Wang},~\bfnm{Rui}\binits{R.}}, \bauthor{\bsnm{Zhou},~\bfnm{Deyu}\binits{D.}} \AND \bauthor{\bsnm{He},~\bfnm{Yulan}\binits{Y.}}
(\byear{2019}).
\btitle{Atm: Adversarial-neural topic model}.
\bjournal{Information Processing \& Management}
\bvolume{56}
\bpages{102098}.
\end{barticle}
\endbibitem

\bibitem[\protect\citeauthoryear{Whaibeh, Vogt and Mahmoud}{2022}]{whaibeh2022addressing}
\begin{barticle}[author]
\bauthor{\bsnm{Whaibeh},~\bfnm{Emile}\binits{E.}}, \bauthor{\bsnm{Vogt},~\bfnm{Emily~L}\binits{E.~L.}} \AND \bauthor{\bsnm{Mahmoud},~\bfnm{Hossam}\binits{H.}}
(\byear{2022}).
\btitle{Addressing the behavioral health needs of sexual and gender minorities during the COVID-19 Pandemic: a Review of the expanding role of digital health technologies}.
\bjournal{Current Psychiatry Reports}
\bvolume{24}
\bpages{387--397}.
\end{barticle}
\endbibitem

\bibitem[\protect\citeauthoryear{Wu et~al.}{2021}]{wu2021prevalence}
\begin{barticle}[author]
\bauthor{\bsnm{Wu},~\bfnm{Tianchen}\binits{T.}}, \bauthor{\bsnm{Jia},~\bfnm{Xiaoqian}\binits{X.}}, \bauthor{\bsnm{Shi},~\bfnm{Huifeng}\binits{H.}}, \bauthor{\bsnm{Niu},~\bfnm{Jieqiong}\binits{J.}}, \bauthor{\bsnm{Yin},~\bfnm{Xiaohan}\binits{X.}}, \bauthor{\bsnm{Xie},~\bfnm{Jialei}\binits{J.}} \AND \bauthor{\bsnm{Wang},~\bfnm{Xiaoli}\binits{X.}}
(\byear{2021}).
\btitle{Prevalence of mental health problems during the COVID-19 pandemic: A systematic review and meta-analysis}.
\bjournal{Journal of affective disorders}
\bvolume{281}
\bpages{91--98}.
\end{barticle}
\endbibitem

\bibitem[\protect\citeauthoryear{Ye et~al.}{2024}]{YeSupplemental}
\begin{bmisc}[author]
\bauthor{\bsnm{Ye},~\bfnm{Hanwen}\binits{H.}}, \bauthor{\bsnm{Moreno},~\bfnm{Tatiana}\binits{T.}}, \bauthor{\bsnm{Alpern},~\bfnm{Adrianne}\binits{A.}}, \bauthor{\bsnm{Ehwerhemuepha},~\bfnm{Louis}\binits{L.}} \AND \bauthor{\bsnm{Qu},~\bfnm{Annie}\binits{A.}}
(\byear{2024}).
\btitle{Supplementary Materials to “Dynamic Topic Language Model on heterogeneous Children’s Mental Health Clinical Notes"}.
\end{bmisc}
\endbibitem

\end{thebibliography}


\begin{thebibliography}{23}

\bibitem[\protect\citeauthoryear{Afifi}{2007}]{afifi2007gender}
\begin{barticle}[author]
\bauthor{\bsnm{Afifi},~\bfnm{Mustafa}\binits{M.}}
(\byear{2007}).
\btitle{Gender differences in mental health}.
\bjournal{Singapore medical journal}
\bvolume{48}
\bpages{385}.
\end{barticle}
\endbibitem

\bibitem[\protect\citeauthoryear{Blei and Lafferty}{2006}]{blei2006dynamic}
\begin{binproceedings}[author]
\bauthor{\bsnm{Blei},~\bfnm{David~M}\binits{D.~M.}} \AND \bauthor{\bsnm{Lafferty},~\bfnm{John~D}\binits{J.~D.}}
(\byear{2006}).
\btitle{Dynamic topic models}.
In \bbooktitle{Proceedings of the 23rd international conference on Machine learning}
\bpages{113--120}.
\end{binproceedings}
\endbibitem

\bibitem[\protect\citeauthoryear{Blei, Ng and Jordan}{2003}]{blei2003latent}
\begin{barticle}[author]
\bauthor{\bsnm{Blei},~\bfnm{David~M}\binits{D.~M.}}, \bauthor{\bsnm{Ng},~\bfnm{Andrew~Y}\binits{A.~Y.}} \AND \bauthor{\bsnm{Jordan},~\bfnm{Michael~I}\binits{M.~I.}}
(\byear{2003}).
\btitle{Latent dirichlet allocation}.
\bjournal{Journal of machine Learning research}
\bvolume{3}
\bpages{993--1022}.
\end{barticle}
\endbibitem

\bibitem[\protect\citeauthoryear{Card, Tan and Smith}{2017}]{card2017neural}
\begin{barticle}[author]
\bauthor{\bsnm{Card},~\bfnm{Dallas}\binits{D.}}, \bauthor{\bsnm{Tan},~\bfnm{Chenhao}\binits{C.}} \AND \bauthor{\bsnm{Smith},~\bfnm{Noah~A}\binits{N.~A.}}
(\byear{2017}).
\btitle{Neural models for documents with metadata}.
\bjournal{arXiv preprint arXiv:1705.09296}.
\end{barticle}
\endbibitem

\bibitem[\protect\citeauthoryear{Diamond, Takapoui and Boyd}{2016}]{initialseed}
\begin{barticle}[author]
\bauthor{\bsnm{Diamond},~\bfnm{Steven}\binits{S.}}, \bauthor{\bsnm{Takapoui},~\bfnm{Reza}\binits{R.}} \AND \bauthor{\bsnm{Boyd},~\bfnm{Stephen}\binits{S.}}
(\byear{2016}).
\btitle{A general system for heuristic solution of convex problems over nonconvex sets}.
\bjournal{arXiv preprint arXiv:1601.07277}.
\end{barticle}
\endbibitem

\bibitem[\protect\citeauthoryear{Head and Zerner}{1985}]{head1985broyden}
\begin{barticle}[author]
\bauthor{\bsnm{Head},~\bfnm{John~D}\binits{J.~D.}} \AND \bauthor{\bsnm{Zerner},~\bfnm{Michael~C}\binits{M.~C.}}
(\byear{1985}).
\btitle{A \textsc{B}royden—\textsc{F}letcher—\textsc{G}oldfarb—\textsc{S}hanno optimization procedure for molecular geometries}.
\bjournal{Chemical physics letters}
\bvolume{122}
\bpages{264--270}.
\end{barticle}
\endbibitem

\bibitem[\protect\citeauthoryear{Hochreiter and Schmidhuber}{1997}]{hochreiter1997long}
\begin{barticle}[author]
\bauthor{\bsnm{Hochreiter},~\bfnm{Sepp}\binits{S.}} \AND \bauthor{\bsnm{Schmidhuber},~\bfnm{J{\"u}rgen}\binits{J.}}
(\byear{1997}).
\btitle{Long short-term memory}.
\bjournal{Neural Computation}
\bvolume{9}.
\end{barticle}
\endbibitem

\bibitem[\protect\citeauthoryear{Kingma and Ba}{2014}]{Adam}
\begin{barticle}[author]
\bauthor{\bsnm{Kingma},~\bfnm{Diederik~P}\binits{D.~P.}} \AND \bauthor{\bsnm{Ba},~\bfnm{Jimmy}\binits{J.}}
(\byear{2014}).
\btitle{Adam: A method for stochastic optimization}.
\bjournal{arXiv preprint arXiv:1412.6980}.
\end{barticle}
\endbibitem

\bibitem[\protect\citeauthoryear{Kingma and Welling}{2013}]{kingma2013auto}
\begin{barticle}[author]
\bauthor{\bsnm{Kingma},~\bfnm{Diederik~P}\binits{D.~P.}} \AND \bauthor{\bsnm{Welling},~\bfnm{Max}\binits{M.}}
(\byear{2013}).
\btitle{Auto-encoding variational bayes}.
\bjournal{arXiv preprint arXiv:1312.6114}.
\end{barticle}
\endbibitem

\bibitem[\protect\citeauthoryear{Loshchilov and Hutter}{2016}]{cosAnnealing}
\begin{barticle}[author]
\bauthor{\bsnm{Loshchilov},~\bfnm{Ilya}\binits{I.}} \AND \bauthor{\bsnm{Hutter},~\bfnm{Frank}\binits{F.}}
(\byear{2016}).
\btitle{Sgdr: Stochastic gradient descent with warm restarts}.
\bjournal{arXiv preprint arXiv:1608.03983}.
\end{barticle}
\endbibitem

\bibitem[\protect\citeauthoryear{Marshal et~al.}{2011}]{marshal2011suicidality}
\begin{barticle}[author]
\bauthor{\bsnm{Marshal},~\bfnm{Michael~P}\binits{M.~P.}}, \bauthor{\bsnm{Dietz},~\bfnm{Laura~J}\binits{L.~J.}}, \bauthor{\bsnm{Friedman},~\bfnm{Mark~S}\binits{M.~S.}}, \bauthor{\bsnm{Stall},~\bfnm{Ron}\binits{R.}}, \bauthor{\bsnm{Smith},~\bfnm{Helen~A}\binits{H.~A.}}, \bauthor{\bsnm{McGinley},~\bfnm{James}\binits{J.}}, \bauthor{\bsnm{Thoma},~\bfnm{Brian~C}\binits{B.~C.}}, \bauthor{\bsnm{Murray},~\bfnm{Pamela~J}\binits{P.~J.}}, \bauthor{\bsnm{D'Augelli},~\bfnm{Anthony~R}\binits{A.~R.}} \AND \bauthor{\bsnm{Brent},~\bfnm{David~A}\binits{D.~A.}}
(\byear{2011}).
\btitle{Suicidality and depression disparities between sexual minority and heterosexual youth: A meta-analytic review}.
\bjournal{Journal of adolescent health}
\bvolume{49}
\bpages{115--123}.
\end{barticle}
\endbibitem

\bibitem[\protect\citeauthoryear{Mcauliffe and Blei}{2007}]{mcauliffe2007supervised}
\begin{barticle}[author]
\bauthor{\bsnm{Mcauliffe},~\bfnm{Jon}\binits{J.}} \AND \bauthor{\bsnm{Blei},~\bfnm{David}\binits{D.}}
(\byear{2007}).
\btitle{Supervised topic models}.
\bjournal{Advances in neural information processing systems}
\bvolume{20}.
\end{barticle}
\endbibitem

\bibitem[\protect\citeauthoryear{Olsson and Nelson}{1975}]{olsson1975nelder}
\begin{barticle}[author]
\bauthor{\bsnm{Olsson},~\bfnm{Donald~M}\binits{D.~M.}} \AND \bauthor{\bsnm{Nelson},~\bfnm{Lloyd~S}\binits{L.~S.}}
(\byear{1975}).
\btitle{The \textsc{N}elder-\textsc{M}ead simplex procedure for function minimization}.
\bjournal{Technometrics}
\bvolume{17}
\bpages{45--51}.
\end{barticle}
\endbibitem

\bibitem[\protect\citeauthoryear{Pl{\"o}derl and Tremblay}{2015}]{ploderl2015mental}
\begin{barticle}[author]
\bauthor{\bsnm{Pl{\"o}derl},~\bfnm{Martin}\binits{M.}} \AND \bauthor{\bsnm{Tremblay},~\bfnm{Pierre}\binits{P.}}
(\byear{2015}).
\btitle{Mental health of sexual minorities. A systematic review}.
\bjournal{International review of psychiatry}
\bvolume{27}
\bpages{367--385}.
\end{barticle}
\endbibitem

\bibitem[\protect\citeauthoryear{Robbins and Monro}{1951}]{robbins1951stochastic}
\begin{barticle}[author]
\bauthor{\bsnm{Robbins},~\bfnm{Herbert}\binits{H.}} \AND \bauthor{\bsnm{Monro},~\bfnm{Sutton}\binits{S.}}
(\byear{1951}).
\btitle{A stochastic approximation method}.
\bjournal{The Annals of Mathematical Statistics}
\bvolume{22}
\bpages{400--407}.
\end{barticle}
\endbibitem

\bibitem[\protect\citeauthoryear{Rosenfield and Mouzon}{2013}]{rosenfield2013gender}
\begin{barticle}[author]
\bauthor{\bsnm{Rosenfield},~\bfnm{Sarah}\binits{S.}} \AND \bauthor{\bsnm{Mouzon},~\bfnm{Dawne}\binits{D.}}
(\byear{2013}).
\btitle{Gender and mental health}.
\bjournal{Handbook of the sociology of mental health}
\bpages{277--296}.
\end{barticle}
\endbibitem

\bibitem[\protect\citeauthoryear{Russell and Fish}{2016}]{russell2016mental}
\begin{barticle}[author]
\bauthor{\bsnm{Russell},~\bfnm{Stephen~T}\binits{S.~T.}} \AND \bauthor{\bsnm{Fish},~\bfnm{Jessica~N}\binits{J.~N.}}
(\byear{2016}).
\btitle{Mental health in lesbian, gay, bisexual, and transgender (LGBT) youth}.
\bjournal{Annual review of clinical psychology}
\bvolume{12}
\bpages{465--487}.
\end{barticle}
\endbibitem

\bibitem[\protect\citeauthoryear{Sgarro}{1981}]{sgarro1981informational}
\begin{barticle}[author]
\bauthor{\bsnm{Sgarro},~\bfnm{Andrea}\binits{A.}}
(\byear{1981}).
\btitle{Informational divergence and the dissimilarity of probability distributions}.
\bjournal{Calcolo}
\bvolume{18}
\bpages{293--302}.
\end{barticle}
\endbibitem

\bibitem[\protect\citeauthoryear{Sibson}{1969}]{sibson1969information}
\begin{barticle}[author]
\bauthor{\bsnm{Sibson},~\bfnm{Robin}\binits{R.}}
(\byear{1969}).
\btitle{Information radius}.
\bjournal{Zeitschrift f{\"u}r Wahrscheinlichkeitstheorie und verwandte Gebiete}
\bvolume{14}
\bpages{149--160}.
\end{barticle}
\endbibitem

\bibitem[\protect\citeauthoryear{Srivastava and Sutton}{2017}]{srivastava2017autoencoding}
\begin{barticle}[author]
\bauthor{\bsnm{Srivastava},~\bfnm{Akash}\binits{A.}} \AND \bauthor{\bsnm{Sutton},~\bfnm{Charles}\binits{C.}}
(\byear{2017}).
\btitle{Autoencoding variational inference for topic models}.
\bjournal{arXiv preprint arXiv:1703.01488}.
\end{barticle}
\endbibitem

\bibitem[\protect\citeauthoryear{Stevens et~al.}{2012}]{stevens2012exploring}
\begin{binproceedings}[author]
\bauthor{\bsnm{Stevens},~\bfnm{Keith}\binits{K.}}, \bauthor{\bsnm{Kegelmeyer},~\bfnm{Philip}\binits{P.}}, \bauthor{\bsnm{Andrzejewski},~\bfnm{David}\binits{D.}} \AND \bauthor{\bsnm{Buttler},~\bfnm{David}\binits{D.}}
(\byear{2012}).
\btitle{Exploring topic coherence over many models and many topics}.
In \bbooktitle{Proceedings of the 2012 joint conference on empirical methods in natural language processing and computational natural language learning}
\bpages{952--961}.
\end{binproceedings}
\endbibitem

\bibitem[\protect\citeauthoryear{Tang and Ishwaran}{2017}]{tang2017random}
\begin{barticle}[author]
\bauthor{\bsnm{Tang},~\bfnm{Fei}\binits{F.}} \AND \bauthor{\bsnm{Ishwaran},~\bfnm{Hemant}\binits{H.}}
(\byear{2017}).
\btitle{Random forest missing data algorithms}.
\bjournal{Statistical Analysis and Data Mining: The ASA Data Science Journal}
\bvolume{10}
\bpages{363--377}.
\end{barticle}
\endbibitem

\bibitem[\protect\citeauthoryear{Zvornicanin}{2023}]{zvornicanin2023coherence}
\begin{bmisc}[author]
\bauthor{\bsnm{Zvornicanin},~\bfnm{Enes}\binits{E.}}
(\byear{2023}).
\btitle{When Coherence Score is Good or Bad in Topic Modeling}.
\end{bmisc}
\endbibitem

\end{thebibliography}
}

\end{document}


\title{Supplementary Materials \\ to ``Dynamic Topic Language Model on heterogeneous Children's Mental Health Clinical Notes"}
\begin{aug}
\author[]{\fnms{}~\snm{Hanwen Ye}{}},
\author[]{\fnms{}~\snm{Tatiana Moreno}{}},
\author[]{\fnms{}~\snm{Adrianne Alpern}{}},
\author[]{\fnms{}~\snm{Louis Ehwerhemuepha}{}},
\and
\author[]{\fnms{}~\snm{Annie Qu}{}}
\end{aug}

\tableofcontents

\section{Variational inference proof}

In this section, we prove the following proposition \ref{prop:LELBO} from Section \blue{3.2} under Assumptions \blue{1} and \blue{2}.
\begin{proposition} \label{prop:LELBO}
    Under Assumptions \blue{1}-\blue{2}, the evidence lower bound (ELBO) for a single document generated by Process \blue{2} over a finite $T$-stage longitudinal time horizon is 
    \vspace{-0.2em}
    \begin{align}
    \hspace{3em} & \log P(\bm{w}_{1:T} | X_{1:T}, Y, \beta) \ge \label{eq: ELBO-log-likelihood}\raisetag{-2.5em} \\[-0.15em]
    & \hspace{1.5em}\underbrace{-\mathbb{KL}(q_{\psi_1}(\theta_1 | \bm{w}_1, X_1, Y) \;||\; p(\theta_1 | \theta_0, X_1, Y)) +  \mathop{\mathbb{E}}_{\theta_1 \sim q_{\psi_1}} (\log P(\bm{w}_1| \theta_1, \beta))}_{\text{Single-stage ELBO for the first stage}} + 
    \label{eq: L-ELBO} \raisetag{-2.5em} \\[-0.5em]
    & \hspace{-2.7em} \underbrace{\sum_{t=2}^T \left\{-\mathop{\mathbb{E}}_{\theta_{t-1} \sim q_{\psi_{t-1}}} \left( \mathbb{KL}[q_{\psi_t}(\theta_t | \bm{w}_t, X_t, Y) \;||\; p(\theta_t | \theta_{t-1}, X_t, Y)]\right) + \mathop{\mathbb{E}}_{\theta_t \sim q_{\psi_t}} (\log P(\bm{w}_t| \theta_t, \beta))\right\}}_{\text{Temporal-dependent ELBO for the follow-up stages}}  \notag
    \end{align}
\end{proposition}

\begin{proof}
We show the proof in two steps. First, we lower bound the log-likelihood of the documents, $\log P(\bm{w}_{1:T} | X_{1:T}, Y, \beta)$. By the the objective of the variational inference, we minimize the KL-divergence between the variational distribution and the true underlying posterior, i.e., $\mathbb{KL}(Q(\theta_{1:T} | \mathbf{w}_{1:T}, X_{1:T}, Y) || P(\theta_{1:T} | \mathbf{w}_{1:T}, X_{1:T}, Y, \beta))$. The equalities are established below,
\begin{align}
    &\mathbb{KL}(Q(\theta_{1:T} | \mathbf{w}_{1:T}, X_{1:T}, Y) || P(\theta_{1:T} | \mathbf{w}_{1:T}, X_{1:T}, Y, \beta)) \label{eq: obj-KL}\\
&= \mathbb{E} \left( Q(\theta_{1:T} | \mathbf{w}_{1:T}, X_{1:T}, Y) \cdot \log\frac{Q(\theta_{1:T} | \mathbf{w}_{1:T}, X_{1:T}, Y)}{P(\theta_{1:T} | \mathbf{w}_{1:T}, X_{1:T}, Y, \beta))}\right) \notag \\
&= \mathbb{E} \left( Q(\theta_{1:T} | \mathbf{w}_{1:T}, X_{1:T}, Y) \cdot \log\frac{Q(\theta_{1:T} | \mathbf{w}_{1:T}, X_{1:T}, Y)}{P(\theta_{1:T}, \mathbf{w}_{1:T} | X_{1:T}, Y, \beta)) \;/\; P(\mathbf{w}_{1:T} | X_{1:T}, Y, \beta))}\right) \notag \\
&= \mathbb{E} \left( Q(\theta_{1:T} | \mathbf{w}_{1:T}, X_{1:T}, Y) \cdot \log\frac{Q(\theta_{1:T} | \mathbf{w}_{1:T}, X_{1:T}, Y)}{P(\theta_{1:T}, \mathbf{w}_{1:T} | X_{1:T}, Y, \beta)}\right) + \notag \\
& \quad \; \mathbb{E} \left( Q(\theta_{1:T} | \mathbf{w}_{1:T}, X_{1:T}, Y) \cdot \log P(\mathbf{w}_{1:T} | X_{1:T}, Y, \beta)\right) \notag \\
&= \mathbb{KL}(Q(\theta_{1:T} | \mathbf{w}_{1:T},X_{1:T}, Y) \;||\; P(\theta_{1:T}, \mathbf{w}_{1:T} | X_{1:T}, Y, \beta)) + \log P(\mathbf{w}_{1:T} | X_{1:T}, Y, \beta), \label{eq: obj-KL-derive}
\end{align}

\noindent where Equation \eqref{eq: obj-KL-derive} is based on $\mathbb{E} \left( Q(\theta_{1:T} | \mathbf{w}_{1:T}, X_{1:T}, Y) \cdot \log P(\mathbf{w}_{1:T} | X_{1:T}, Y, \beta)\right) = \log P(\mathbf{w}_{1:T} | X_{1:T}, Y, \beta) \cdot \mathbb{E} \left( Q(\theta_{1:T} | \mathbf{w}_{1:T}, X_{1:T}, Y)\right)$ and $\mathbb{E}(Q(\theta_{1:T} | \mathbf{w}_{1:T}, X_{1:T}, Y)) = 1$. Since the KL divergence \eqref{eq: obj-KL} is non-negative, we find the lower bound of the log-likelihood of the documents, i.e., 
\begin{align}
    &\log P(\mathbf{w}_{1:T} | X_{1:T}, Y, \beta) \\
    &\ge - \mathbb{KL}(Q(\theta_{1:T} | \mathbf{w}_{1:T},X_{1:T}, Y) \;||\; P(\theta_{1:T}, \mathbf{w}_{1:T} | X_{1:T}, Y, \beta)) \label{eq:lower-bound} \\
    &= \mathbb{E}\left(Q(\theta_{1:T} | \mathbf{w}_{1:T},X_{1:T}, Y) \cdot \log \frac{Q(\theta_{1:T} | \mathbf{w}_{1:T},X_{1:T}, Y)}{P(\theta_{1:T}, \mathbf{w}_{1:T} | X_{1:T}, Y, \beta)}\right) \notag \\
    &= \mathbb{E}\left(Q(\theta_{1:T} | \mathbf{w}_{1:T},X_{1:T}, Y) \cdot \log \frac{Q(\theta_{1:T} | \mathbf{w}_{1:T},X_{1:T}, Y)}{p(\theta_{1:T} | X_{1:T}, Y) \cdot P(\mathbf{w}_{1:T} | \theta_{1:T}, X_{1:T}, Y, \beta)}\right) \notag \\
    &= \underbrace{\mathbb{KL}(Q(\theta_{1:T} | \mathbf{w}_{1:T},X_{1:T}, Y) \;||\; p(\theta_{1:T} | X_{1:T}, Y))}_{\text{\textbf{Approximation error}: variational distribution $Q$ and geneartive prior $p$}} + \label{eq: approximation-term}\\
    & \quad \;\; \underbrace{-\mathbb{E}_{\theta_{1;T} \sim Q(\theta_{1:T} | \mathbf{w}_{1:T},X_{1:T}, Y)}(\log P(\mathbf{w}_{1:T} | \theta_{1:T}, X_{1:T}, Y, \beta))}_{\text{\textbf{Reconstruction error}: model likelihood from variational topics}} \label{eq: reconstruction-error}.
\end{align}

Next, we aim to break down the two terms which involves all $T$ number of stages into individual adjacent stages to ease the computations. By further decomposing the approximation error \eqref{eq: approximation-term}, we obtain the following equalities,
\begin{align}
    &\mathbb{KL}(Q(\theta_{1:T} | \mathbf{w}_{1:T},X_{1:T}, Y) \;||\; p(\theta_{1:T} | X_{1:T}, Y)) \notag \\
    &= \mathbb{E}\left(Q(\theta_{1:T} | \mathbf{w}_{1:T},X_{1:T}, Y) \cdot \log \frac{Q(\theta_{1:T} | \mathbf{w}_{1:T},X_{1:T}, Y)}{p(\theta_{1:T} | X_{1:T}, Y)}\right) \notag \\
    &= \mathbb{E}\left( \prod_{j=1}^T q(\theta_j | \mathbf{w}_j, X_j, Y) \cdot \sum_{t=1}^T \log \frac{q(\theta_t | \mathbf{w}_t, X_t, Y)}{p(\theta_t | \theta_{t-1}, X_{1:t}, Y)}\right)  \label{eq:assumption-1}\\
    %
    &= \mathbb{E}\left(\log \frac{q(\theta_1 | \mathbf{w}_{1}, X_{1}, Y)}{p(\theta_1 | \theta_{0}, X_{1}, Y)} q(\theta_1 | \mathbf{w}_1, X_1, Y) \prod_{j=2}^{T} q(\theta_j | \mathbf{w}_j, X_j, Y)\right) +  \notag \\
    & \;\; \sum_{t=2}^T \mathbb{E}\left(\log \frac{q(\theta_t | \mathbf{w}_{t}, X_{t}, Y)}{p(\theta_t | \theta_{t-1}, X_{1:t}, Y)} \prod_{j \in \{t-1, t\}} q(\theta_j | \mathbf{w}_j, X_j, Y) \prod_{j \notin \{t-1, t\}}^{T} q(\theta_j | \mathbf{w}_j, X_j, Y)\right) \notag \\
    &= \mathbb{E}\left(q(\theta_1 | \mathbf{w}_1, X_1, Y)\log \frac{q(\theta_1 | \mathbf{w}_{1}, X_{1}, Y)}{p(\theta_1 | \theta_{0}, X_{1}, Y)}\right) \cdot \prod_{j=2}^{T} \mathbb{E}(q(\theta_j | \mathbf{w}_j, X_j, Y)) + \label{eq:independence-of-q}\\
    & \;\; \sum_{t=2}^T \mathbb{E}\left(\log \frac{q(\theta_t | \mathbf{w}_{t}, X_{t}, Y)}{p(\theta_t | \theta_{t-1}, X_{1:t}, Y)} \prod_{j \in \{t-1, t\}} q(\theta_j | \mathbf{w}_j, X_j, Y)\right) \cdot \prod_{j \notin \{t-1, t\}}^{T} \mathbb{E}\left(q(\theta_j | \mathbf{w}_j, X_j, Y)\right) \notag \\
    &= \mathbb{KL}(q(\theta_1 | \mathbf{w}_1, X_1, Y) \;||\; p(\theta_1 | \theta_{0}, X_{1}, Y)) \cdot 1 + \vphantom{\sum_{i=2}^T} \label{eq: valid-distribution}\\
    &\;\; \sum_{t=2}^T \int_{\theta_{t-1}} q(\theta_{t-1} | \mathbf{w}_{t-1}, X_{t-1}, Y) \left( \int_{\theta_t} q(\theta_t | \mathbf{w}_t, X_t, Y) \log \frac{q(\theta_t | \mathbf{w}_{t}, X_{t}, Y)}{p(\theta_t | \theta_{t-1}, X_{1:t}, Y)} d \theta_t \right)d \theta_{t-1} \cdot 1 \notag \\
    &= \underbrace{{\mathbb{KL}(q(\theta_1 | \mathbf{w}_1, X_1, Y) \;||\; p(\theta_1 | \theta_{0}, X_{1}, Y))}}_{\text{Single-stage KL term for the first stage}} + \vphantom{\sum_{i=2}^T} \label{eq: KL-stage}\\
    &\;\; \underbrace{\sum_{t=2}^T \mathbb{E}_{\theta_{t-1} \sim q(\theta_{t-1} | \mathbf{w}_{t-1}, X_{t-1}, Y)} \left( \mathbb{KL}(q(\theta_t | \mathbf{w}_t, X_t, Y) \;||\; p(\theta_t | \theta_{t-1}, X_{1:t}, Y))\right)}_{\text{Temporal-dependent KL term for the follow-up stages}}. \notag
\end{align}

\vspace{1em}
\noindent Note that Equation \eqref{eq:assumption-1} satisfies as $p(\theta_{1:T} | X_{1:T}, Y) = \prod_{j-1}^{T} p(\theta_i | \theta_{i-1}, X_{1:i}, Y)$ under Assumption \blue{2}; Equation \eqref{eq:independence-of-q} results from the independent factorization of the variational distribution $Q$; and Equation \eqref{eq: valid-distribution} is based on the fact that the variational distribution $q$ at each stage is a valid distribution. Similarly, we decompose the reconstruction error term \eqref{eq: reconstruction-error} into,
\begin{align}
& \mathbb{E}_{\theta_{1;T} \sim Q(\theta_{1:T} | \mathbf{w}_{1:T},X_{1:T}, Y)}(\log P(\mathbf{w}_{1:T} | \theta_{1:T}, X_{1:T}, Y, \beta)) \notag \\
&= \mathbb{E}_{\theta_{1;T} \sim Q(\theta_{1:T} | \mathbf{w}_{1:T},X_{1:T}, Y)} \left( \sum_{t=1}^T \log P(\mathbf{w}_t | \theta_t, \beta) \right) \qquad \text{(by Assumption \blue{1})}\notag \\
&= \sum_{t=1}^T \mathbb{E}_{\theta_t \sim q(\theta_{t} | \mathbf{w}_{t},X_{t}, Y)} \left( \log P(\mathbf{w}_t | \theta_t, \beta) \right) \label{eq: reconstruction-stage}.
\end{align}

Finally, by combining Equation \eqref{eq: KL-stage} and \eqref{eq: reconstruction-stage}, we obtain the ELBO for a single document generated by Process \blue{2} over a finite $T$-stage longitudinal horizon and complete the proof.
\end{proof}

\section{Algorithm: model architecture}

To further illustrate our algorithm, we first introduce the classic single-stage variational autoencoder \citep{kingma2013auto} and then fully explain our model with extension to multistage longitudinal setting in this section.


\subsection{Single-stage VAE}

VAE was first developed to embed variational inference to the deep neural network architecture \citep{kingma2013auto}. According to Figure \ref{fig:VAE}, VAE mainly consists of two major components: encoder and decoder. The encoder intends to learn a low-dimensional latent space of the original data to estimate the mean ($\theta_\mu$) and scale ($\theta_\sigma$) parameters of the variational distribution $Q(\cdot)$, then the decoder aims to reconstruct the original documents from $Q(\cdot)$. Importantly, in order to draw samples from the variational distribution meanwhile enabling gradient descent for model training, VAE constructs variational samples, $\theta = \theta_\mu + \theta_\sigma \cdot \epsilon$, with $\epsilon \sim \mathcal{N}(0,1)$, so that $\theta \sim Q(\theta_\mu, \theta_\sigma^2)$ when $Q(\cdot)$ is specified as a normal distribution. In a traditional single-stage setting, the optimization procedure involves improving the evidence lower bound (ELBO), i.e., 
\begin{equation}
   \psi^*, \beta^* = \argmax_{\psi, \beta} -\mathbb{KL}(Q_{\psi}(\theta |\bm{w}, X, Y) \;||\; p(\theta | X, Y)) + \mathbb{E}_{\theta \sim Q_{\psi}(\cdot)}(\log P(\bm{w} | \theta, \beta)),
\end{equation}
where $p(\cdot)$ is the prior of the topic proportions given the document metadata, $X$, and group membership, $Y$, whereas $P(\cdot)$ is the likelihood of the documents, $\bm{w}$. By maximizing ELBO, we can obtain the optimal topics $\beta$, and estimate topic proportions from $Q(\theta_\mu, \theta_\sigma^2)$.
\begin{figure}[H]
    \centering
    \includegraphics[width=0.7\textwidth]{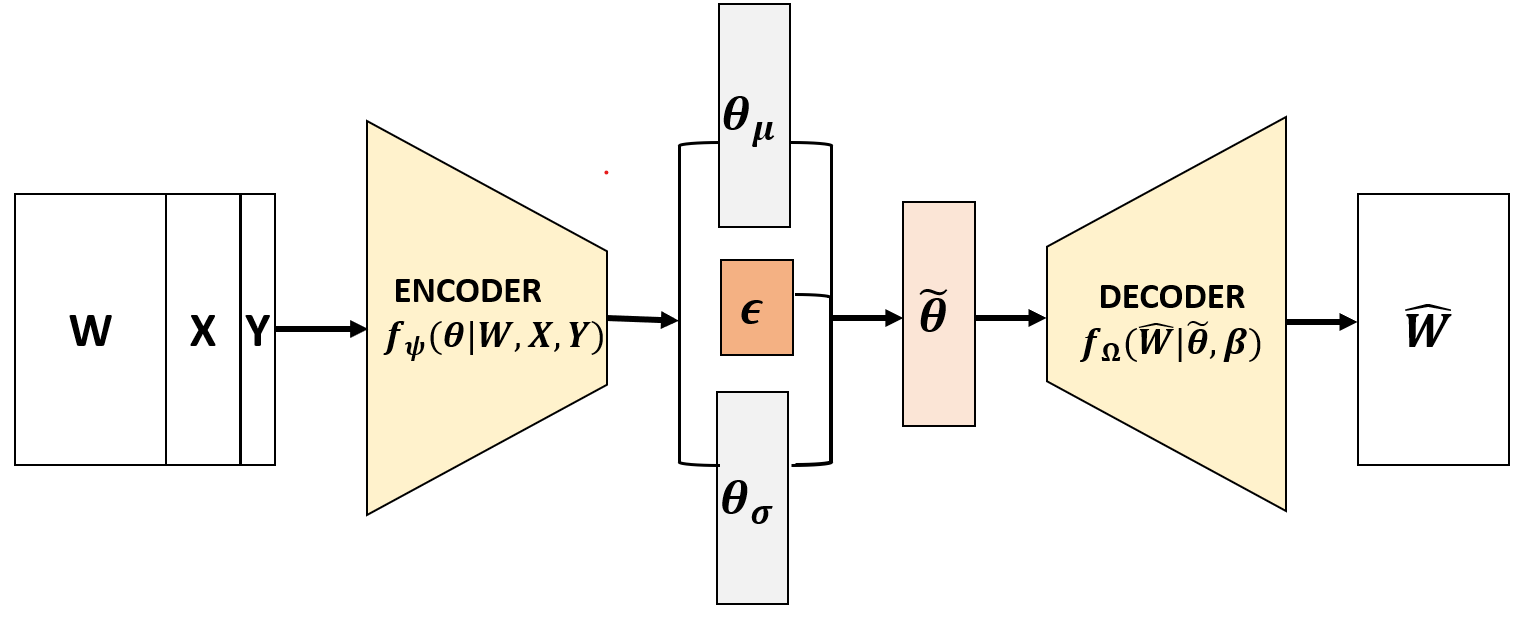}
    \caption{Architecture of the traditional VAE. The encoder and decoder are two parametrizable functions, $f_\psi$ and $f_\Omega$, implemented by fully-connected neural networks. The variational samples $\theta = \theta_\mu + \theta_\sigma \cdot \epsilon$ follows the variational distribution $Q(\theta_\mu, \theta_\sigma^2)$ when $Q \doteq \mathcal{N}$ and $\epsilon \sim \mathcal{N}(0,1)$. }
    \label{fig:VAE}
\end{figure}

\vspace{-2em}
\subsection{Multistage longitudinal VAE}

However, when extending the topic models to a multi-stage longitudinal setting, there involves time-dependency in the topic proportions, i.e., $\theta_t \sim \mathcal{N}(f_{\phi_t}(\theta_{t-1}, X_t, Y), \sigma^2\mathbb{I})$. The prior of topic proportions, $\theta_t$, is now dependent on the ones from the previous stage, $\theta_{t-1}$. As a result, we cannot simply divide the KL-term into individual single stages because of the following,
\begin{equation}
    \mathbb{KL}(Q_{\Psi}(\theta_{1:T} | \bm{w}_{1:T}, X_{1:T}, Y) \;||\; p(\theta_{1:T} | X_{1:T}, Y)) \neq \sum_{t=1}^T \mathbb{KL}(q_{\psi_t}(\theta_{t} | \bm{w}_t, X_{t}, Y) \;||\; p(\theta_{t} | X_{t}, Y)).
    \notag
\end{equation}
To encounter the challenges, we introduce the longitudinal VAE framework, leveraging the derived proposition as follows,
\begin{align}
    &\mathbb{KL}(Q_{\Psi}(\theta_{1:T} | \bm{w}_{1:T}, X_{1:T}, Y) \;||\; p(\theta_{1:T} | X_{1:T}, Y)) \notag \\ 
    &= \sum_{t=1}^T \mathop{\mathbb{E}}_{\theta_{t-1} \sim q_{\psi_{t-1}}} \left( \mathbb{KL}[q_{\psi_t}(\theta_t | \bm{w}_t, X_t, Y) \;||\; p(\theta_t | \theta_{t-1}, X_t, Y)]\right). \notag
    \notag
\end{align}
The architecture of the proposed longitudinal VAE is displayed in Figure \ref{fig:L-VAE}. In comparison, our longitudinal VAE takes $M$ samples of the variational parameters to obtain an empirical estimate of the expectation $\mathbb{E}_{\theta_{t-1} \sim q_{\psi_{t-1}}}$ with respect to the topic proportions from the previous time stage. This enables preserving the time dependencies among the topic proportions.
\begin{figure}[H]
    \centering
    \includegraphics[width=0.8\textwidth]{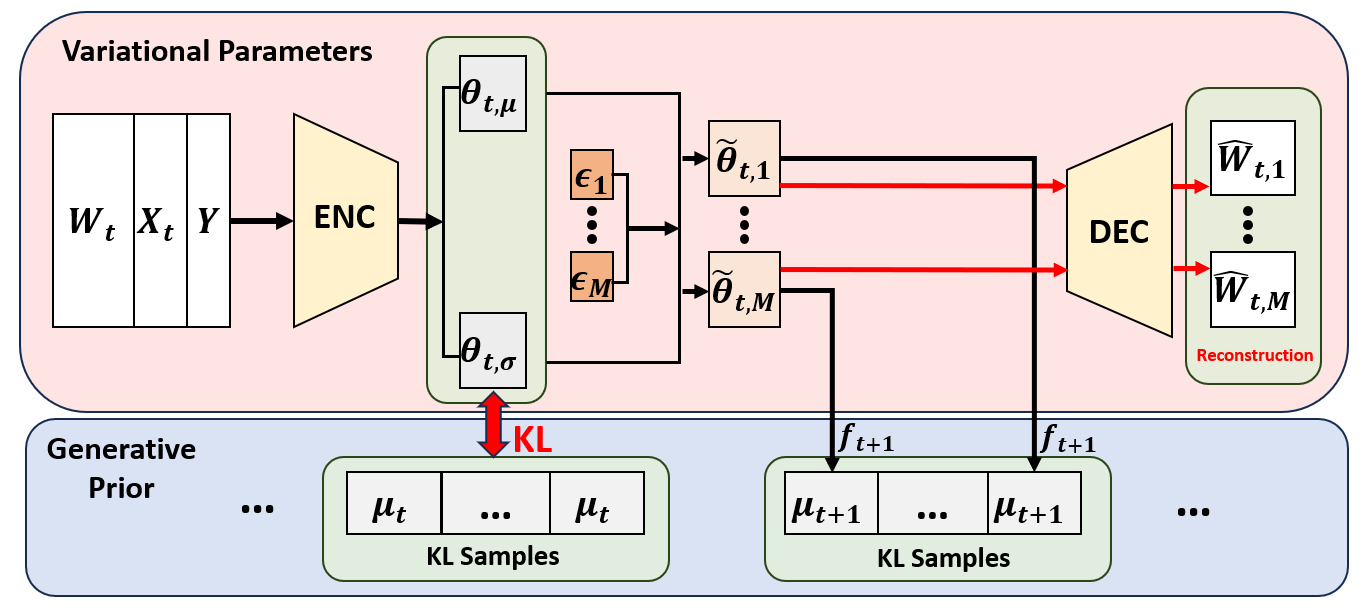}
    \caption{The architecture of the proposed longitudinal VAE. The variational samples $\epsilon_1,...,\epsilon_M$ are taken M times per time stage to evaluate the expectation $\mathbb{E}_{\theta_{t-1} \sim q_{\psi_{t-1}}}$,}
    \label{fig:L-VAE}
\end{figure}

\subsection{HCF-DTM architecture}
Under the extended multistage VAE framework, we present the graphical architecture of our proposed method at time $t$ in Figure \ref{fig:VAE-structure}. As shown, the model first takes in three inputs, i.e., the text data $W_t$, the document's metadata $X_t$, and the corresponding group membership $Y$, to calculate the mean $\mu^q_t$ and standard deviation $\sigma^q_t$ of the variational distribution via the parameterized functions $g^\mu_t$ and $g^\sigma_t$. Then, in a two-group setting, by inverting the group membership input to $-Y$, we obtain $\tilde{\mu}^q_t$ and $\tilde{\sigma}^q_t$ of the counterfactual variational distribution. As both variational and counterfactual variational distributions are defined, we calculate the inter-distributional distance using the Jensen-Shannon divergence measure and include it as the first component of the loss function. Next, we randomly sample $M$ number of latent topic proportions, $\theta_{t, 1:M}$ from the variational distribution. Along with the latent topics $\beta$, we compute the log-likelihood of the observed documents for the reconstruction loss in the loss function. Lastly, for each sampled latent topic proportion $\theta_{t, m}$, we apply the generative prior function $f_{t+1}$ to derive the generative topic proportion prior $\mu_{t+1, m}$ at next time stage. This enables us to obtain an empirical estimate of the KL divergence at time $t+1$, which completes the last piece of the loss function. By repeating the above process for $T$ number of time stages, we obtain the total loss function for the entire longitudinal time horizon.

\begin{figure}[h]
    \centering
    \includegraphics[width=0.8\textwidth]{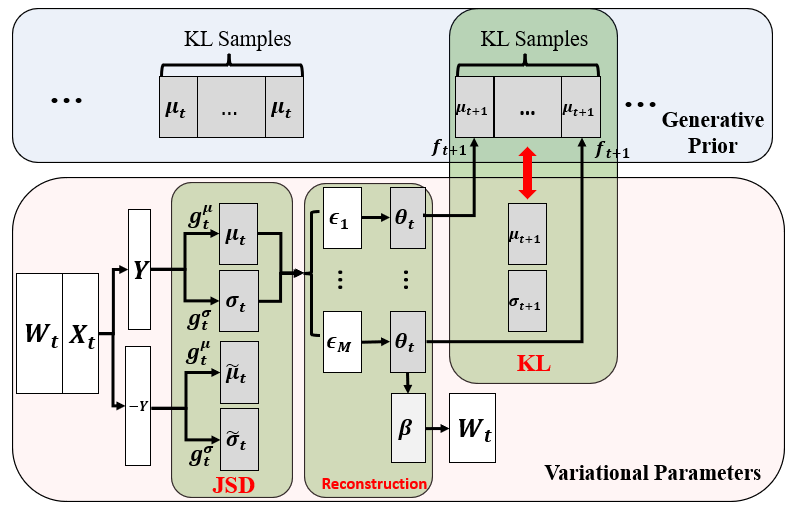}
    \caption{A graphical illustration of the proposed HCF-DTM at time $t$. The loss function can be separated into three pieces (Group-wise Jensen-Shannon distance, reconstruction loss, KL divergence) as shown in green. Note that the generative topic proportion prior $\mu_{t+1}$ is dependent on the re-parametrized variational proportions $\theta_t$ at the previous time stage. A random sample of $\mu_{t+1}$ is collected to obtain an empirical estimate of the KL divergence. }
    \label{fig:VAE-structure}
\end{figure}

The main objective of this optimization procedure is to find the optimal set of parameters of the variational functions $g^\mu_{1:T}$ and $g^\sigma_{1:T}$, which can minimize the loss function. In our choice of models, we adopt the long-short term memory (LSTM) network \citep{hochreiter1997long}, which are effective at capturing the longitudinal dependency among the topics and document's metadata. The network parameters are updated via the stochastic gradient descent (SGD) algorithm \citep{robbins1951stochastic}. To improve the convergence performance, one can consider using the Adam optimizer \citep{Adam}, applying the cosine annealing warm-restart schedule \citep{cosAnnealing} and different initialization seeds \citep{initialseed}, or tuning network hyperparameters, such as the learning rate and training epochs. Each model specification is evaluated via cross-validations and the parameters with the best averaged performance are selected. Notably, we do not limit the parametric form of the variational functions. One can parameterize the functions and estimate via the more conventional Broyden–Fletcher–Goldfarb–Shanno (BFGS) method \citep{head1985broyden}, or the Nelder–Mead method \citep{olsson1975nelder}.

\section{Simulation}

In this section, we elaborate the fitting procedure, include more simulation settings and model performance results to Section 5. Specifically, we conduct sensitivity analysis to investigate the change in method performance when the time-consistent topic assumption breaks, the number of group memberships increases, the group-wise heterogeneity effect is reduced, and the distance metrics is changed. In addition, we include coherence and perplexity scores to evaluate the interpretability of the extracted topics. 

\subsection{Model fitting procedure and evaluation metrics}
Our simulations consider five competing methods: LDA \citep{blei2003latent}, supervised LDA (sLDA) \citep{mcauliffe2007supervised}, multistage dynamic LDA (mdLDA) \citep{blei2006dynamic}, prodLDA \citep{srivastava2017autoencoding} and SCHOLAR \citep{card2017neural}. We summarize the categories of each method in Table \ref{tab:competing_method}. Note that, due to the limited availability of longitudinal topic models, the mdLDA proposed by \cite{blei2006dynamic} is the only competing method among five to dynamically model the topics across time stages simultaneously.

\begin{table}[H]
    \centering
    \resizebox{0.9\textwidth}{!}{
    \begin{tabular}{|l|l|cccccc|}
    \hline   &  & LDA & sLDA & mdLDA & prodLDA & SCHOLAR & HCF-DTM \\ \hline
    \multirow{2}{*}{Model type} & Bayesian & \checkmark  & \checkmark  & \checkmark & & & \\ \cline{2-8}
    & Neural & & & &  \checkmark  & \checkmark & \checkmark\\ \hline 
    \multirow{2}{*}{Stage type} & Single-stage  & \checkmark & \checkmark& & \checkmark & \checkmark & \\ \cline{2-8}
    & Multi-stage & & & \checkmark& & & \checkmark\\ \hline
    \multirow{2}{*}{Data augmentation} & Metadata (X)  &  & & & & \checkmark &\checkmark  \\ \cline{2-8}
    & Label group (Y) & &\checkmark & & & \checkmark & \checkmark\\ \hline
    
    \end{tabular}
    }
    \caption{Categories of considered methods based on model types: Bayesian probabilistic topic models (BPTMs) and neural topic models (NTMs); stage types: single-stage and multi-stage; and data augmentation. }
    \label{tab:competing_method}
\end{table}

\subsubsection{Model fitting procedure}
Fitting the longitudinal topic models is straightforward. Given the simulated bag of words (BOWs) representation of documents, $W_{N \times T \times V}$, both mdLDA and our proposed method will directly estimate the topic-word distribution $\hat{\beta}_{T \times V \times K}$ and the topic proportions $\hat{\theta}_{T \times N \times K}$ ($N$: number of subjects, $T$: number of specified time stages, $V$: vocabulary size,  $K$: number of latent topics). In particular, due to our model specification, the topic-word distribution from our method will have $\hat{\beta}_{1, V \times K } = \hat{\beta}_{2, V \times K} = \cdot\cdot\cdot = \hat{\beta}_{T, V \times K}$. On the other hand, for each of the rest four single-stage methods, we fit the model independently on each time stage, which will yield $\{\hat{\beta}_{V \times K}\}_{t=1}^T$ and $\{\hat{\theta}_{N \times K}\}_{t=1}^T$ respectively. 

Next, since the topics from single-stage methods fail to incorporate time dependencies and could have random numeric order over time (i.e., the second topic $\hat{\beta}_{1,1:V,2}$ at time $1$ could correspond to the first topic $\hat{\beta}_{2,1:V,1}$ at time $2$), we need to re-map the order of extracted topics at each time stage. In practice, the re-mapping procedure could be followed by evaluating each topics and categorizing the ones with similar semantic meanings into the same group. However, this process is noticeably cumbersome when the number of stages or topics increases. To automate this process in simulation, we utilize the ground-true topic distribution $\beta^0_{T \times V \times K}$ and find the permutation order $\mathcal{O}_t$ of fitted topics which minimizes their KL-divergence, i.e.,
\begin{equation}
   \mathcal{O}^*_t = \argmin_{\mathcal{O} \in \text{perm}(1,..,K)} \widehat{\mathbb{KL}} (\hat{\beta}_{t, 1:V, \mathcal{O}} || \beta_{t,1:V,1:K}).
\end{equation}
With the new orderings, topics having the same topic index across time stages (e.g., the first topic $\{\hat{\beta}_{t, 1:V,  \mathcal{O}_1}\}_{t=1}^T$) will have the same semantic interpretation. To simplify the notations, we continue denoting the re-ordered topics as $\hat{\beta}_{T \times V \times K}$ in the following. 

\subsubsection{Evaluation metrics} 
After applying the pre-processing steps on estimated topics as described above, we assess the methods based on the following five metrics: 
\begin{itemize}
    \item Empirical KL-divergence between estimated and ground-true topic distributions.
    \begin{equation}
        \widehat{\mathbb{KL}}(\hat{\bm{\beta}} \;||\; \bm{\beta}) 
= \frac{1}{T \cdot K} \sum_{t=1}^T \sum_{k=1}^K \left(\sum_{v=1}^V \hat{\beta}_{t,v,k} \cdot \log(\beta_{t,v,k}/ \hat{\beta}_{t,v,k})\right).
\label{metric:KL}
    \end{equation}
    A smaller KL-divergence value represents a closer distributional distance between the estimated topics and the ground-true topics. When the estimated topic word probabilities equals to the ground-truth, the KL-divergence will reach its minimum at 0.
    
\vspace{1em}
    \item Coherence scores of top 15 words within each fitted topics.
    \begin{equation}
    \text{coherence}_{\text{umass}}(\hat{\bm{\beta}}) = \frac{1}{T \cdot K}\sum_{t=1}^T \sum_{k=1}^K \left(\sum_{i=1}^{15} \sum_{j \neq i}^{15} \log \frac{D(v_i^{(t, k)}, v_j^{(t, k)}) + 1}{D(v_j^{(t, k)})}\right),
    \label{metric:coherence}
    \end{equation}
    where $\{v_i^{(t,k)}\}_{i=1}^{15}$ are the top 15 most probable words from $\hat{\beta}_{t,1:V,k}$, $D(v_i^{(t,k)}, v_j^{(k,t)})$ is the number of documents at time t containing both words $v_i^{(t,k)}$ and $v_j^{(t,k)}$ , and $D(v_j^{(t,k)})$ is the number of documents at time t containing the word $D(v_j^{(t,k)})$ alone. Noticeably, when two words appear more often compared to a single word, i.e., $D(v_i^{(t,k)}, v_j^{(k,t)}) > D(v_j^{(t,k)})$, a larger coherence score will be obtained.
   
    \vspace{1em}
    \item Perplexity scores of fitted models over observed documents.
    \begin{equation}
    \footnotesize
        \text{perplexity}(\hat{\bm{\beta}}, \hat{\bm{\theta}}; \bm{W}) = \frac{1}{T} \sum_{t=1}^T \exp \left\{\frac{1}{N} \sum_{i=1}^N \left(-\frac{1}{\sum_{v=1}^V W_{i,t,v}} \sum_{v=1}^V W_{i,t,v} \cdot \log(\hat{\theta}_{t,i,1:K} \cdot \hat{\beta}_{t,v,1:K}^\intercal)\right) \right\}.
    \label{metric:perplexity}
    \end{equation}
    As shown, the perplexity is defined as the exponentiated average negative log-likelihood of a sequence. Thus, a smaller perplexity represents a larger likelihood of observing such corpus under the fitted model. 
    
    \vspace{1em}
    \item Alignment accuracy of dominant topics between fitted and ground-true topic proportions. 
    \begin{equation}
    \text{Accuracy}_{\text{align}}(\hat{\bm{\theta}}, \bm{\theta}) = \frac{1}{T \cdot N} \sum_{t=1}^T \sum_{i=1}^N \mathbb{I} \left(\argmax_{k} \hat{\theta}_{t,i,1:K} = \argmax_{k} \theta_{t,i,1:K} \right).
    \label{metric:acc-align}
    \end{equation}
    The dominant topic is the one having the largest topic proportions at a specified timepoint.
    
    \vspace{1em}
    \item Classification accuracy of group memberships based on fitted topic proportions.
    \begin{equation}
        \text{Accuracy}_{\text{cls}}(\hat{\bm{\theta}}, \bm{Y}) = \frac{1}{T \cdot N} \sum_{t=1}^T \sum_{i=1}^N \mathbb{I} \left( f_{\text{cls}} (\hat{\theta}_{t,i,1:K}) = Y_i \right),
    \label{metric:acc-group}
    \end{equation}
    where we apply a simple logistic regression solely on estimated topic proportions to predict group membership, i.e., $Y \sim \text{logistic}(\hat{\bm{\theta}})$.
\end{itemize}
\vspace{1em}

The five metrics individually inspects the model from three aspects: topics recovery rate, dominant topics identification, and group-wise topic separation capability. In particular, KL divergence \eqref{metric:KL} directly quantifies how close $\hat{\bm{\beta}}$ are to the ground-true ones; alignment accuracy \eqref{metric:acc-align} indicates the quality of $\hat{\bm{\theta}}$; classification accuracy \eqref{metric:acc-group} measures the amount of group-wise information contained in $\hat{\bm{\theta}}$; and both coherence \eqref{metric:coherence} and perplexity scores \eqref{metric:perplexity} connect the model estimations back to the observed documents, further examining the topic interpretability and goodness-of-fit.

\subsection{Ablation studies and assumption examination}
\subsubsection{Break of the time-consistent topic assumption} \label{A:sim:break-consistent}
One of the key assumption under our proposed generative process is that there exists a long-lasting set of topics which remain constant over time periods. This is useful when one is interested to understand the underlying consistent theme of a topic, and intend to quantify the popularity trend of each topic by solely using the topic proportions. However, we haven't discussed scenarios where the time-consistent topic assumption breaks. In this section, we conduct sensitivity analysis to investigate our model performance when the topic distribution is dynamically changing.

To establish temporal dependencies on the topic distributions, we consider constructing the ground-true topic distributions $\bm{\beta}$ based on the following process, i.e., 
\begin{equation}
    \beta_{t,v,k} = p_{\mathcal{N}}(v; \mathcal{N}(\mu_{t,k}, 1 + \frac{t}{T}\cdot \phi)),
\end{equation}

\noindent where $p_{\mathcal{N}}(\cdot)$ is the density of a normal distribution, $\mu_k = \lfloor k \cdot \frac{V}{K} \rfloor$, and $v$, as words, ranges from 1 to $V$. Specifically, the topics at first stage is generated according to the same sampling procedure as described in the main text. For the follow-up stages, we introduce a hyper-parameter $\phi$ which controls the longitudinal transition variance. An example of resulting topics can be found in Figure \ref{fig:simulated_topics}. Noticeably, as $\phi$ value increases, the longitudinal variations among the topics become larger.
\begin{figure}[H]
    \centering
    \includegraphics[width=0.32\textwidth]{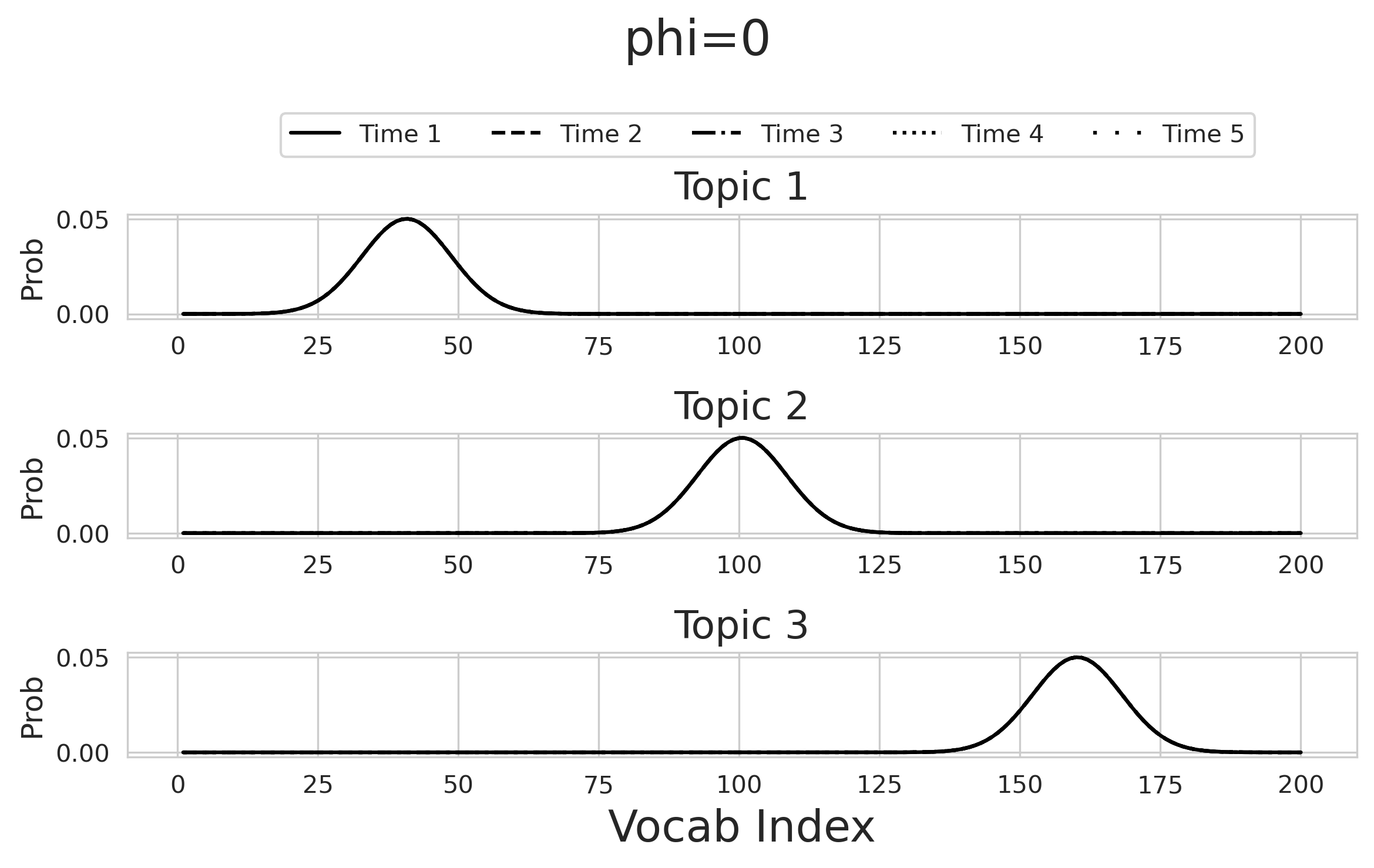}
    \includegraphics[width=0.32\textwidth]{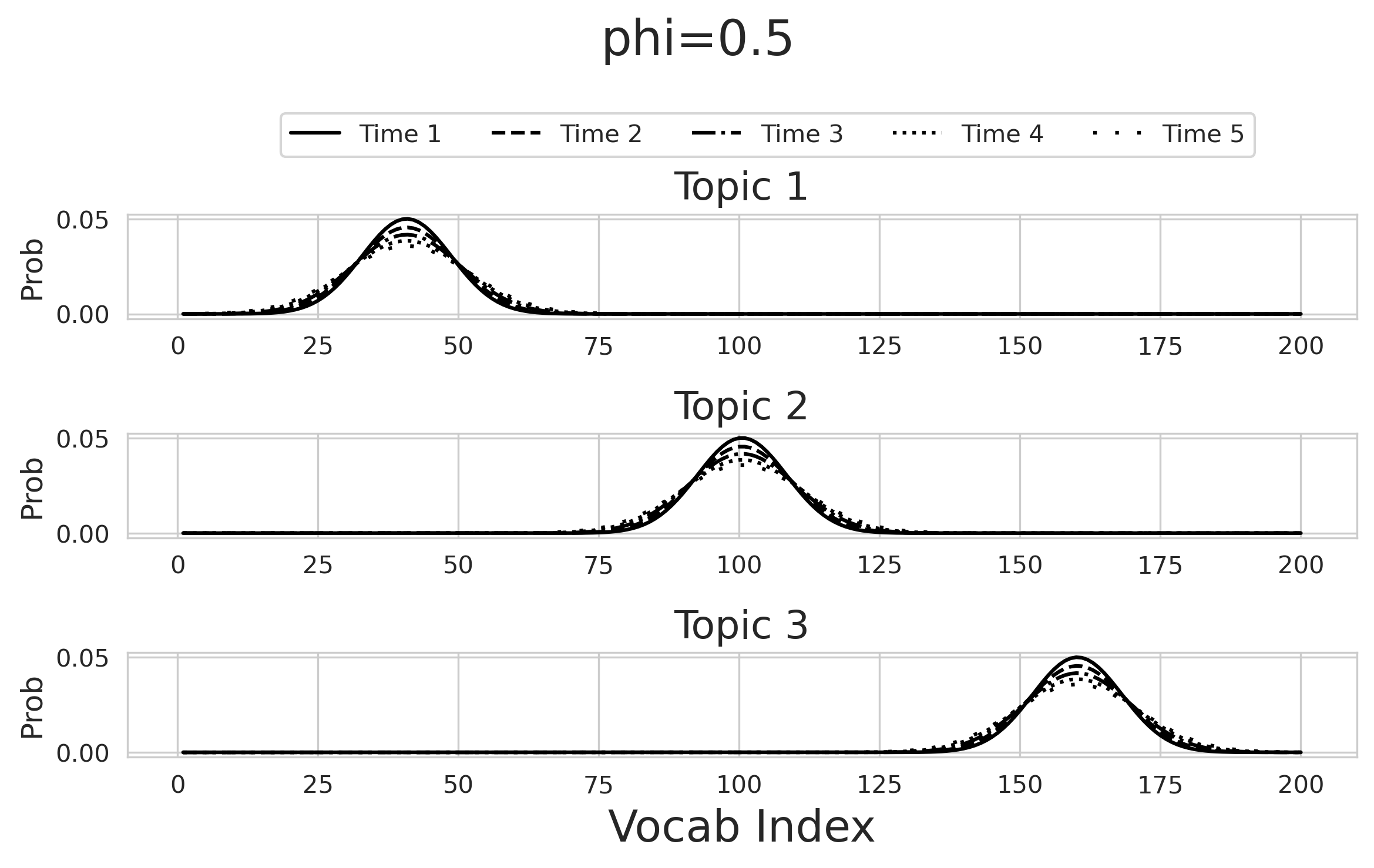}
    \includegraphics[width=0.32\textwidth]{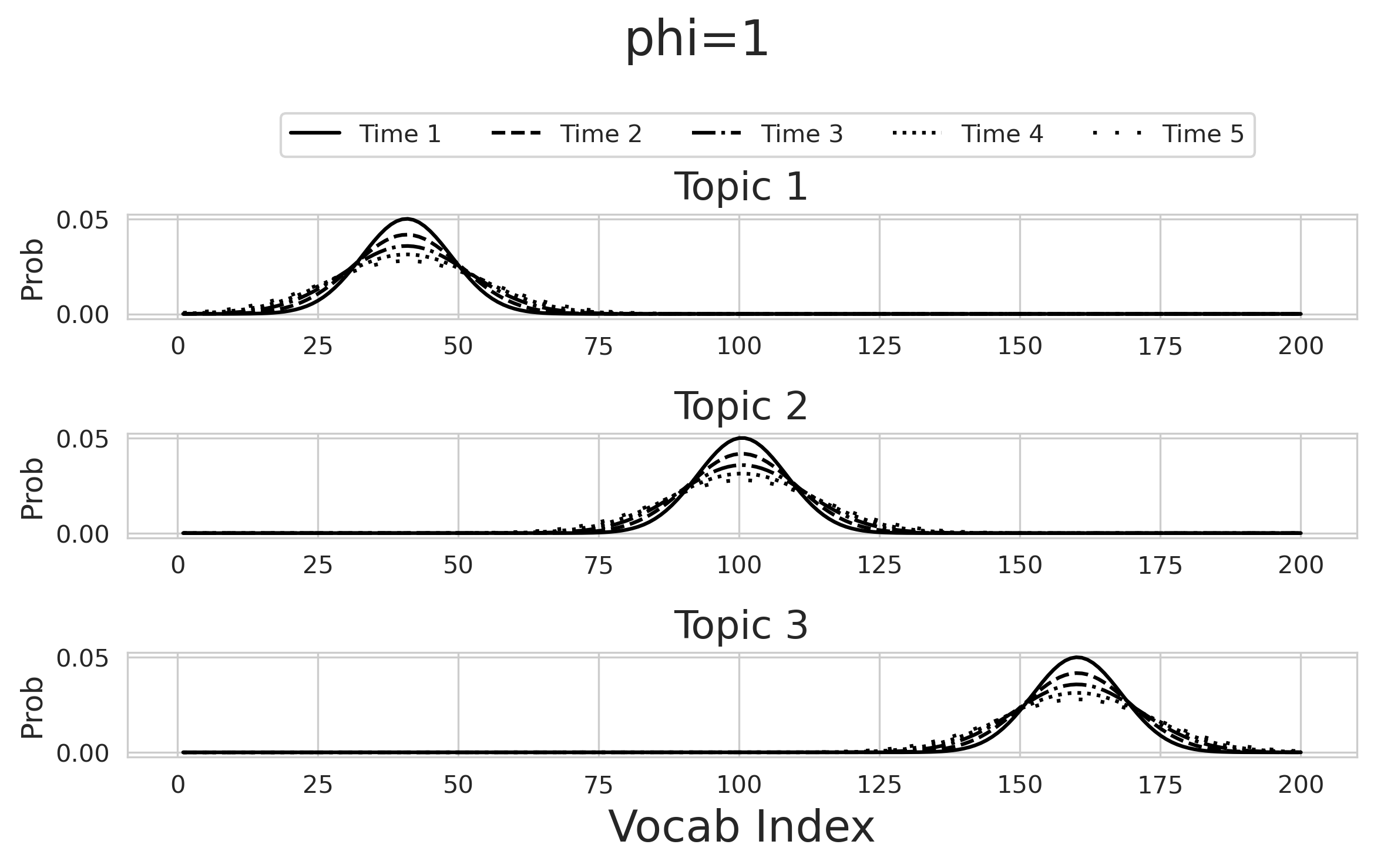}
    \caption{\small An illustration of simulated topic distributions. The number of topics is 3, number of stage is 5, number of word is 200, and the longitudinal dependency $\phi$ is set to 0 (left), 0.5 (middle), and 1 (right).}
  \label{fig:simulated_topics}
\end{figure}

After constructing the ground-true topic distributions, and simulating the topic proportions and documents same way as the main text, we summarize the results from 20 repeated experiments in Table \ref{tab:topic-change-K3} ($K=3$, $T=5$) and Table \ref{tab:topic-change-K5} ($K=5$, $T=5$). By comparing with individual competing method, we obtain the following three observations. First, our method still has the best-performing results, though the improvement rate decreases as the number of topics increases and the topics have larger transition variance. This is expected since more latent topics increases the number of parameters need to be estimated, and a larger transition variance indicates a greater deviation from the time-consistent topics assumption. However, our model exhibits robustness to slight misspecifications in consistent topic assumptions.

Secondly, the mdLDA method is sensitive to the selection of topic transition prior. In simulations, we provide mdLDA-true model with a near-optimal transition prior and demonstrate that a substantial improvement can be made to the mdLDA model with the default prior. However, finding such an optimal transition prior could be challenging in practice. Additionally, without the incorporation of document metadata, the mdLDA-true model fails to account for individual heterogeneity, and thus, under-performs our proposed method.

In addition, SCHOLAR and LDA, in general, have the next best-performing results. Since both methods are single-stage, the estimated topics do not depend on the results from the first stage and can be more adaptive to each individual stage. However, one has to manually categorize the obtained topics to the same theme, which imposes additional difficulties when the number of topics and stages increases and the topics are not visually separable.

\begin{table}[H]
    \centering
    \def\arraystretch{1.35}%
    \resizebox{\textwidth}{!}{
\begin{tabular}{|l|l|lllllll|l|}
\hline
metrics & $\phi$ &        HCF-DTM &           mdLDA  &      mdLDA-true &        SCHOLAR &            LDA &            SLDA &         prodLDA & imp-rate \\
\hline
KL & 0.0 &  \textbf{3.879} (1.10) &  19.318 (3.07) &  11.924 (1.66) &  7.561 (1.05) &  13.965 (2.64) &  17.118 (0.46) &  22.074 (0.02) &   48.70\% \\
          & 0.5 &  \textbf{3.263} (0.88) &  15.091 (3.98) &  10.282 (2.05) &  5.812 (1.08) &  11.404 (2.39) &  13.970 (0.47) &  19.029 (0.02) &   43.86\% \\
          & 1.0 &  \textbf{3.310} (1.26) &  13.685 (3.80) &   9.558 (2.29) &  4.799 (0.89) &   9.568 (1.95) &  10.137 (0.23) &  16.427 (0.02) &   31.03\% \\ \hline
coherence & 0.0 &  \textbf{3.208} (0.05) &   1.028 (0.32) &   2.075 (0.26) &  2.489 (0.11) &   2.424 (0.33) &  -0.909 (0.23) &  -1.120 (0.05) &   28.89\% \\
          & 0.5 &  \textbf{2.822} (0.06) &   0.954 (0.39) &   1.758 (0.22) &  2.175 (0.10) &   1.880 (0.32) &  -0.748 (0.17) &  -1.145 (0.06) &   29.75\% \\
          & 1.0 &  \textbf{2.514} (0.07) &   0.944 (0.50) &   1.429 (0.30) &  1.973 (0.08) &   1.602 (0.36) &  -0.994 (0.12) &  -1.071 (0.08) &   27.42\% \\ \hline
perplexity & 0.0 &  \textbf{3.681} (0.01) &   4.047 (0.04) &   3.921 (0.06) &  4.152 (0.03) &   3.889 (0.13) &   9.210 (1.08) &   5.301 (0.00) &    5.35\% \\
          & 0.5 &  \textbf{3.868} (0.01) &   4.094 (0.04) &   4.079 (0.05) &  4.266 (0.03) &   4.068 (0.12) &   8.646 (0.90) &   5.303 (0.00) &    4.92\% \\
          & 1.0 &  \textbf{4.029} (0.01) &   4.132 (0.05) &   4.209 (0.06) &  4.348 (0.03) &   4.203 (0.11) &   8.218 (0.82) &   5.301 (0.00) &    2.49\% \\ \hline
dominant\_acc & 0.0 &  \textbf{0.965} (0.01) &   0.565 (0.18) &   0.942 (0.11) &  0.929 (0.01) &   0.948 (0.06) &   0.309 (0.32) &   0.307 (0.11) &    1.79\% \\
          & 0.5 &  \textbf{0.963} (0.01) &   0.509 (0.21) &   0.942 (0.11) &  0.924 (0.02) &   0.837 (0.12) &   0.283 (0.20) &   0.306 (0.09) &    2.23\% \\
          & 1.0 &  \textbf{0.965} (0.01) &   0.546 (0.21) &   0.935 (0.14) &  0.929 (0.03) &   0.770 (0.16) &   0.279 (0.23) &   0.353 (0.09) &    3.21\% \\ \hline
group\_acc & 0.0 &  \textbf{0.973} (0.01) &   0.963 (0.01) &   0.963 (0.01) &  0.940 (0.01) &   0.960 (0.01) &   0.962 (0.01) &   0.808 (0.04) &    1.04\% \\
          & 0.5 &  \textbf{0.974} (0.01) &   0.963 (0.01) &   0.962 (0.01) &  0.940 (0.01) &   0.961 (0.01) &   0.963 (0.01) &   0.812 (0.07) &    1.14\% \\
          & 1.0 &  \textbf{0.974} (0.01) &   0.963 (0.01) &   0.962 (0.01) &  0.940 (0.01) &   0.961 (0.01) &   0.963 (0.01) &   0.794 (0.04) &    1.14\% \\ \hline

\end{tabular}
}
\caption{\footnotesize Dynamically changing topic distribution: model performances when $K=3$ and $T=5$, where $\phi$ represents the magnitude of topic temporal transition variances (larger the greater). Standard errors are summarized in the parentheses next to the estimated means. The improvement rate compares HCF-DTM against the best performer.}
\label{tab:topic-change-K3}
\end{table}

\begin{table}[H]
    \centering
    \def\arraystretch{1.35}%
    \resizebox{\textwidth}{!}{
\begin{tabular}{|l|l|lllllll|l|}
\hline
metrics & $\phi$ &        HCF-DTM &           mdLDA  &      mdLDA-true &        SCHOLAR &            LDA &            SLDA &         prodLDA & imp-rate \\
\hline
KL & 0.0 &  \textbf{14.721} (1.26) &  24.984 (2.20) &  15.259 (2.54) &  16.874 (0.94) &  17.088 (2.22) &  22.761 (1.59) &  27.001 (0.01) &    3.53\% \\
          & 0.5 &  \textbf{12.479} (1.21) &  22.190 (3.73) &  13.312 (2.01) &  13.412 (1.05) &  15.196 (2.43) &  18.515 (1.27) &  24.478 (0.01) &    6.26\% \\
          & 1.0 &  \textbf{11.028} (1.21) &  18.566 (3.35) &  11.828 (3.19) &  11.434 (1.18) &  13.537 (2.45) &  14.997 (1.57) &  22.222 (0.01) &    3.55\% \\ \hline
coherence & 0.0 &   \textbf{2.025} (0.26) &   0.541 (0.26) &   1.361 (0.29) &   1.752 (0.17) &   1.971 (0.14) &  -0.335 (0.16) &  -0.884 (0.04) &    2.74\% \\
          & 0.5 &   \textbf{1.738} (0.33) &   0.614 (0.38) &   1.322 (0.20) &   1.535 (0.17) &   1.657 (0.17) &  -0.555 (0.08) &  -0.869 (0.04) &    4.89\% \\
          & 1.0 &   \textbf{1.563} (0.28) &   0.579 (0.36) &   1.127 (0.19) &   1.284 (0.18) &   1.408 (0.22) &  -0.521 (0.08) &  -0.830 (0.05) &   11.01\% \\ \hline
perplexity & 0.0 &   \textbf{3.497} (0.01) &   3.808 (0.05) &   3.585 (0.08) &   4.020 (0.03) &   3.593 (0.07) &   9.163 (0.96) &   5.300 (0.00) &    2.45\% \\
          & 0.5 &   \textbf{3.681} (0.01) &   3.795 (0.04) &   3.747 (0.06) &   4.118 (0.03) &   3.764 (0.06) &   8.661 (0.85) &   5.299 (0.00) &    1.76\% \\
          & 1.0 &   \textbf{3.831} (0.01) &   3.881 (0.03) &   3.862 (0.04) &   4.210 (0.03) &   3.891 (0.06) &   8.415 (0.81) &   5.301 (0.00) &    0.80\% \\ \hline
dominant\_acc & 0.0 &   \textbf{0.905} (0.10) &   0.291 (0.17) &   0.546 (0.27) &   0.490 (0.09) &   0.806 (0.22) &   0.152 (0.15) &   0.196 (0.09) &   12.28\% \\
          & 0.5 &   \textbf{0.872} (0.13) &   0.314 (0.21) &   0.476 (0.34) &   0.502 (0.10) &   0.703 (0.23) &   0.182 (0.15) &   0.224 (0.08) &   24.04\% \\
          & 1.0 &   \textbf{0.865} (0.13) &   0.327 (0.18) &   0.551 (0.32) &   0.477 (0.07) &   0.669 (0.24) &   0.199 (0.14) &   0.194 (0.10) &   29.30\% \\ \hline
group\_acc & 0.0 &   \textbf{0.983} (0.01) &   0.963 (0.01) &   0.964 (0.01) &   0.937 (0.01) &   0.960 (0.01) &   0.964 (0.01) &   0.915 (0.02) &    1.97\% \\
          & 0.5 &   \textbf{0.984} (0.01) &   0.964 (0.01) &   0.964 (0.01) &   0.937 (0.01) &   0.959 (0.01) &   0.964 (0.01) &   0.895 (0.04) &    2.07\% \\
          & 1.0 &   \textbf{0.984} (0.01) &   0.964 (0.00) &   0.963 (0.00) &   0.938 (0.01) &   0.960 (0.01) &   0.964 (0.01) &   0.906 (0.02) &    2.07\% \\ \hline

\end{tabular}
}
\caption{\footnotesize Dynamically changing topic distribution: model performances when $K=5$ and $T=5$, where $\phi$ represents the magnitude of topic temporal transition variances (larger the greater). The improvement rate compares HCF-DTM against the best performer.}
\label{tab:topic-change-K5}
\end{table}

\subsubsection{Increasing number of groups} Our proposed method introduces counterfactual topic distributions to account for group-wise heterogeneity. By maximizing the inter-distributional distances during the optimization procedure,
we aim to disentangle the group identities directly from the estimated topics. In previous simulations, we only consider the cases with two groups and utilize the mutual information (MI) as the distance measure. Here, we extend experiments to the multi-group setting. 

Suppose there are $G$ number of groups and Subject $i$ belongs to Group 1, i.e., $Y=1$. In this case, the original topic distribution for Subject $i$ is $Q(\theta_i \;|\; W_i, X_i, Y_i=1)$, while the counterfactual distributions are $Q(\theta_i | W_i, X_i, Y_i=2),...,Q(\theta_i \;|\; W_i, X_i, Y_i=G)$. Currently, there are two ways to augment the inter-distributional distances: information radius \citep{sibson1969information} and average divergence score \citep{sgarro1981informational}. The information radius considers the joint distribution $M \doteq \sum_{g=1}^G Q(\theta_i \;|\; W_i, X_i, Y_i=g)$ and calculates the total divergence to the average distribution via,
\begin{align}
    &\text{Dist}\{Q(\theta_i \;|\; W_i, X_i, Y_i=1) \;||\; Q(\theta_i | W_i, X_i, Y_i=2),...,Q(\theta_i \;|\; W_i, X_i, Y_i=G)\} \notag \\
    &= \sum_{g=1}^G \frac{1}{G} \cdot \mathbb{KL}(Q(\theta_i \;|\; W_i, X_i, Y_i=g) \;||\; M),
    \label{eq:KL-tot}
\end{align}
while the average divergence score computes the averaged divergence to the reference (original) distribution as,
\begin{align}
    &\text{Dist}\{Q(\theta_i \;|\; W_i, X_i, Y_i=1) \;||\; Q(\theta_i | W_i, X_i, Y_i=2),...,Q(\theta_i \;|\; W_i, X_i, Y_i=G)\} \notag \\
    &= \sum_{g=2}^G \frac{1}{G-1} \cdot \mathbb{KL}(Q(\theta_i \;|\; W_i, X_i, Y_i=1) \;||\; Q(\theta_i \;|\; W_i, X_i, Y_i=g)). 
    \label{eq:KL-ave}
\end{align}
Specifically, in a two-group scenario, the information radius is equivalent to the mutual information (MI) or Jensen Shannon divergence (JSD). In the following, we conduct simulations when the number of groups is set to 2, 3, and 4, and the distance metrics are calculated via information radius. According to Table \ref{tab:number-groups}, our method still reaches the highest group-wise separation accuracy, though the improvement margin decreases as more groups are involved. It is worth noting that sLDA reaches the next-highest group-wise accuracy. However, because sLDA relies on an additional classifier, there is no guarantee that a strong group classification result indicates interpretable and separable underlying topic distributions. In contrast, our model directly maximizes group-wise heterogeneity on the topic distribution level, and thus retains a high level of interpretation and topics recovery rate compared to other methods.
\begin{table}[H]
    \centering
    \def\arraystretch{1.4}%
    \resizebox{0.95\textwidth}{!}{
\begin{tabular}{|l|l|lllllll|l|}
\hline
metrics & G &        HCF-DTM &           DTM  & DTM-true  &      SCHOLAR &            LDA &            SLDA &         prodLDA & imp-rate \\
\hline
KL & 2 &  \textbf{12.536} (3.59) &  30.741 (2.80) &  22.608 (2.40) &  20.857 (1.22) &  23.322 (1.88) &  31.850 (0.99) &  31.806 (0.01) &   39.90\% \\
          & 3 &   \textbf{9.347} (4.23) &  30.072 (2.85) &  17.786 (1.83) &  19.277 (1.57) &  22.135 (2.04) &  30.736 (1.03) &  31.803 (0.01) &   47.45\% \\
          & 4 &   \textbf{5.514} (1.78) &  26.819 (2.37) &  16.011 (3.10) &  18.449 (1.38) &  20.993 (2.46) &  31.101 (0.37) &  31.810 (0.01) &   65.56\% \\ \hline
coherence & 2 &   \textbf{3.193} (0.33) &   0.771 (0.27) &   1.653 (0.33) &   2.999 (0.13) &   2.507 (0.26) &  -0.496 (0.07) &  -1.036 (0.03) &    6.47\% \\
          & 3 &   \textbf{3.510} (0.33) &   0.909 (0.26) &   1.788 (0.36) &   2.945 (0.22) &   2.706 (0.34) &  -0.963 (0.03) &  -1.082 (0.04) &   19.19\% \\
          & 4 &   \textbf{3.830} (0.12) &   1.454 (0.29) &   2.087 (0.41) &   3.359 (0.21) &   2.875 (0.43) &  -0.947 (0.03) &  -1.113 (0.04) &   14.02\% \\ \hline
perplexity & 2 &   \textbf{3.492} (0.02) &   3.789 (0.04) &   3.641 (0.05) &   4.245 (0.03) &   3.661 (0.06) &   9.624 (0.87) &   5.300 (0.00) &    4.09\% \\
          & 3 &   \textbf{3.476} (0.02) &   3.876 (0.04) &   3.730 (0.05) &   4.216 (0.03) &   3.709 (0.07) &   9.383 (0.66) &   5.301 (0.00) &    6.28\% \\
          & 4 &   \textbf{3.468} (0.03) &   3.710 (0.10) &   3.757 (0.10) &   4.157 (0.02) &   3.732 (0.07) &   9.238 (0.38) &   5.300 (0.00) &    6.52\% \\ \hline
dominant\_acc & 2 &   \textbf{0.909} (0.05) &   0.346 (0.15) &   0.790 (0.18) &   0.710 (0.06) &   0.821 (0.12) &   0.233 (0.14) &   0.178 (0.07) &   10.72\% \\
          & 3 &   \textbf{0.934} (0.05) &   0.304 (0.13) &   0.630 (0.18) &   0.785 (0.07) &   0.838 (0.12) &   0.173 (0.15) &   0.191 (0.07) &   11.46\% \\
          & 4 &   \textbf{0.960} (0.03) &   0.395 (0.10) &   0.617 (0.17) &   0.786 (0.07) &   0.852 (0.11) &   0.212 (0.14) &   0.197 (0.07) &   12.68\% \\ \hline
group\_acc & 2 &   \textbf{0.951} (0.02) &   0.855 (0.01) &   0.860 (0.01) &   0.807 (0.01) &   0.845 (0.03) &   0.861 (0.01) &   0.811 (0.02) &   10.45\% \\
          & 3 &   \textbf{0.866} (0.03) &   0.791 (0.01) &   0.799 (0.01) &   0.760 (0.01) &   0.764 (0.03) &   0.802 (0.01) &   0.727 (0.02) &    7.98\% \\
          & 4 &   \textbf{0.815} (0.02) &   0.756 (0.01) &   0.758 (0.01) &   0.747 (0.01) &   0.708 (0.04) &   0.770 (0.01) &   0.658 (0.02) &    5.84\% \\ \hline
\end{tabular}
}
\caption{\footnotesize Increasing number of groups: model performances when $K=5$, $T=5$, and $\phi=0$, where $G$ represents the number of groups. Standard errors are summarized in the parentheses next to the estimated means. The improvement rate compares HCF-DTM against the best performer of the competing methods.}
\label{tab:number-groups}
\end{table}

\subsubsection{Specifications of distance metrics} 

An important component in our group-wise separation procedure is to calculate the distances between counterfactual distributions. As a result, specifications of distance metrics are crucial. In this ablation study, we examine the sensitivities of different distance metrics and test the significance of including such distance metrics. Specifically, in additional to the KL-divergence scores in equations \eqref{eq:KL-tot} and \eqref{eq:KL-ave}, we consider the Manhattan distance (L1-norm), Euclidean distance (L2-norm), and Chebyshev distance (supremum-norm). To signify our model without the distance metrics, we remove the group-wise distance component from our objective function and denote this model as the ``None'' model. In this simulation setting, we also reduce the group effects from the generating process and fix four groups ($G=4$) to examine models' group-wise separation abilities when there is minimal signal in the topic proportions indicating group-wise differences. Simulation results are summarized in Table \ref{tab:distance-ablation}.

As shown, the incorporation of distance metrics outperforms our model without the group-wise component (``None'') at most by 4\% in the group-wise classification accuracy, and meanwhile, generally improves the coherence scores. This result indicates that maximizing the inter-distributional distances can indeed improve group separation and topic interpretations. In addition, by comparing against the KL-type distance metrics, the distance norms can further improve the KL-divergence of the proposed method. One possible reason is that these norm metrics only maximize the distances between variational topic posterior means, which has a more straightforward calculation on the latent topic spaces. KL-type metrics, on the other hand, measures distribution-wise distances and additionally utilize the noisy variational posterior variance. Nevertheless, KL-type metrics improves the interpretation and group-wise separation between topics and is more suitable for the variational Bayes setting where the posterior distribution is attainable.

\begin{table}[H]
    \centering
    \def\arraystretch{1.3}%
    \resizebox{0.9\textwidth}{!}{
\begin{tabular}{|l|lllll|}
\hline
Distance &  KL &     coherence &    perplexity &  dominant\_acc &     group\_acc \\
\hline
None        &  13.444 (3.01) &  3.076 (0.27) &  \textbf{3.880} (0.03) &  0.790 (0.08) &  0.794 (0.01) \\ \hline \hline
KL-radius   &  13.599 (2.90) &  \textbf{3.134} (0.25) &  3.889 (0.03) &  \textbf{0.823} (0.07) &  0.823 (0.01) \\ \hline
KL-averaged &  13.717 (2.99) &  3.118 (0.28) &  3.890 (0.03) &  0.809 (0.08) &  \textbf{0.833} (0.02) \\ \hline
L1          &  13.424 (2.80) &  3.108 (0.27) &  3.881 (0.02) &  0.789 (0.07) &  0.829 (0.01) \\ \hline
L2          &  13.566 (2.86) &  3.088 (0.27) &  3.881 (0.02) &  0.782 (0.07) &  0.809 (0.01) \\ \hline
L-inf       &  \textbf{13.173} (2.75) &  3.123 (0.26) &  \textbf{3.880} (0.02) &  0.794 (0.07) &  0.802 (0.01) \\ \hline
\end{tabular}
}
\caption{\footnotesize Distance metrics ablation studies: Proposed model performances when $K=5$, $T=5$, $\phi=0$, and $G=4$. Standard errors are summarized in the parentheses next to the estimated means.}
\label{tab:distance-ablation}
\end{table}

\subsubsection{Sensitivity analysis of group effects} Our proposed topic separation approach maximizes group-wise heterogeneity to better reveal distinct trends between groups. Under normal scenarios, the definition of groups indicates that there exist some level of group effects/heterogeneity. However, in this study, we discuss our model performance when this group effect is minimal. Specifically, we remove the direct group membership effect from Equation (9) specified in our main text, and only consider covariates main effects in the topic proportion generative process below, i.e.,
\begin{equation}
\theta_{t,i,k} = \sigma(\underbrace{\gamma_{t,k}^m \cdot X_{i,t,.}}_{\text{covariates main effect}} + \underbrace{\gamma_t^\theta \cdot \theta_{t-1,i,k}}_{\text{dependency from previous $\theta$}}).
\end{equation}

According to the results shown in Table \ref{tab:Group-effects}, we observe our method has a significant performance drop in KL divergence, and some level of performance decrease among other metrics when the group disparity is removed. This result is expected as our method is designed to maximize group-wise heterogeneity, which, in this case, the group-wise signal does not exist. However, surprisingly, our method still remains a high level of group-wise classification accuracy. While other methods have about $50\%$ group classification accuracy, ours is able to reach $72.9\%$. This suggests that maximizing group-wise heterogeneity on the topic proportions can still efficiently improve differentiating group memberships.  

\begin{table}[H]
    \centering
    \def\arraystretch{1.4}%
    \resizebox{0.95\textwidth}{!}{
\begin{tabular}{|l|l|lllllll|l|}
\hline
metrics & Group effects &        HCF-DTM &           mdLDA &      mdLDA-true &        SCHOLAR &            LDA &            SLDA &         prodLDA & imp-rate \\
\hline
KL & True &  \textbf{3.879} (1.10) &  19.318 (3.07) &  11.924 (1.66) &           7.561 (1.05) &  13.965 (2.64) &  17.118 (0.46) &  22.074 (0.02) &   48.70\% \\
          & False &  \textbf{5.826} (0.65) &  18.056 (3.70) &   9.603 (3.71) &           5.845 (0.52) &  15.941 (1.57) &  25.960 (0.14) &  22.070 (0.02) &    0.33\% \\ \hline
coherence & True &  \textbf{3.208} (0.05) &   1.028 (0.32) &   2.075 (0.26) &           2.489 (0.11) &   2.424 (0.33) &  -0.909 (0.23) &  -1.120 (0.05) &   28.89\% \\
          & False &  \textbf{2.700} (0.03) &   1.091 (0.25) &   1.716 (0.31) &           1.380 (0.08) &   2.035 (0.38) &  -1.228 (0.05) &  -1.262 (0.09) &   32.68\% \\ \hline
perplexity & True &  \textbf{3.681} (0.01) &   4.047 (0.04) &   3.921 (0.06) &           4.152 (0.03) &   3.889 (0.13) &   9.210 (1.08) &   5.301 (0.00) &    5.35\% \\
          & False &  \textbf{4.223} (0.01) &   4.413 (0.03) &   4.387 (0.05) &           4.525 (0.01) &   4.420 (0.04) &   7.518 (0.04) &   5.301 (0.00) &    3.74\% \\ \hline
dominant\_acc & True &  \textbf{0.965} (0.01) &   0.565 (0.18) &   0.942 (0.11) &           0.929 (0.01) &   0.948 (0.06) &   0.309 (0.32) &   0.307 (0.11) &    1.79\% \\
          & False &           0.963 (0.01) &   0.469 (0.13) &   0.780 (0.15) &  \textbf{0.968} (0.02) &   0.800 (0.12) &   0.272 (0.16) &   0.338 (0.03) &   -0.52\% \\ \hline
group\_acc & True &  \textbf{0.973} (0.01) &   0.963 (0.01) &   0.963 (0.01) &           0.940 (0.01) &   0.960 (0.01) &   0.962 (0.01) &   0.808 (0.04) &    1.04\% \\
          & False &  \textbf{0.729} (0.14) &   0.499 (0.01) &   0.502 (0.01) &           0.499 (0.01) &   0.500 (0.01) &   0.503 (0.01) &   0.499 (0.01) &   44.93\% \\ \hline
\end{tabular}
}
\caption{\footnotesize Group-effects analysis: model performances when $K=3$, $T=5$, $\phi=0$, and $G=2$. Standard errors are summarized in the parentheses next to the estimated means. The improvement rate compares HCF-DTM against the best performer of the competing methods.}
\label{tab:Group-effects}
\end{table}

\subsection{Extension to the dynamically-changing topics} \label{A:sim:extension}

Our proposed HCF-DTM is built upon the time-consistent topic assumption, which is useful to uncover long-lasting topics and provide readily usable topic proportions to track the evolution trend of each topic with guaranteed consistent theme meanings. However, admittedly, in a more general use case, topics tend to evolve across time. In this section, we provide a brief overview of extending our method to adapt a dynamically-changing topic scenario.

To account for evolving topics, we can adopt the topic generative prior $\beta_t \sim \mathcal{N}(\beta_{t-1}, \sigma_0^2)$ as in the multistage dynamic LDA \citep{blei2006dynamic} and incorporate the sampled variational topics, i.e., $\Tilde{\beta}_t = \mu_t + \sigma_t \cdot \epsilon_t$, where $\epsilon_t \sim \mathcal{N}(0,1)$, and ($\mu_t, \sigma_t$) are the mean and scale parameters for the variational normal distribution. Then, the reconstruction error component of the derived evidence lower bound now becomes:
\begin{equation}
    \sum_{t=1}^T \mathop{\mathbb{E}}_{\theta_t \sim q_{\psi_t}, \beta_t \sim    q_{(mu_t, \sigma_t}} (\log P(\bm{w}_t| \theta_t, \beta_t)).
    \label{eq:reconstruction-changing-topics-new}
\end{equation}
Additionally, to make the generated topic pertain to its generative prior, we minimize their multistage KL divergence, similarly to the topic proportions, i.e.,
\begin{equation}
    \mathbb{KL} ( q(\beta_1; \mu_1, \sigma_1) \;||\; p(\beta_1)) + \sum_{t=2}^T  \mathbb{E}_{\beta_{t-1} \sim q_{\mu_{t-1}, \sigma_{t-1}}}  \mathbb{KL} ( q(\beta_t| \mu_t, \sigma_t) \;||\; p(\beta_t | \beta_{t-1}, \sigma_0^2)).
    \label{eq:KL-changing-topics-new}
\end{equation}
By plugging equations \eqref{eq:reconstruction-changing-topics-new} and \eqref{eq:KL-changing-topics-new} to our original optimization bound, we obtain the optimization bound for evolving topics. 

Next, to demonstrate our HCF-DTM with evolving topics, we conduct simulations under the same setting as Section \ref{A:sim:break-consistent}. The results are summarized in Table \ref{tab:HCF-DTM-D}, where $\sigma^2_0$ of dynamic topic models are grid-searched among $\{10^{-3},0.01,0.1,0.5,1\}$. Notably, our method with consistent topics still has the best-performing results due to the minimal number of topic parameters need to estimate (i.e., $\beta$ v.s. $\{\beta_{t}\}_{t=1}^T$). On the other hand, when comparing our model with changing topics against the mdLDA, our model improves the interpretability of the topics, indicated by higher coherence and dominant topic accuracy scores. However, the inclusion of additional topics variational parameters increases the complexity of optimization procedure, leading to a larger perplexity score and notable performance differences compared to the HCF-DTM with time-consistent topics. This simulation further demonstrates that assuming $\sigma_0^2=0$ for all topics can simplify our model optimization with one set of topics, and meanwhile, align with our specific application needs.

\begin{table}[H]
    \centering
    \def\arraystretch{1.1}%
    \resizebox{0.8\textwidth}{!}{
\begin{tabular}{|l|l|lllll|}
\hline
method &  $\sigma_0^2$ &            KL &     coherence &  dominant\_acc &     group\_acc &    perplexity \\ \hline
 HCF-DTM &  0 &  3.263 (0.88) &  2.822 (0.06) &  0.963 (0.01) &  0.974 (0.01) &  3.868 (0.01) \\\hline
HCF-DTM & 1e-3 &  15.097 (2.47) &  1.476 (0.60) &  0.769 (0.27) &  0.950 (0.01) &  4.983 (0.34) \\
mdLDA & 1e-3 & 15.091 (3.98) &  0.954 (0.39) &  0.509 (0.21) &  0.963 (0.01) &  4.094 (0.04) \\ \hline
\end{tabular}
}
\caption{\footnotesize Dynamically changing topic distribution: model performances when $K=3$, $T=5$, $\phi=0.5$, and $G=2$. Standard errors are summarized in the parentheses next to the estimated means.}
\label{tab:HCF-DTM-D}
\end{table}

\subsection{Additional simulation tables}
This section includes more model performance results to Section 5 in our main text. Under the same simulation setting with non-linear generative prior functions, we present all dominant topic matching accuracy and document group membership prediction accuracy in Table \ref{tab:appendix-dominant-acc} and \ref{tab:appendix-group-acc} respectively. 

\begin{table}[hbt]
    \centering
        \def\arraystretch{1.5}%
    \resizebox{0.9\textwidth}{!}{
\begin{tabular}{|l|l|lllllll|}
\hline
  K & T &             LDA &            sLDA &         prodLDA &         SCHOLAR &            mdLDA &                   HCF-DTM &          Imp-rate \\
\hline
3 & 3 &  0.856 (0.146) &  0.254 (0.303) &  0.307 (0.089) &  0.640 (0.083) &  0.721 (0.165) &  \textbf{0.947} (0.010) &  \textbf{10.631}\% \\
  & 5 &  0.855 (0.140) &  0.287 (0.261) &  0.355 (0.089) &  0.700 (0.069) &  0.646 (0.194) &  \textbf{0.944} (0.008) &  \textbf{10.409}\% \\
  & 8 &  0.847 (0.141) &  0.236 (0.220) &  0.351 (0.068) &  0.743 (0.039) &  0.552 (0.106) &  \textbf{0.945} (0.013) &  \textbf{11.570}\% \\ \hline
5 & 3 &  0.754 (0.155) &  0.236 (0.285) &  0.174 (0.098) &  0.578 (0.085) &  0.723 (0.136) &  \textbf{0.854} (0.122) &  \textbf{13.263}\% \\
  & 5 &  0.745 (0.178) &  0.201 (0.248) &  0.231 (0.082) &  0.547 (0.063) &  0.578 (0.113) &  \textbf{0.871} (0.088) &  \textbf{16.913}\% \\
  & 8 &  0.751 (0.189) &  0.261 (0.247) &  0.224 (0.078) &  0.568 (0.079) &  0.502 (0.142) &  \textbf{0.896} (0.048) &  \textbf{19.308}\% \\ \hline
8 & 3 &  0.681 (0.233) &  0.113 (0.150) &  0.167 (0.076) &  0.399 (0.088) &  0.731 (0.137) &  \textbf{0.904} (0.104) &  \textbf{23.666}\% \\
  & 5 &  0.674 (0.206) &  0.115 (0.190) &  0.111 (0.059) &  0.351 (0.084) &  0.586 (0.119) &  \textbf{0.855} (0.087) &  \textbf{26.855}\% \\
  & 8 &  0.668 (0.207) &  0.108 (0.116) &  0.124 (0.057) &  0.405 (0.082) &  0.428 (0.073) &  \textbf{0.845} (0.051) &  \textbf{26.497}\% \\
\hline
\end{tabular}
}
\caption{Dominant topic accuracy from estimated topic proportions when the generative prior function is non-linear. Standard errors are summarized in the parentheses next to the estimated means. The improvement rate compares HCF-DTM against the best performer of the competing methods.}
\label{tab:appendix-dominant-acc}
\end{table}

\begin{table}[hbt]
    \centering
        \def\arraystretch{1.5}%
    \resizebox{0.9\textwidth}{!}{
\begin{tabular}{|l|l|lllllll|}
\hline
  K & T &             LDA &            sLDA &         prodLDA &         SCHOLAR &            mdLDA &                   HCF-DTM &          Imp-rate \\
\hline
3 & 3 &  0.846 (0.014) &  0.855 (0.009) &  0.703 (0.049) &  0.784 (0.014) &  0.853 (0.010) &  \textbf{0.906} (0.035) &   \textbf{5.965}\% \\
  & 5 &  0.843 (0.016) &  0.852 (0.008) &  0.722 (0.037) &  0.787 (0.017) &  0.851 (0.009) &  \textbf{0.895} (0.038) &   \textbf{5.047}\% \\
  & 8 &  0.841 (0.021) &  0.853 (0.009) &  0.721 (0.025) &  0.800 (0.011) &  0.853 (0.010) &  \textbf{0.895} (0.027) &   \textbf{4.924}\% \\ \hline
5 & 3 &  0.843 (0.010) &  0.858 (0.010) &  0.798 (0.026) &  0.822 (0.012) &  0.858 (0.009) &  \textbf{0.932} (0.025) &   \textbf{8.625}\% \\
  & 5 &  0.842 (0.012) &  0.855 (0.009) &  0.790 (0.019) &  0.820 (0.010) &  0.856 (0.010) &  \textbf{0.908} (0.028) &   \textbf{6.075}\% \\
  & 8 &  0.849 (0.009) &  0.862 (0.009) &  0.801 (0.017) &  0.821 (0.009) &  0.862 (0.008) &  \textbf{0.905} (0.030) &   \textbf{4.988}\% \\ \hline
8 & 3 &  0.843 (0.010) &  0.857 (0.008) &  0.802 (0.022) &  0.825 (0.010) &  0.858 (0.009) &  \textbf{0.975} (0.010) &  \textbf{13.636}\% \\
  & 5 &  0.847 (0.011) &  0.860 (0.009) &  0.809 (0.016) &  0.833 (0.011) &  0.861 (0.010) &  \textbf{0.931} (0.028) &   \textbf{8.130}\% \\
  & 8 &  0.845 (0.012) &  0.861 (0.009) &  0.807 (0.016) &  0.827 (0.009) &  0.860 (0.008) &  \textbf{0.901} (0.028) &   \textbf{4.646}\% \\
\hline
\end{tabular}
}
\caption{Group-membership accuracy from estimated topic proportions when the generative prior function is non-linear. Standard errors are summarized in the parentheses next to the estimated means. The improvement rate compares HCF-DTM against the best performer of the competing methods.}
\label{tab:appendix-group-acc}
\end{table}

Similarly to the results established in Sections \blue{5.2} and \blue{5.3}, both tables show that the proposed method outperform the rest of the competing methods. Specifically, Table \ref{tab:appendix-dominant-acc} shows that, as the number of stages increases under each specification of number of latent topics, our margin of improvement in the dominant topic accuracy is also enlarged. This illustrates that our method can effectively estimate the topic proportions with the documents metadata and group-membership, especially when the inference task involves a large number of stages or latent topics. On the other hand, we observe that the improvement rate of the document group membership prediction accuracy decreases with the number of stages based on Table \ref{tab:appendix-group-acc}. This decline suggests a reduction in the group-wise heterogeneity effect but a strengthened longitudinal dependencies among the topics over time. However, despite the decrease in the improvement rate, the proposed method still reaches the best performance and further improves the accuracy with more latent topics under the same number of stages.

Next, we present the empirical KL divergence results when the generative prior is linear in Table \ref{tab:appendix-KL}, Once again, our method showcases its best-performing capability of reducing KL divergence in every scenario. However, compared to the non-linear results, the improvement rate is slightly smaller in the linear setting. This is due to the reduced functional complexity leading to a more competitive result from the competing methods. Nevertheless, as the number of stages increases, the inference complexity rebounds, making our improvement rate similar to the ones under the non-linear setting.

\begin{table}[h]
    \centering
        \def\arraystretch{1.5}%
    \resizebox{0.9\textwidth}{!}{
\begin{tabular}{|l|l|lllllll|}
\hline
  K & T &             LDA &            sLDA &         prodLDA &         SCHOLAR &            mdLDA &                   HCF-DTM &          Imp-rate \\
\hline
3 & 3 &  15.414 (2.363) &  19.227 (0.184) &  22.070 (0.026) &  15.516 (0.516) &   8.804 (4.917) &   \textbf{4.206} (0.622) &  \textbf{52.226}\% \\
  & 5 &  15.372 (2.316) &  19.253 (0.247) &  22.069 (0.016) &  15.197 (0.484) &  12.131 (4.212) &   \textbf{4.096} (0.664) &  \textbf{66.235}\% \\
  & 8 &  15.402 (2.317) &  19.214 (0.273) &  22.065 (0.013) &  15.098 (0.653) &  14.006 (2.915) &   \textbf{3.739} (0.576) &  \textbf{73.304}\% \\ \hline
5 & 3 &  18.968 (1.944) &  23.124 (1.104) &  26.997 (0.014) &  22.084 (0.557) &  13.527 (3.219) &  \textbf{12.021} (1.969) &  \textbf{11.133}\% \\
  & 5 &  18.992 (1.863) &  23.029 (0.739) &  27.005 (0.013) &  21.853 (0.586) &  16.696 (2.395) &  \textbf{10.298} (2.554) &  \textbf{38.321}\% \\
  & 8 &  19.093 (1.784) &  22.901 (0.687) &  27.003 (0.009) &  21.632 (0.475) &  19.251 (2.403) &   \textbf{8.767} (3.103) &  \textbf{54.083}\% \\ \hline
8 & 3 &  23.579 (1.035) &  27.854 (0.892) &  31.007 (0.009) &  27.108 (0.343) &  24.038 (0.780) &  \textbf{19.101} (1.320) &  \textbf{18.991}\% \\ 
  & 5 &  23.463 (1.151) &  27.593 (0.633) &  30.999 (0.008) &  26.984 (0.288) &  25.488 (0.824) &  \textbf{14.566} (2.391) &  \textbf{37.919}\% \\
  & 8 &  23.511 (1.014) &  27.777 (0.501) &  31.000 (0.006) &  26.822 (0.270) &  26.165 (0.950) &  \textbf{12.130} (2.377) &  \textbf{48.407}\% \\
\hline
\end{tabular}
}
\caption{Empirical KL divergence between the estimated and actual topic word distributions when the generative prior function is linear. Standard errors are summarized in the parentheses next to the estimated means. The improvement rate compares HCF-DTM against the best performer of the competing methods.}
\label{tab:appendix-KL}
\end{table}

\section{Real data application}
\subsection{Data descriptive statistics}
The mental health clinical notes from CHOC records a comprehensive review of a patient, such as summary of past notes, medication history, lab test results and treatment plan. In our analysis, we focus on the text portion in the Background/Assessment section of the notes and augment with the patients' vital measurements from a structural database. A descriptive summary of the selected covariates can be found in Table \ref{tab:descriptive_statistics}. As shown, a majority of inpatient children are teenager girls with an averaged age 14.5 years old. In addition, we notice there are a lot of missing values among the vital recordings. To preprocess the metadata, we impute the missing covariate values by random forests \citep{tang2017random}. Together with the bag of 273 selected words, we prepared the input data to the proposed model.

\begin{table}[hbt]
    \centering
        \bgroup
    \def\arraystretch{1.05}%
    \resizebox{0.9\textwidth}{!}{
    \begin{tabular}{lllllll}
    \hline 
    & Count & Mean & Median & Min & Max & S.D. \\
    \hline
    \textbf{Demographics} & & & & & & \\ \hline
    Female & 3313 & 0.664 & 1 & 0 & 1 & 0.472 \\
    Ethnicity: Hispanic & 3313 & 0.459 & 0 & 0 & 1 & 0.498 \\
    Race: White & 3313 & 0.626 & 1 & 0 & 1 & 0.483 \\
    Race: Asian & 3313 & 0.070 & 0 & 0 & 1 & 0.255 \\
    Race: Black & 3313 & 0.051 & 0 & 0 & 1 & 0.220 \\
    Age & 2564 & 14.48 & 14.90 & 4.50 & 18.90 & 2.44 \\
    \hline
    \textbf{Basic Measurements} & & & & & & \\ \hline
    Weight & 2530 & 61.42 & 58.00 & 14.70 & 164.00 & 19.99 \\
    Height & 2366 & 158.40 & 160.50 & 15.70 & 195.00 & 18.12 \\
    Body Temperature & 2534 & 36.54 & 36.50 & 31.70 & 38.20 & 0.29 \\
    Heart Rate & 2540 & 84.48 & 84.00 & 52.40 & 137.00 & 9.99 \\
    Systolic Blood Pressure & 2539 & 114.29 & 113.70 & 83.00 & 151.00 & 8.50 \\
    Diastolic Blood Pressure & 2539 & 70.87 & 70.60 & 45.30 & 114.00 & 5.72 \\
    \hline
    \end{tabular}
}
    \caption{Descriptive statistics of all obtained covariates for inpatient children during their first visist.}
    \label{tab:descriptive_statistics}
    \egroup
\end{table}

\subsection{Verification of generative process assumption}

\subsubsection{Time-consistent topic assumption justification}

To validate if our motivated clinical data indeed contains time-invariant topics but with disparate proportions for different groups over time, we conduct preliminary data analysis using the popular multistage dynamic LDA (mdLDA) \citep{blei2006dynamic}. This method assumes the topics follow the generative process as follows,
\begin{equation}
    \beta_t \sim \mathcal{N}(\beta_{t-1}, \sigma_0^2),
\end{equation}
where $\sigma_0^2$ is specified as the topic transition variance prior. In our case, we assume $\sigma_0^2=0$. However, it is worth noting that the topic posterior distribution still can shift over time with the invariant topic transition prior, and the mdLDA is unable to utilize any document metadata, such as subjects' measurement and group information. Though not perfectly pertained to our need, mdLDA is sufficient as an explanatory data analysis tool and the results with three topics are displayed in Figure \ref{fig:explanatory-analysis}. As shown, each fitted topic contains consistent words throughout the five-year time stages, and can recover the common classic mental health topics, i.e., ``anxiety" and ``suicide." In addition, as the fitted topic posterior does not deviate much from the prior, we show that there exist several static topics in our data and the time-invariant topic assumption is justified under our research interests.
\begin{figure}[H]
    \centering
    \includegraphics[width=0.85\textwidth]{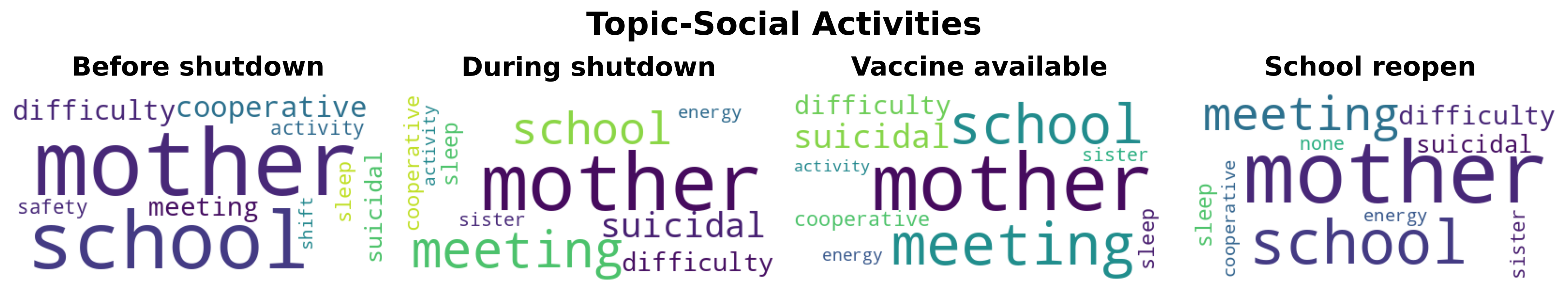}
    \includegraphics[width=0.85\textwidth]{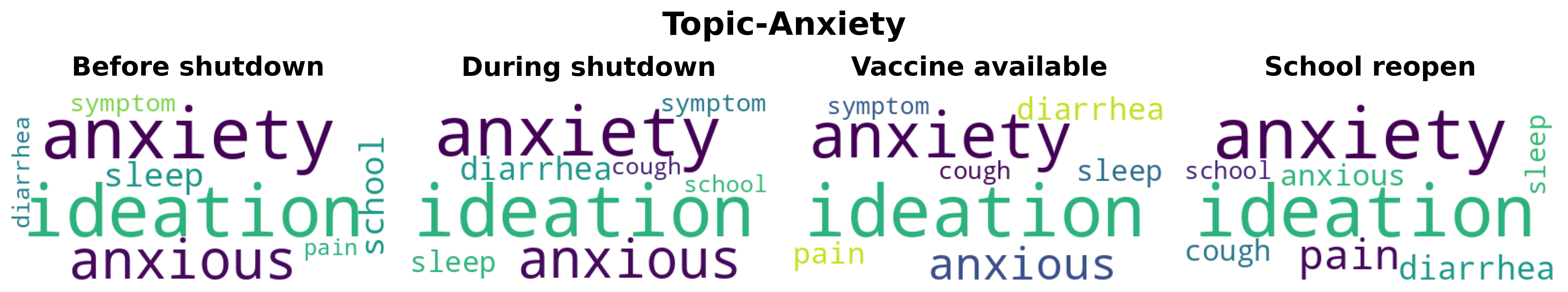}
    \includegraphics[width=0.85\textwidth]{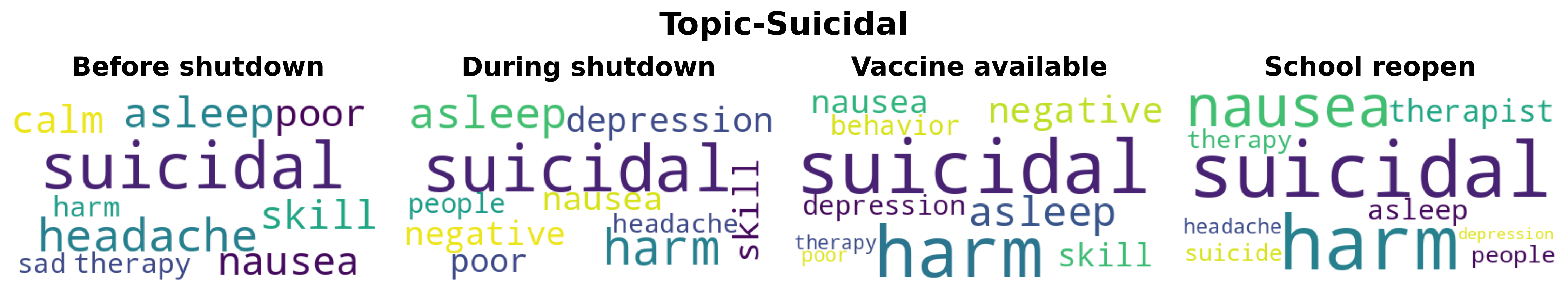}
    \caption{\footnotesize Three explanatory topics fitted by multistage dynamic LDA \citep{blei2006dynamic} with the topic transition variance prior specified as 0 ($\sigma^2_0=0$).}
    \label{fig:explanatory-analysis}
\end{figure}

To further justify the necessity of assuming time-consistent topics, we re-fit the mdLDA  with package default topics transition prior ($\sigma^2_0=0.1)$.  According to the results displayed in Figure \ref{fig:real-LLDA-editor}, it is noticeable that during the first three time points, topic 1 is about suicidal thoughts and anxiety, whereas topic 3 can be interpreted as family interactions. But, by the fourth time stage, the theme of topics 1 and 3 seems to be exchanged: words like ``family'' and ``school'' emerged in topic 1, and ``suicidal'' appeared in topic 3. This is one illustrating example due to the poor selection of the transition prior $\sigma^2_0$. Unfortunately, currently there is no better approach in the literature to choose the optimal transition prior. While some rely on metrics like coherence or perplexity scores to examine the goodness-of-fit for the model, these methods still require manual inspection of the word-clouds at each time stage to interpret and ensure the topics are under the same theme. As the number of stages and topics increases, the cumbersome nature of this process will inevitably impose a significant challenge in topic interpretations and representations.

In addition, even if one successfully extracts consistent themes from the latent topics, it is still challenging to interpret the topic proportions at each time stage. For example, suppose we observe a marginal upward trend in the topic proportions within topic 1 during the first two time stages. However, considering that the word-cloud of topic 1 at the second time stage in Figure \ref{fig:real-LLDA-editor} includes additional ``mother'' as a significant term, it would be difficult to attribute this increasing proportion trend solely to a higher prevalence of the common theme, i.e., ``suicidal ideation,'' rather than the inclusion of newly significant words. In other words, the resulting topic proportions cannot reflect topic progressions due to the lack of longitudinal consistency from the change in topic distributions at each time point.

\begin{figure}[H]
    \centering
    \includegraphics[width=0.85\textwidth]{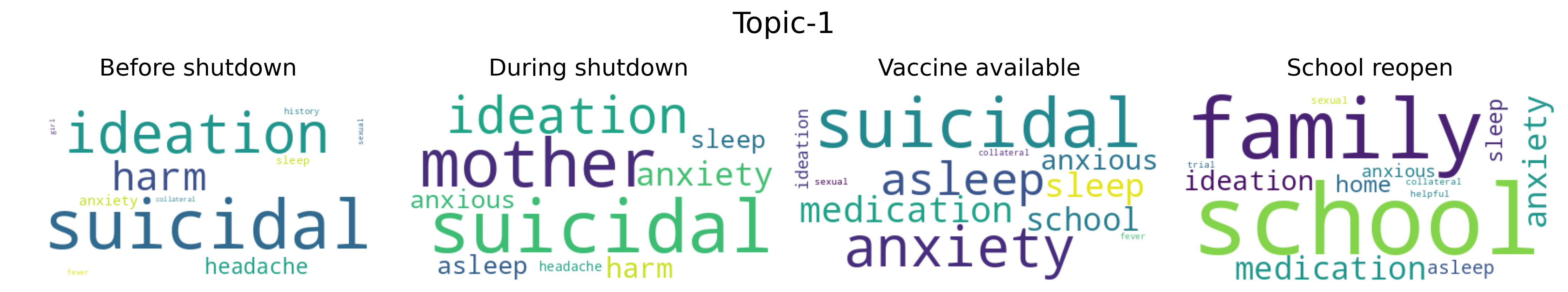}
    \includegraphics[width=0.85\textwidth]{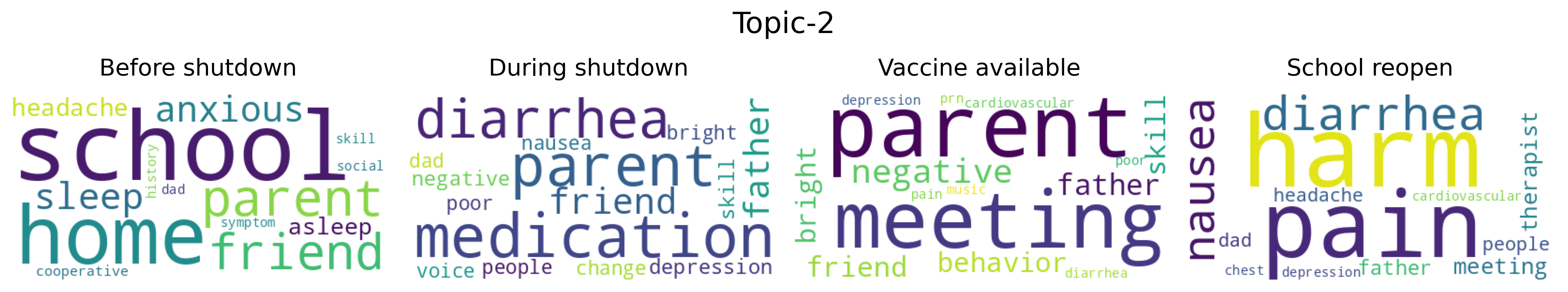}
    \includegraphics[width=0.85\textwidth]{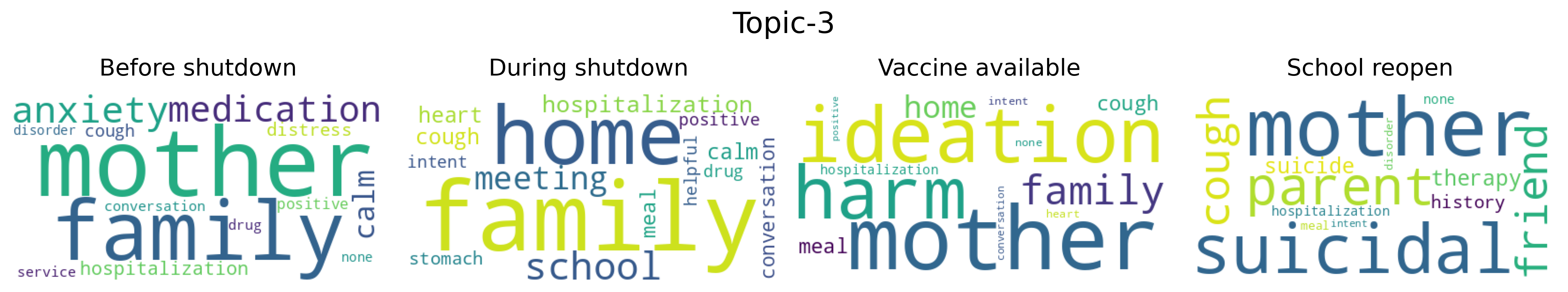}
    \caption{ \footnotesize An exemplary topic found by the multistage dynamic LDA \citep{blei2006dynamic} with topic transition variance prior specified as $0.1$ (package default setting). }
    \label{fig:real-LLDA-editor}
\end{figure}

Lastly, in comparison to the dynamic topic assumption, we revisit our static topic assumption to the mental health application. First, our generated topics persist over time and are guaranteed to belong to the same theme without further manual inspections. This facilitates the recovery of consistent topics, such as negative emotions (suicidal thoughts) and social interactions (family, school), from the documents. Secondly, the consistency of the topics enables the utilization of topic proportions to depict the longitudinal trajectory of topic prevalence. This allows us to detect dynamic patterns in patients' mental status over time, such as whether children patients experienced more social interactions or not. From an optimization standpoint, since our prior belief is to discover consistent themes/topics, it is more efficient to set the transition prior $\sigma^2_0=0$ and only optimize one set of topics, instead of setting a very small prior and optimizing topics at each stage. This is further demonstrated from the additional simulation settings with dynamically changing topics we conducted in Section \ref{A:sim:extension}.

\subsubsection{Existence of group-wise heterogeneity}
In the next step, we conduct explanatory data analysis to demonstrate there exists group-wise heterogeneity in the topic proportions. Specifically, based on the mdLDA model fitted in Figure \ref{fig:real-LLDA-editor}, we visualize the percentages of clinical notes with ``Suicidal'' as their dominant topic according to the sexual and gender identity of each child patient.  From Figure \ref{fig:group-wise-exp}, it is noticeable that the sexual and gender minority (SGM) children have more suicidal intention when the state shutdowns and it decreases faster after the vaccine becomes available. The differences in the trend demonstrate the existence of group-wise disparity in the topic proportions over time. In fact, according to previous mental health studies, girls and SGM children tend to report their symptoms differently compared to boys and non-SGM youth \citep{afifi2007gender,rosenfield2013gender,marshal2011suicidality,ploderl2015mental,russell2016mental}. Therefore, it is reasonable to assume group-wise heterogeneity persists within the topic proportions.
\begin{figure}[H]
    \centering
    \includegraphics[width=0.9\textwidth]{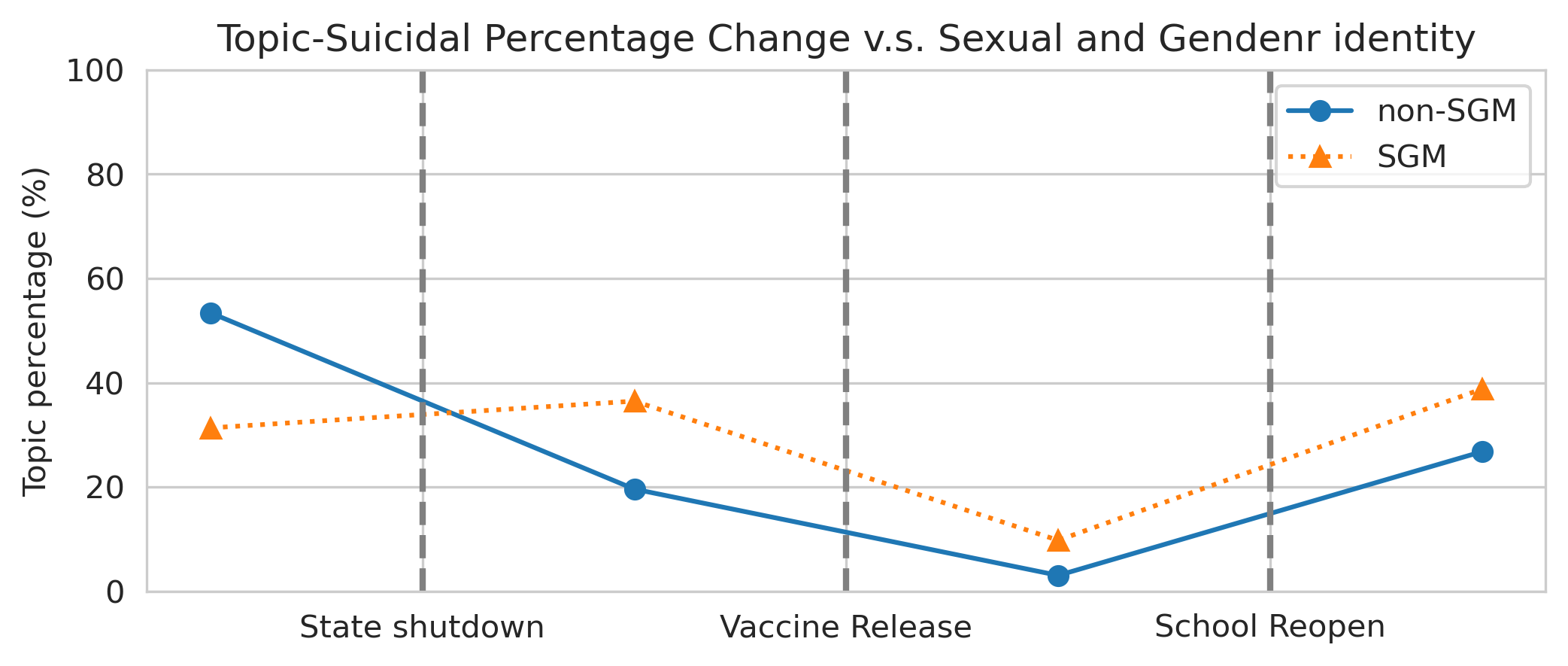}
    \caption{\footnotesize Percentage changes in ``Suicidal'' topic fitted by mdLDA based on sexual and gender identities.}
    \label{fig:group-wise-exp}
\end{figure}

\subsection{Choice of number of topics}
In the literature, one can choose the number of topics based on the widely-adopted evaluation metrics, such as coherence and perplexity scores. Generally, a higher coherence score and a lower perplexity score indicate a model which is more easily interpretable by humans and better fits the data. However, it is important to note that the highest coherence score does not always correspond to the optimal choice of the number of topics. Studies by \cite{stevens2012exploring} and \cite{zvornicanin2023coherence} show that coherence scores tend to increase with the number of topics.

For a real case, we adopt a two-fold approach to determine the number of topics. First, we evaluate the previous coherence and perplexity metrics, and secondly, we conduct preliminary word-cloud analysis to provide an initial estimate of the number of topics, while taking into account our targeted themes based on our study interests. The guiding principle is to achieve decent coherence scores while ensuring that the results remain readily interpretable without posing excessive challenges during evaluation. To be more specific, we consider the word-clouds from the explanatory multistage dynamic topic model \citep{blei2006dynamic} with default parameters in Figure \ref{fig:real-LLDA-editor}. Without delving into topic longitudinal progression and only focusing on the keywords, we observe potential themes such as negative sentiment (``anxiety'', ``suicidal'', ``depression'', ``negative'', ), social interactions (``mother'', ``parent'', ``family'', ``friend'', ``people'', ``school''), physical discomfort (``headache'', ``nausea'', ``pain'', ``harm'', ``cough''), medication (``treatment'', ``hospitalization'', ``medication'', ``therapy''), and positive emotions (``positive'', ``cooperative'', ``calm''). While the topic models may not provide a clear segregation of these themes, this analysis offers a reasonable range of topic numbers to consider. 

In the following, we list the word-cloud results under $K=4,5, \text{and } 8$. Across all settings, the topic model is able to identify anxiety/suicidal intentions, family and social interactions, as well as physical discomfort and medication themes. In particular, when the number of topics increases, the identified topics become more specific and can be manually grouped into broader themes. For instance, when $K=8$, we can categorize topics 1, 4, and 7 to the negative sentiment theme; topics 2, 6, and 8 to the family and social theme; and topics 3 and 5 to the physical discomfort theme. Additionally, with more specific topics, the eight topics tend to reach the highest coherence score. As a result, in our application, we choose to select three topics as they still effectively capture the three major themes present in the real data, based on the models with more topics.
\begin{figure}[H]
    \centering
    \includegraphics[width=\textwidth]{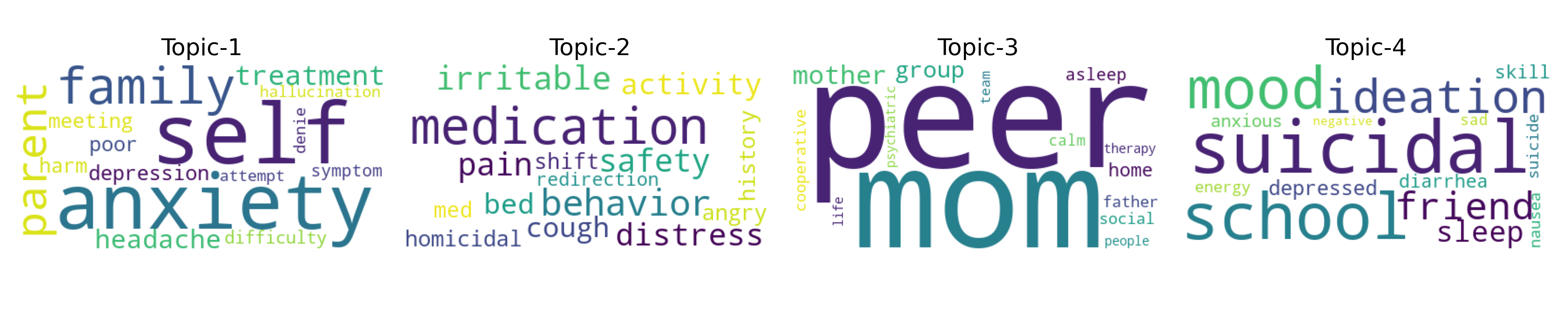}
    \caption{Word-clouds when four topics are specified ($K=4$).}
    \label{fig:K=4}
\end{figure}
\begin{figure}[H]
    \centering
    \includegraphics[width=\textwidth]{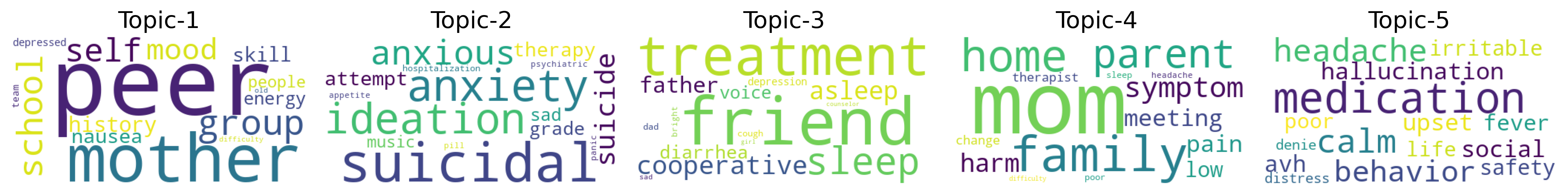}
    \caption{Word-clouds when five topics are specified ($K=5$).}
    \label{fig:K=4}
\end{figure}
\begin{figure}[H]
    \centering
    \includegraphics[width=\textwidth]{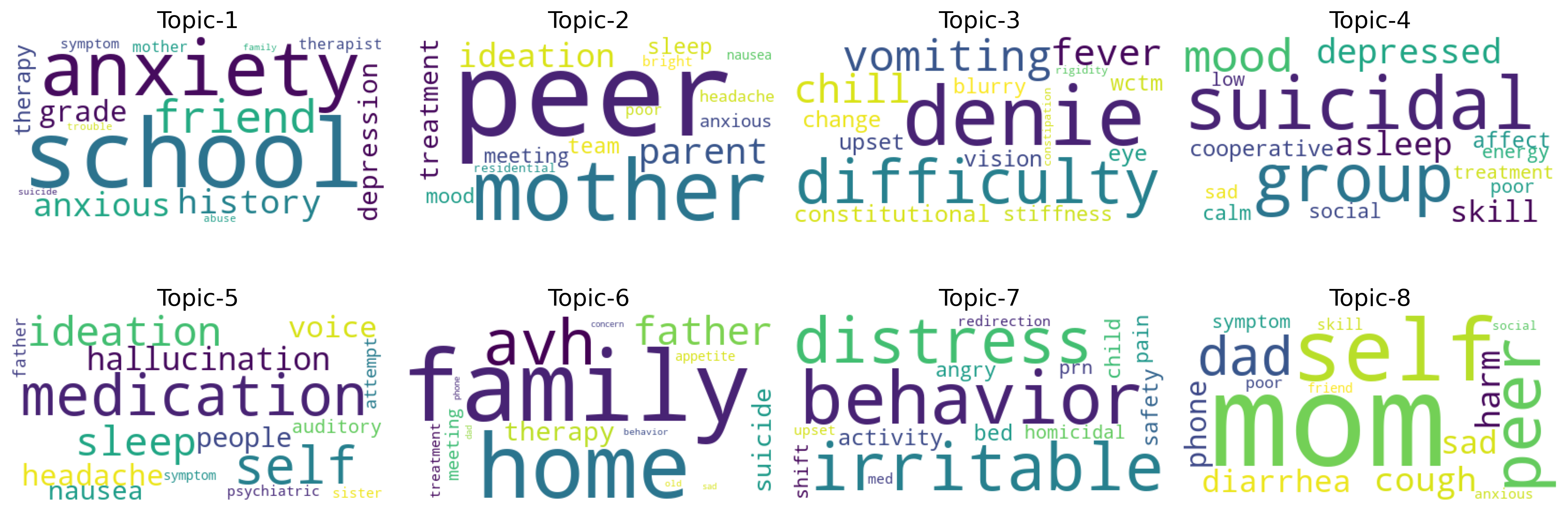}
    \caption{Word-clouds when eight topics are specified ($K=8$).}
    \label{fig:K=8}
\end{figure}

\newpage
{
\bibliographystyle{imsart-nameyear} 
\bibliography{bibliography}       
}